\documentclass[11pt]{article}
\usepackage{microtype}
\usepackage{graphicx}
\usepackage{subcaption}
\usepackage{booktabs}
\usepackage{hyperref}

\usepackage[margin=1in]{geometry}
\usepackage[numbers,sort&compress]{natbib}
\setcitestyle{numbers,square,comma}
\usepackage{amsmath,amssymb,amsthm,mathtools,bm}
\usepackage{tabularx,array,multirow,longtable,float}
\usepackage{xcolor}
\usepackage{tikz,pgfplots}
\usetikzlibrary{arrows.meta,positioning,fit,backgrounds,calc,shapes.geometric,decorations.pathreplacing}
\usepgfplotslibrary{groupplots}
\pgfplotsset{compat=1.18}
\usepackage[most]{tcolorbox}
\usepackage{proof-at-the-end}

\definecolor{navy}{HTML}{153B5B}
\definecolor{blue}{HTML}{3B82C4}
\definecolor{teal}{HTML}{20A39E}
\definecolor{coral}{HTML}{F07167}
\definecolor{gold}{HTML}{E9B44C}
\definecolor{ink}{HTML}{17212B}
\definecolor{muted}{HTML}{5E6A75}
\definecolor{pale}{HTML}{F3F7FA}
\definecolor{paleblue}{HTML}{EAF3FA}
\definecolor{paleteal}{HTML}{E8F7F5}
\definecolor{palecoral}{HTML}{FDEEEB}
\hypersetup{colorlinks=true,linkcolor=navy,citecolor=navy,urlcolor=blue,
  pdftitle={LLM Capability Limits: Static Emergence and Dynamic Boundary Control},pdfauthor={Yi Liu}}

\newtheorem{theorem}{Theorem}
\newtheorem{lemma}[theorem]{Lemma}
\newtheorem{proposition}[theorem]{Proposition}
\newtheorem{corollary}[theorem]{Corollary}
\theoremstyle{definition}
\newtheorem{definition}[theorem]{Definition}

\newtheorem{example}[theorem]{Example}

\let\paperTheorem\theorem
\let\endpaperTheorem\endtheorem
\let\paperLemma\lemma
\let\endpaperLemma\endlemma
\let\paperProposition\proposition
\let\endpaperProposition\endproposition
\let\paperCorollary\corollary
\let\endpaperCorollary\endcorollary
\let\paperProof\proof
\let\endpaperProof\endproof
\RenewDocumentEnvironment{theorem}{O{}+b}{%
  \begin{theoremEnd}{paperTheorem}[#1]#2\end{theoremEnd}%
}{}
\RenewDocumentEnvironment{lemma}{O{}+b}{%
  \begin{theoremEnd}{paperLemma}[#1]#2\end{theoremEnd}%
}{}
\RenewDocumentEnvironment{proposition}{O{}+b}{%
  \begin{theoremEnd}{paperProposition}[#1]#2\end{theoremEnd}%
}{}
\RenewDocumentEnvironment{corollary}{O{}+b}{%
  \begin{theoremEnd}{paperCorollary}[#1]#2\end{theoremEnd}%
}{}
\RenewDocumentEnvironment{proof}{+b}{%
  \begin{proofEnd}#1\end{proofEnd}%
}{}
\newcommand{\restoreappendixenvironments}{%
  \let\theorem\paperTheorem
  \let\endtheorem\endpaperTheorem
  \let\lemma\paperLemma
  \let\endlemma\endpaperLemma
  \let\proposition\paperProposition
  \let\endproposition\endpaperProposition
  \let\corollary\paperCorollary
  \let\endcorollary\endpaperCorollary
  \let\proof\paperProof
  \let\endproof\endpaperProof
}
\pratendSetGlobal{text link={See \hyperref[app:proofs]{Appendix Module I, Proofs of main-text results}.}}
\newcommand{\F}{\mathcal F}
\newcommand{\R}{\mathcal R}
\newcommand{\A}{\mathcal A^\dagger}
\newcommand{\U}{\mathcal U}
\newcommand{\Sspace}{\mathcal S}
\newcommand{\TV}{\operatorname{TV}}
\newcommand{\E}{\mathbb E}
\newcommand{\ind}{\mathbf 1}
\newcommand{\Opt}{\operatorname{Opt}}
\newcommand{\Ker}{\operatorname{Ker}}
\newcommand{\clco}{\operatorname{cl\,co}}

\newtcolorbox{claimbox}{colback=paleblue,colframe=blue,boxrule=0.7pt,arc=2mm,
  left=1.5mm,right=1.5mm,top=1mm,bottom=1mm}
\newtcolorbox{takeaway}{colback=paleteal,colframe=teal,boxrule=0.7pt,arc=2mm,
  left=1.5mm,right=1.5mm,top=1mm,bottom=1mm}
\newtcolorbox{cautionbox}{colback=palecoral,colframe=coral,boxrule=0.7pt,arc=2mm,
  left=1.5mm,right=1.5mm,top=1mm,bottom=1mm}

\title{\textbf{LLM Capability Limits: Static Emergence and Dynamic Boundary Control}}
\author{Yi Liu\\
\small MS Student\\
\small School of Astronomy and Space Science, University of Science and Technology of China\\
\small \texttt{scnuliuyi@mail.ustc.edu.cn}}

\begin{document}
\maketitle

\begin{abstract}
Test-time emergence in LLM systems has a deployment boundary: additional computation can realize decisions already supported by the deployed information--execution structure, while evidence, tools, memory, and executable semantics can change the class inherited by later computation. We formalize this boundary through inherited structural capability $\mathcal D_{\mathcal J}$ and resource-indexed finite realization $\F_s(\mathcal J,M)$. At a common budget, Theorem~\ref{thm:main-budget-reversal} gives an exact decision representation: a successor improves every bounded-loss task exactly when its closed convex finite envelope retains the predecessor's. Terminal capability can therefore expand while same-budget capability strictly reverses. The same object yields finite-slice recovery and a workload-tail information radius for open-ended evaluation. Dynamically, Bellman value prices the successor capability class together with the finite policies it preserves. Nested realization makes every fixed extra resource increment vanish at saturation, allowing persistent positive successor value to dominate that increment. The resulting theory turns emergence into a boundary, compatibility, measurement, and control problem.
\end{abstract}


\section{Emergence as a capability-boundary problem}
\label{sec:intro}

More test-time reasoning can produce striking gains and can also saturate at a task boundary. The decisive question is which task distinctions and executable consequences the deployed system already inherits, and which interventions change them. Reasoning, search, verification, and additional budget realize more of an inherited decision class. Retrieval, new observations, tools, interpreters, memory channels, and clarification can instead change the class available to later computation. These directions can improve the same score while moving different capability boundaries.

The boundary becomes operational under a fixed resource ceiling. A successor can expose new distinctions or executable actions, lower its terminal risk floor, and simultaneously make an earlier useful policy more expensive to realize. We therefore ask two coupled questions: where does structure-preserving emergence stop, and what must an upgrade retain to improve every bounded-loss task at the same token, latency, memory, energy, or tool-call budget?

We answer this question by separating two levels:
\begin{equation}
\resizebox{0.98\columnwidth}{!}{$
\boxed{\text{inherited structural capability}}
\ \supseteq\ 
\boxed{\text{finite realized capability}}
$}
\label{eq:two-level-view}
\end{equation}
The first level records the decision rules supported by task distinctions and executable consequences. The second records the portion exposed by a realization mechanism under a declared deployment budget. Their separation creates two different orders: terminal capability and finite deployment capability.

At a higher level of abstraction, the same architecture follows from the Symbolization--Substructure Thesis: \emph{language is a consequence-preserving, partial symbolic substructure of thought}. A language-mediated system inherits a task-relative symbolic organization, connects represented distinctions to executable consequences, and organizes those inherited possibilities into bounded behavior:
\begin{equation}
\resizebox{0.98\columnwidth}{!}{$
\text{distinctions}
\ \longrightarrow\ 
\text{executable consequences}
\ \longrightarrow\ 
\text{bounded realization and control}
$}
\label{eq:main-three-incompletenesses}
\end{equation}
Substantive incompleteness governs which task distinctions enter the effective record; substrate incompleteness governs which represented distinctions acquire executable consequences; high-level incompleteness governs how inherited possibilities are searched, realized, selected, and stopped. These layers generate the operational coordinates of the theory. Appendix~\ref{app:symbolization} develops the philosophical construction in full.

The paper advances three connected contributions. First, it represents an emergence boundary with two inherited/deployed objects: structural capability records the decisions supported by available distinctions and executable semantics, while finite capability records the portion realizable under an external resource contract. Second, Theorem~\ref{thm:main-budget-reversal} identifies the exact object seen by universal same-budget performance: inclusion of the predecessor's closed convex finite envelope is equivalent to bounded-loss dominance over every finite decision problem. Terminal expansion and finite retention can therefore disagree, yielding a strict deployment reversal with a workload-specific witness. Corollary~\ref{cor:main-behavioral-identification} turns this representation into finite-slice measurement, and Theorem~\ref{thm:main-workload-tail-radius} gives the workload-relative information radius beyond finite slices. Third, endpoint-defined successor classes make the boundary itself controllable: Bellman value prices what an intervention realizes now, what finite capability it preserves, and which class later reasoning inherits.

\section{Inherited and finite capability in LLM systems}
\label{sec:llm-capability}

For a single task problem, let $X$ be a task state, $Y$ its target, $\A$ a common semantic action space, and $\ell:\A\times\mathcal Y\to[0,\infty)$ a loss. Universal finite-slice comparisons keep a declared finite task abstraction $\mathcal X$ and semantic action abstraction $\A$ fixed while allowing the finite target space $\mathcal Y$, the task law $Q$ on $\mathcal X\times\mathcal Y$, and the bounded loss to vary. This finite slice is the resolution at which an evaluator asks for envelope recovery; Theorem~\ref{thm:main-workload-tail-radius} extends the measurement question to general workload spaces. An effective interface $\phi:\mathcal X\to\mathcal Z$ records the distinctions that enter the deployed system. Its fibers identify task states that remain observationally identical to every downstream computation that preserves the interface. Executable support $H$ records which interface--action kernels are operationally available through the model, interpreter, tools, permissions, and environment.

\begin{definition}[Inherited structural and finite capability]
\label{def:main-two-level-capability}
For $\mathcal J=(\phi,H)$, let $\Pi_{\phi,H}\subseteq\Ker(\mathcal Z,\A)$ be a nonempty executable kernel class and define
\begin{equation}
\mathcal D_{\mathcal J}
:=\{\kappa\circ\phi:\kappa\in\Pi_{\phi,H}\}
\subseteq\Ker(\mathcal X,\A).
\label{eq:main-structural-class}
\end{equation}
A realization profile $M$ equips $\mathcal D_{\mathcal J}$ with a cost map $c_{\mathcal J,M}$ into a common ordered resource space. At budget $s$,
\begin{equation}
\F_s(\mathcal J,M)
:=\{d\in\mathcal D_{\mathcal J}:c_{\mathcal J,M}(d)\preceq s\}.
\label{eq:main-resource-family}
\end{equation}
Thus $\mathcal D_{\mathcal J}$ is inherited structural capability and $\F_s(\mathcal J,M)$ is finite realized capability. Finite-risk comparisons use budgets with $\F_s(\mathcal J,M)\ne\varnothing$.
\end{definition}

The distinction is operational. Evidence acquisition can refine $\phi$; tool or interpreter expansion can enlarge $H$; both change $\mathcal D_{\mathcal J}$. Reasoning, search, compilation, and resource allocation act through $M$ and $s$ to expose more or less of an inherited class. The same product label can occupy either side: retrieval that filters already accessible evidence is a realization operation, while retrieval that reveals a previously unavailable task distinction changes structural capability.

The philosophical layers project directly into this two-level LLM object:
\begin{equation}
\underbrace{\phi}_{\text{substantive distinction}}
\longrightarrow
\underbrace{(\phi,H)}_{\text{executable structure}}
\longrightarrow
\underbrace{\F_s(\mathcal J,M)}_{\text{bounded realization}}.
\label{eq:main-philosophy-to-coordinates}
\end{equation}
The structural class combines what the system can distinguish with what it can execute. Finite realization then asks which of those decision rules can actually be implemented under the deployed mechanism and resource ceiling.

For a task law $P$, define the structural and finite-budget risk floors
\begin{align}
\R^*_{\mathcal J}(P,\ell)
&:=\inf_{d\in\mathcal D_{\mathcal J}}\mathcal L_{P,\ell}(d),\\
\R^*_{s,\mathcal J,M}(P,\ell)
&:=\inf_{d\in\F_s(\mathcal J,M)}\mathcal L_{P,\ell}(d).
\label{eq:main-two-risks}
\end{align}
Because $\F_s(\mathcal J,M)\subseteq\mathcal D_{\mathcal J}$, every finite realization inherits the structural risk floor of $\mathcal D_{\mathcal J}$.

\begin{theorem}[Symbolic-fiber emergence invariance]
\label{thm:main-fiber-invariance}
If $\phi(x)=\phi(x')$, then every $d\in\mathcal D_{\mathcal J}$ satisfies
\begin{equation}
d(\cdot\mid x)=d(\cdot\mid x').
\label{eq:main-fiber-invariance}
\end{equation}
Consequently, on a balanced two-label twin pair with identical effective interface and opposite targets, every structure-preserving realization has average accuracy at most $1/2$; with a binary semantic action space, the bound is exactly $1/2$. Refining the interface so that the two states separate can move this floor.
\end{theorem}
\begin{proof}
For $d\in\mathcal D_{\mathcal J}$ there is $\kappa\in\Pi_{\phi,H}$ with $d(B\mid x)=\kappa(B\mid\phi(x))$. Equality of $\phi(x)$ and $\phi(x')$ therefore gives equality of the complete conditional action laws. For an equiprobable pair with distinct correct labels $a_0,a_1$, the common action law has average accuracy $\tfrac12[d(a_0\mid x)+d(a_1\mid x)]\le\tfrac12$; equality holds when the semantic action space consists of those two labels.
\end{proof}

The first emergence law follows directly: structure-preserving computation reorganizes distinctions already inherited by the system and remains bounded by distinctions and executable consequences present in that structure. The next section extends this collision law to approximate task drift and then studies structural capability growth.

\section{Static limits of emergence}
\label{sec:static-limits}

\subsection{Structural floor under fixed realization}

The exact twin construction gives a sharp floor at zero drift. A behavioral certificate extends the same logic to changing task laws and arbitrary structure-preserving computation.

\begin{theorem}[Static emergence invariance under bounded drift]
\label{thm:main-bounded-drift}
If $0\le\ell\le L$, $\TV(P_\tau,P_0)\le\varepsilon$, and the terminal decision $d_\tau$ is produced by a trajectory that preserves the inherited task evidence and executable support, then
\begin{equation}
\mathcal L_{P_\tau,\ell}(d_\tau)
\ge
\R^*_{\mathcal J_0}(P_0,\ell)-L\varepsilon.
\label{eq:main-bounded-certificate}
\end{equation}
Hence a gain beyond the drift allowance of the inherited structural floor certifies task-relevant evidence refinement or executable-structure change.
\end{theorem}
\begin{proof}
Structure preservation implies $d_\tau\in\mathcal D_{\mathcal J_0}$, so $\mathcal L_{P_0,\ell}(d_\tau)\ge\R^*_{\mathcal J_0}(P_0,\ell)$. For bounded loss, total variation gives $|\mathcal L_{P_\tau,\ell}(d_\tau)-\mathcal L_{P_0,\ell}(d_\tau)|\le L\TV(P_\tau,P_0)\le L\varepsilon$. Combining the inequalities proves the claim.
\end{proof}

Theorem~\ref{thm:main-bounded-drift} turns static emergence invariance into a deployment-level diagnostic: after pricing drift, any larger improvement certifies movement of task-relevant evidence or executable structure. A hidden-state collision is the equality case; a one-bit reveal moves $\phi$ and releases the floor discontinuously. Executable expansion supplies the action-side analogue: additional reasoning can become effective only after the required mapping enters $H$.

\paragraph{Approximate structural inheritance.}
The exact collision law has a graded total-variation extension: directed deficiency between structural risk envelopes upper-bounds bounded-loss regret under approximate simulation. Appendix~\ref{app:broader-scope} develops this approximation geometry, its triangle law, and its relation to approximate interface collisions.

\subsection{Structural growth and finite deployment}

Structural enrichment moves the terminal decision class. Deployment compares a second object: the policies a successor can realize under the same external resource ceiling. The budget axis is declared independently of the compared architectures---for example tokens, wall-clock latency, memory, energy, tool cost, or a vector of these resources. A realization profile $M$ records how a deployed system converts that common resource into executable decision rules. Two systems can therefore share the same ceiling $s$ while exposing different finite envelopes.

For finite task and action spaces, let $\clco(C)$ denote the closed convex hull of $C$, and write
\begin{equation}
\widehat{\mathcal D}_{\mathcal J}:=\clco(\mathcal D_{\mathcal J}),
\qquad
\widehat{\F}_s(\mathcal J,M):=\clco(\F_s(\mathcal J,M)).
\label{eq:main-convexified-envelopes}
\end{equation}
Closed convexification gives the risk-equivalent envelope because bounded expected loss is continuous and linear in the decision kernel. These four objects form the focused capability square:
\begin{equation}
\boxed{
\begin{array}{ccc}
\mathcal D_{\mathcal J} & \longrightarrow & \widehat{\mathcal D}_{\mathcal J}\\[3pt]
\cup && \cup\\[-2pt]
\F_s(\mathcal J,M) & \longrightarrow & \widehat\F_s(\mathcal J,M).
\end{array}}
\label{eq:main-capability-square}
\end{equation}
The vertical axis separates terminal structural capability from finite deployment capability. The horizontal axis separates literal executability from the risk-equivalent envelope identified by bounded expected-loss comparison. Theorem~\ref{thm:main-budget-reversal} acts on the lower-right object; terminal dominance acts on the upper-right object. Their vertical mismatch is the source of finite-budget reversal.

The deployment question is a resource-indexed representation problem: which capability object is identified by performance across all bounded decision objectives at one common ceiling? The next theorem answers this exactly. Universal same-budget comparison recovers the closed convex finite envelope induced jointly by inherited information--execution structure, the realization profile, and the external resource scale.

\begin{theorem}[Finite-budget capability retention]
\label{thm:main-budget-reversal}
On finite task-state and action spaces, for deployed pairs $(\mathcal J_1,M_1)$ and $(\mathcal J_2,M_2)$ compared at the same budget $s$,
\begin{equation}
\begin{aligned}
&\widehat{\F}_s(\mathcal J_1,M_1)
\subseteq
\widehat{\F}_s(\mathcal J_2,M_2)
\\[-1pt]
&\quad\Longleftrightarrow\quad
\R^*_{s,\mathcal J_2,M_2}(Q,\ell)
\le
\R^*_{s,\mathcal J_1,M_1}(Q,\ell)
\\[-1pt]
&\qquad\text{for every finite target space $\mathcal Y$, every task law $Q$ on}\\[-1pt]
&\qquad\text{$\mathcal X\times\mathcal Y$ whose $X$-marginal $Q_X$ has full support,}\\[-1pt]
&\qquad\text{and every bounded $\ell:\A\times\mathcal Y\to[0,\infty)$.}
\end{aligned}
\label{eq:main-finite-characterization}
\end{equation}
If finite-envelope inclusion fails, then for every prescribed $X$-marginal $q_X$ with full support on finite $\mathcal X$, a bounded nonnegative loss under that workload strictly reverses the finite-budget order. If, simultaneously,
$\widehat{\mathcal D}_{\mathcal J_1}\subseteq\widehat{\mathcal D}_{\mathcal J_2}$,
then system~2 weakly dominates terminal risk on every bounded task while the workload-specific witness ranks it strictly worse at budget $s$.
\end{theorem}
\begin{proof}
Let $C_i=\widehat{\F}_s(\mathcal J_i,M_i)$. Bounded expected risk is a continuous linear functional of a finite decision kernel, so $C_1\subseteq C_2$ gives universal same-budget dominance. Conversely, if $C_1\nsubseteq C_2$, choose $d_1\in C_1\setminus C_2$. Since $C_2$ is closed and convex in a finite-dimensional kernel polytope, strict point--set separation gives a linear functional with value at $d_1$ strictly below its infimum on $C_2$, hence with a strictly smaller infimum on $C_1$. Fix any prescribed full-support $X$-marginal, take the finite target space $\mathcal Y=\mathcal X$ with $Y=X$ deterministically, and use statewise shifts plus one positive rescaling to represent the separating functional as a bounded nonnegative loss, producing the reversal task. The same linear-risk argument on $\widehat{\mathcal D}_{\mathcal J_i}$ gives the terminal statement.
\end{proof}

\begin{takeaway}
\textbf{Finite-envelope retention is the deployment invariant of capability growth.} At a common resource ceiling, retaining the predecessor's closed convex finite envelope exactly characterizes universal improvement across bounded-loss tasks.
\end{takeaway}

\paragraph{Emergence-boundary interpretation.}
The theorem separates two meanings of ``more capable.'' Structural expansion enlarges the terminal risk envelope by admitting new distinctions or executable mappings. Finite retention asks whether the predecessor's effective decisions remain available under the shared deployment contract. Same-budget universal improvement is governed by the finite envelope, while eventual possibility is governed by the structural envelope. Their mismatch pinpoints a deployment boundary at which terminal emergence and finite usability diverge. This gives the static theory its pivot: capability growth can be evaluated by what becomes possible, what remains deployable, and the resource required to restore anything lost in transition.

\begin{corollary}[Finite-slice identification and probe recovery]
\label{cor:main-behavioral-identification}
Fix a declared finite evaluation slice $(\mathcal X,\A)$ and budget $s$. Then
\begin{equation}
\begin{aligned}
\widehat\F_{s,1}=\widehat\F_{s,2}
\quad\Longleftrightarrow\quad&\\[-2pt]
\R^*_{s,1}(Q,\ell)&=\R^*_{s,2}(Q,\ell)
\quad\text{for every }(\mathcal Y,Q,\ell).
\end{aligned}
\label{eq:main-behavioral-identification}
\end{equation}
under the universal task convention of Theorem~\ref{thm:main-budget-reversal}.
Let $n=|\mathcal X|$, $a=|\A|$, and let the slice dimension be $d_{\rm slice}=n(a-1)$. Let $\mathcal K=\Ker(\mathcal X,\A)$, and center the kernel polytope at the uniform kernel $\bar d$. Its radius in the row-sum-zero affine space is
$R=\sqrt{n(1-1/a)}$.
For $C=\widehat\F_s(\mathcal J,M)$ define the lower support value
$\rho_C(u)=\inf_{d\in C}\langle u,d-\bar d\rangle$ for unit directions $u$ in that affine space, and write
$H_u(t)=\{d\in\mathcal K:\langle u,d-\bar d\rangle\ge t\}$. If $U$ is an $\eta$-net and certified probe thresholds satisfy
$0\le \rho_C(u)-\underline\rho_u\le\gamma$, then
\begin{equation}
\begin{aligned}
P_U&:=\bigcap_{u\in U}H_u(\underline\rho_u),\\[-2pt]
C&\subseteq P_U,
& d_{\rm H}(C,P_U)&\le 2R\eta+\gamma .
\end{aligned}
\label{eq:main-hausdorff-tomography}
\end{equation}
An $\eta$-net exists with $m:=|U|\le(1+2/\eta)^{d_{\rm slice}}$.

Fix any full-support workload $q_X$ and let $q_{\min}=\min_x q_X(x)$. Every unit direction $u$ is realized by the $[0,1]$-valued probe with deterministic target $Y=X$ and
$\ell_u(a,x)=\tfrac12+\tfrac{q_{\min}}{2q_X(x)}u_{x,a}$.
If each probe policy is $\epsilon$-optimal and its risk is evaluated on $N$ fresh i.i.d. task draws, let
$\tau_N=\sqrt{\log(2m/\alpha)/(2N)}$ and set
$\underline\rho_u=\frac{2}{q_{\min}}(\widehat r_u-\epsilon-\tau_N-\tfrac12)$.
The resulting empirical outer body satisfies, with probability at least $1-\alpha$,
\begin{equation}
d_{\rm H}(C,P_U)
\le 2R\eta+
\frac{2}{q_{\min}}
\left(\epsilon+2\tau_N\right).
\label{eq:main-tomography-rate}
\end{equation}
Hence $d_{\rm H}(C,P_U)\le\delta$ follows from
\begin{equation}
\begin{aligned}
\eta&=\frac{\delta}{4R},
& m&\le\left(1+\frac{8R}{\delta}\right)^{d_{\rm slice}},\\[-2pt]
N&\ge\frac{2\log(2m/\alpha)}{(q_{\min}\delta/4-\epsilon)^2}.&&
\end{aligned}
\label{eq:main-tomography-sample-complexity}
\end{equation}
whenever $\epsilon<q_{\min}\delta/4$. Conversely, for $d_{\rm slice}\ge2$ there is a constant $c_{d_{\rm slice}}>0$ such that every deterministic adaptive exact-support procedure that guarantees Hausdorff error at most $\delta$ uniformly over closed convex $C\subseteq\mathcal K$ requires
\begin{equation}
m\ge c_{d_{\rm slice}}\,(a\delta)^{-(d_{\rm slice}-1)/2},
\qquad 0<\delta\le (8a)^{-1}.
\label{eq:main-tomography-lower-bound}
\end{equation}
Finally, let $I_k=\operatorname{co}\{d_1,\ldots,d_k\}\subseteq C\subseteq P_U$ be the inner--outer sandwich from elicited executable policies and probes. If $d_{\rm H}(I_k,P_U)\le\delta$, then for every held-out task with $0\le\ell\le L$, writing $R^*_B(Q,\ell)=\inf_{d\in B}\mathcal L_{Q,\ell}(d)$,
\begin{equation}
\begin{aligned}
R^*_{P_U}(Q,\ell)
&\le \R^*_{s,\mathcal J,M}(Q,\ell)
\le R^*_{I_k}(Q,\ell),\\[-2pt]
R^*_{I_k}(Q,\ell)-R^*_{P_U}(Q,\ell)
&\le L\sqrt a\,\delta.
\end{aligned}
\label{eq:main-heldout-risk-certificate}
\end{equation}
\end{corollary}
\begin{proof}
Equality follows by applying Theorem~\ref{thm:main-budget-reversal} in both directions. For the recovery bound, $C$ is closed and convex inside the finite-dimensional kernel polytope, so metric projection exists; project any $x\in P_U$ onto $C$ and choose the separating unit direction from the projection residual. The lower support function is $R$-Lipschitz on the unit sphere, so an $\eta$-net direction loses at most $R\eta$ on the support value and at most $R\eta$ on the test point, giving $d_{\rm H}(C,P_U)\le2R\eta+\gamma$. The standard volumetric sphere-net bound gives $m\le(1+2/\eta)^{d_{\rm slice}}$. Because every $u$ has row sum zero, the stated loss satisfies $0\le\ell_u\le1$ and
$\mathcal L_u(d)=\tfrac12+\tfrac{q_{\min}}2\langle u,d-\bar d\rangle$. An $\epsilon$-optimal probe and a uniform Hoeffding event over $m$ fresh evaluations give support-threshold slack at most $2(\epsilon+2\tau_N)/q_{\min}$. Substitution proves Eq.~\eqref{eq:main-tomography-rate}; Eq.~\eqref{eq:main-tomography-sample-complexity} allocates half of $\delta$ to angular discretization and half to probe error.
For the lower bound, $\mathcal K-\bar d$ contains the Euclidean ball of radius $1/a$ in the row-sum-zero space. Set $r=1/(2a)$ and run any deterministic adaptive exact-support procedure on $C_0=rB$. Its transcript fixes at most $m$ queried directions. If spherical caps of angular radius $\theta=\arccos(r/(r+3\delta))$ around those directions fail to cover the unit sphere, choose an uncovered $u$ and set $C_1=\operatorname{co}(C_0,\{-(r+3\delta)u\})$. For $\delta\le(8a)^{-1}$ both sets lie in $\mathcal K$, their support answers agree on every queried direction, and $d_{\rm H}(C_0,C_1)=3\delta$. A spherical-cap volume bound therefore forces $m\ge c_{d_{\rm slice}}(a\delta)^{-(d_{\rm slice}-1)/2}$ for uniform error $\delta$.
For the held-out task, the coefficient vector of the linear risk functional has Euclidean norm at most $L\sqrt a$. The Hausdorff bound therefore changes its optimum by at most $L\sqrt a\,\delta$; nesting $I_k\subseteq C\subseteq P_U$ gives the displayed interval.
\end{proof}

\begin{theorem}[Workload-tail information radius]
\label{thm:main-workload-tail-radius}
Let $(\mathcal X,\Sigma,q)$ be a standard probability space and $\A=\{0,1\}$. For measurable $t:\mathcal X\to[-1,1]$, define the singleton envelope $C_t=\{d_t\}$ by
\[
d_t(1\mid x)=\frac{1+t(x)}2,\qquad d_t(0\mid x)=\frac{1-t(x)}2,
\]
and equip kernels with workload-average distance
\[
d_q(d,e):=\int_{\mathcal X}\TV\!\left(d(\cdot\mid x),e(\cdot\mid x)\right)dq(x).
\]
Let $p_{(1)}\ge p_{(2)}\ge\cdots$ be the atomic masses of $q$ and
$\tau_m(q):=1-\sum_{i=1}^m p_{(i)}$, with all non-atomic mass retained in the residual. The minimax worst-case error of reconstructing $C_t$ from $m$ deterministic adaptive exact bounded-loss probes is
\begin{equation}
\mathfrak R_m^{\rm sing}(q)=\frac12\,\tau_m(q).
\label{eq:main-workload-tail-radius}
\end{equation}
Hence a countably atomic workload has exact residual radius
$\frac12\sum_{i>m}p_{(i)}$, while a non-atomic workload has
$\mathfrak R_m^{\rm sing}(q)=1/2$ for every finite $m$; a non-atomic component of mass $\beta$ yields radius at least $\beta/2$.
\end{theorem}
\begin{proof}
Every bounded-loss probe on the binary singleton family is an affine functional of $t$. For the upper bound, query the values of $t$ on the $m$ largest atoms with binary deterministic targets, copy those coordinates, and set the residual to zero; the resulting $d_q$ error is at most $\tau_m(q)/2$. For the lower bound, run an arbitrary adaptive procedure on $t=0$, which fixes $m$ linear probe functionals. Partition the probability space so that the $m$ largest cells carry at most $\sum_{i\le m}p_{(i)}+\varepsilon$ mass, restrict $t$ to be cellwise constant, and intersect the resulting cube with the common kernel of the $m$ functionals. A vertex has at most $m$ unsaturated coordinates, producing a transcript-null $t$ with $\|t\|_{L_1(q)}\ge\tau_m(q)-\varepsilon$. The indistinguishable pair $t,-t$ is separated by $d_q(d_t,d_{-t})=\|t\|_{L_1(q)}$, so one member lies at least $(\tau_m(q)-\varepsilon)/2$ from the common reconstruction. Let $\varepsilon\downarrow0$.
\end{proof}

\begin{corollary}[Countable head--tail recovery]
\label{cor:main-countable-truncation}
Let $q=\sum_{i\ge1}p_i\delta_{x_i}$ with
$p_1\ge p_2\ge\cdots>0$ and $|\A|=a<\infty$. For a closed envelope $C$, let
$T_N(q)=\sum_{i>N}p_i$ and let $\pi_NC$ be its projection onto the first $N$ task states. If an outer head approximation $P_N$ satisfies
\[
\pi_NC\subseteq P_N,
\qquad
d_H^{(2)}(\pi_NC,P_N)\le\varepsilon_N,
\]
and $\widetilde P_N:=\{d:\pi_Nd\in P_N\}$, then under workload-average total variation,
\begin{equation}
C\subseteq\widetilde P_N,
\qquad
d_{H,q}(C,\widetilde P_N)
\le
T_N(q)+\frac{\sqrt a}{2}\varepsilon_N .
\label{eq:main-countable-truncation}
\end{equation}
Conditioned on the retained head, Corollary~\ref{cor:main-behavioral-identification} applies with effective dimension
$d_N=N(a-1)$ and smallest conditional workload mass
$q_{\min,N}=p_N/(1-T_N(q))$. Moreover, for every held-out decision problem with $X$-marginal $q$ and $0\le\ell\le L$,
\begin{equation}
0\le
R_C^*(Q,\ell)-R_{\widetilde P_N}^*(Q,\ell)
\le
L\!\left(T_N(q)+\frac{\sqrt a}{2}\varepsilon_N\right).
\label{eq:main-countable-risk-transfer}
\end{equation}
\end{corollary}
\begin{proof}
Only the direction from $\widetilde P_N$ back to $C$ requires control. Fix $d\in\widetilde P_N$ and $\zeta>0$. By the Hausdorff bound, there exists $c\in C$ whose first $N$ rows satisfy
$\|\pi_Nd-\pi_Nc\|_2\le\varepsilon_N+\zeta$.
The unresolved tail contributes at most $T_N(q)$. On the retained head,
\[
\sum_{i\le N}p_i\TV(d_i,c_i)
\le \frac{\sqrt a}{2}\|\pi_Nd-\pi_Nc\|_2
\le \frac{\sqrt a}{2}(\varepsilon_N+\zeta).
\]
Adding head and tail and then letting $\zeta\downarrow0$ gives Eq.~\eqref{eq:main-countable-truncation}. Conditioning on the retained head normalizes the masses to $p_i/(1-T_N(q))$ and produces the stated finite-slice parameters. For $0\le\ell\le L$, expected loss is $L$-Lipschitz in workload-average total variation; nesting $C\subseteq\widetilde P_N$ and the Hausdorff bound give Eq.~\eqref{eq:main-countable-risk-transfer}.
\end{proof}

\paragraph{Evaluation geometry across scales.}
Corollary~\ref{cor:main-behavioral-identification} gives shape recovery inside a declared finite task/action abstraction; Theorem~\ref{thm:main-workload-tail-radius} gives the complementary coverage law on open-ended workloads; Corollary~\ref{cor:main-countable-truncation} joins the two. Define the effective workload size
\[
N_q(\delta):=\min\{N:T_N(q)\le\delta\}.
\]
Theorem~\ref{thm:main-workload-tail-radius} makes $N_q(2\delta)$ necessary and sufficient for $\delta$-accurate singleton recovery on a countable workload. For general envelopes, one chooses a head with $T_N(q)\le\delta/2$ and recovers that finite head to $\varepsilon_N\le\delta/\sqrt a$. Thus evaluation complexity decomposes into workload coverage and finite-head shape recovery; polynomial and geometric workload tails induce polynomial and logarithmic effective head sizes, respectively.

\paragraph{Finite evaluation and held-out prediction.}
The inner--outer gap $d_{\rm H}(I_k,P_U)$ is a data-dependent stopping certificate. Corollary~\ref{cor:main-behavioral-identification} transfers that certificate to decision objectives absent from probe construction. Across budgets it estimates the curve $s\mapsto\widehat\F_s$; matched evidence, support, realization-profile, and ceiling interventions then attribute envelope movement to $\phi$, $H$, $M$, or $s$.

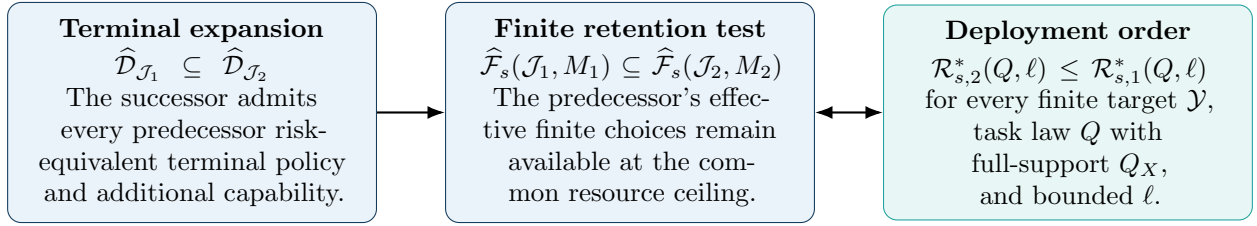
\begin{figure*}[t]
\centering
\begin{tikzpicture}[>=Latex,
  panel/.style={draw=navy,rounded corners=1.8mm,fill=paleblue,align=center,inner sep=6pt,text width=0.27\textwidth,font=\small},
  result/.style={draw=teal,rounded corners=1.8mm,fill=paleteal,align=center,inner sep=6pt,text width=0.27\textwidth,font=\small}]
\node[panel] (terminal) {\textbf{Terminal expansion}\\[2pt]
$\widehat{\mathcal D}_{\mathcal J_1}\subseteq\widehat{\mathcal D}_{\mathcal J_2}$\\
The successor admits every predecessor risk-equivalent terminal policy and additional capability.};
\node[panel,right=9mm of terminal] (finite) {\textbf{Finite retention test}\\[2pt]
$\widehat\F_s(\mathcal J_1,M_1)\subseteq\widehat\F_s(\mathcal J_2,M_2)$\\
The predecessor's effective finite choices remain available at the common resource ceiling.};
\node[result,right=9mm of finite] (order) {\textbf{Deployment order}\\[2pt]
$\R^*_{s,2}(Q,\ell)\le\R^*_{s,1}(Q,\ell)$\\
for every finite target $\mathcal Y$,\\ task law $Q$ with full-support $Q_X$,\\ and bounded $\ell$.};
\draw[->,thick] (terminal) -- (finite);
\draw[<->,thick] (finite) -- (order);
\end{tikzpicture}
\caption{\textbf{The finite-budget deployment order.} Terminal capability and finite deployment capability are separate orders. Theorem~\ref{thm:main-budget-reversal} identifies finite-envelope retention as the exact condition for universal same-budget dominance. Under terminal expansion, failure of the finite retention test yields a bounded-loss reversal witness.}
\label{fig:finite-retention-order}
\end{figure*}

\paragraph{Decision representation and reversal.}
Theorem~\ref{thm:main-budget-reversal} identifies the complete same-budget decision object generated by $(\mathcal J,M,s)$: the closed convex finite envelope. Two deployed systems with the same envelope share every bounded-loss optimum at that budget; every missing predecessor point produces a strict bounded-loss witness under any prescribed full-support workload. Terminal order on $\widehat{\mathcal D}_{\mathcal J}$ is a separate comparison, so terminal expansion and finite deployment order can point in opposite directions. The reversal therefore expresses capability compatibility across an upgrade and is exposed by workload-specific decision witnesses.

A minimal reversal makes the distinction concrete. Let $X=(U,V)$ be uniform on $\{0,1\}^2$ and $Y=U\oplus V$. The interface $\phi(X)=U$ has structural risk floor $1/2$, while $\psi(X)=(U,V)$ admits zero terminal risk. At one common budget, let the predecessor realize the constant-zero rule with risk $1/2$ and the refined successor realize only the complement of XOR with risk $1$. Terminal capability expands while finite-budget risk reverses.

\paragraph{Exact restoration resource.}
For a scalar external resource axis, define the least target budget that restores the predecessor's risk-equivalent finite capability,
\begin{equation}
T^{\rm risk}_{1\to2}(s)
:=\inf\{t\ge0:\widehat\F_s(\mathcal J_1,M_1)
\subseteq\widehat\F_t(\mathcal J_2,M_2)\}.
\label{eq:main-resource-restoration}
\end{equation}
Writing $\R^*_{t,i}:=\R^*_{t,\mathcal J_i,M_i}$, Theorem~\ref{thm:main-budget-reversal} gives the equivalent decision representation
\begin{equation}
\resizebox{0.96\columnwidth}{!}{$
\begin{aligned}
T^{\rm risk}_{1\to2}(s)
=\inf\{t\ge0:\;&\R^*_{t,2}(Q,\ell)\le\R^*_{s,1}(Q,\ell)\\[-2pt]
&\text{for every finite target space $\mathcal Y$, task law $Q$ on $\mathcal X\times\mathcal Y$,}\\[-2pt]
&\text{with full-support $Q_X$ and every bounded $\ell:\A\times\mathcal Y\to[0,\infty)$}\}.
\end{aligned}
$}
\label{eq:main-resource-restoration-risk}
\end{equation}
Thus the same closed convex envelope yields both an exact same-budget order and the resource threshold at which a successor recovers the predecessor's universal bounded-loss capability. Retained fallbacks, compatibility routes, cached solvers, and low-cost tool paths implement this geometry directly by carrying old finite choices into a richer successor.

\paragraph{Literal and risk-equivalent restoration.}
The same upgrade can be measured at two strengths. Define
\begin{align}
T^{\rm exec}_{1\to2}(s)
&:=\inf\{t:\F_s(\mathcal J_1,M_1)\subseteq\F_t(\mathcal J_2,M_2)\},\\
T^{\rm risk}_{1\to2}(s)
&:=\inf\{t:\widehat\F_s(\mathcal J_1,M_1)\subseteq\widehat\F_t(\mathcal J_2,M_2)\}.
\label{eq:main-two-restoration-maps}
\end{align}
Literal policy simulation implies risk-equivalent simulation, hence
\begin{equation}
T^{\rm risk}_{1\to2}(s)\le T^{\rm exec}_{1\to2}(s).
\label{eq:main-restoration-order}
\end{equation}
The gap between the two curves measures the resource value of risk-equivalent replacement: a successor may recover every bounded-loss objective before it can literally reproduce every predecessor policy.

\paragraph{Approximate finite retention.}
Exact inclusion has a quantitative relaxation on the same total-variation scale. Let $\delta_{Q_X,s}(1\!\to\!2)$ denote the directed expected-TV deficiency from $\widehat\F_s(\mathcal J_1,M_1)$ to $\widehat\F_s(\mathcal J_2,M_2)$, defined formally in \hyperref[app:proofs]{Appendix Module I}. For $0\le\ell\le L$,
\begin{equation}
\R^*_{s,2}(Q,\ell)
\le
\R^*_{s,1}(Q,\ell)
+L\,\delta_{Q_X,s}(1\!\to\!2).
\label{eq:main-finite-deficiency-risk}
\end{equation}
With full-support $Q_X$, closed-envelope inclusion is exactly the zero-deficiency case; positive deficiency measures the bounded-loss price of approximate compatibility. \hyperref[app:proofs]{Appendix Module I} gives the formal definition and composition properties.

\paragraph{Upgrade-chain geometry.}
Restoration maps compose across successive architecture changes, with affine bounds separating multiplicative realization overhead from startup burden. Appendix~\ref{app:broader-scope} gives the composition law and literal/risk-equivalent variants.

\subsection{Why the static limits induce control}

The two static results generate a directional control problem. Structure-preserving realization remains bounded by the inherited structural floor, while structural enrichment creates a new successor realization pair with its own finite envelope. An adaptive LLM therefore chooses a capability direction at every nonterminal state:
\begin{equation}
\resizebox{0.98\columnwidth}{!}{$
\boxed{\text{realize the inherited class}}
\ \text{or}\ 
\boxed{\text{construct a successor class}}
$}
\label{eq:main-control-dilemma}
\end{equation}
Structural construction carries successor value. A retrieved distinction can make later reasoning useful; an enabled tool can create policies that become reachable after further computation. The successor class becomes the state space inherited by future control. This is the dynamic problem developed next.

\section{Dynamic control over inherited capability}
\label{sec:dynamic-core}

\subsection{Endpoint semantics: realization versus construction}

Control direction is defined by the structural decision class at the realized endpoint. Retrieval, tools, memory, clarification, routing, and reasoning can therefore occupy either direction depending on the successor they create. For a non-stop transition $S\xrightarrow{u}S'$, define
\begin{definition}[Endpoint structural novelty]
\label{def:main-endpoint-novelty}
\begin{align}
\chi_C^{\rm end}(S,u,S')
&:=\mathbf 1\!\left\{
\mathcal D_{\mathcal J(S')}
\nsubseteq
\mathcal D_{\mathcal J(S)}
\right\},\\
p_C(S,u)
&:=\mathbb E[\chi_C^{\rm end}(S,u,S')\mid S,u].
\label{eq:main-endpoint-novelty}
\end{align}
A non-stop action is \emph{Ockhamian} when $p_C(S,u)=0$: its successor remains within inherited structural possibility. It is \emph{Chattonian-containing} when $p_C(S,u)>0$: it can construct a successor decision class containing capability outside inherited simulation.
\end{definition}

This endpoint definition is state-relative. Retrieval is Ockhamian when it searches evidence already admitted by the effective interface and Chattonian-containing when it exposes a new task-relevant distinction. A tool call is Ockhamian when it selects an enabled mapping and Chattonian-containing when it changes executable support. The terminology names the two directions generated by the formal endpoint relation \cite{ockham,chatton}.

\begin{figure}[t]
\centering
\begin{tikzpicture}[
  >={Latex[length=2.2mm,width=1.6mm]},
  class/.style={draw=navy,rounded corners=1.5mm,fill=paleblue,align=center,
    inner xsep=4.5pt,inner ysep=4pt,text width=0.23\linewidth,minimum height=8mm,font=\small},
  act/.style={draw=teal,rounded corners=1.5mm,fill=paleteal,align=center,
    inner xsep=4.5pt,inner ysep=4pt,text width=0.23\linewidth,minimum height=8mm,font=\small},
  flow/.style={line width=0.8pt}]
\node[class] (dj) {$\mathcal D_{\mathcal J}$\\[-1pt]{\scriptsize inherited class}};
\node[act,right=0.28\linewidth of dj] (succ) {$\mathcal D_{\mathcal J'}$\\[-1pt]{\scriptsize successor class}};
\node[class,below=7.5mm of dj] (fs) {$\F_s(\mathcal J,M)$\\[-1pt]{\scriptsize finite realization}};
\node[act,right=0.28\linewidth of fs] (real) {\textbf{realize}\\[-1pt]{\scriptsize within $\mathcal D_{\mathcal J}$}};

% Structural construction and inheritance use separated, parallel straight arrows.
\draw[flow,->] ([yshift=2.3mm]dj.east) --
  node[midway,above=1.2mm,font=\scriptsize,text=muted]{construct}
  ([yshift=2.3mm]succ.west);
\draw[flow,<-] ([yshift=-2.3mm]dj.east) --
  node[midway,below=1.2mm,font=\scriptsize,text=muted]{inherit}
  ([yshift=-2.3mm]succ.west);

% Finite realization and within-class control.
\draw[flow,->] (dj) --
  node[midway,left=1.8mm,font=\scriptsize,align=right,text=muted]{budget\\[-1pt]mechanism}
  (fs);
\draw[flow,->] (fs) -- (real);
\end{tikzpicture}
\caption{\textbf{Focused capability-control loop.} The structural class $\mathcal D_{\mathcal J}$ is inherited; a deployed realization pair exposes $\F_s(\mathcal J,M)$. Control either realizes more of the inherited class or constructs a successor structural class, which becomes the capability inherited by later decisions. Appendix~\ref{app:broader-scope} gives the full hierarchy over asymptotic realization, reachability, selection, and meta-control.}
\label{fig:focused-loop}
\end{figure}
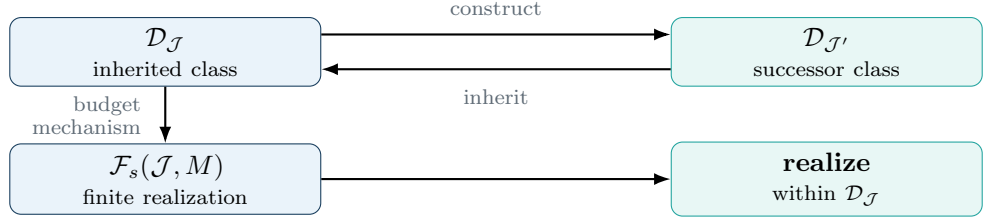

\subsection{Successor capability as a Bellman state}

Let $\R_{\rm reach}(S)$ denote the best loss among decisions currently reachable if the controller stops. For a non-stop action with successor $S'$,
\begin{align}
Q^*(S,u)
&=c(S,u)+\mathbb E[V^*(S')\mid S,u],\\
V^*(S)
&=\min\!\left\{\R_{\rm reach}(S),
\inf_{u\in\mathcal U(S)}Q^*(S,u)\right\}.
\label{eq:main-bellman}
\end{align}
The state is class-valued in the relevant sense: an action can change $\mathcal J(S')$ and hence the structural decision class on which all subsequent realization operates. Structural actions therefore carry option value. Evidence can create value through later computation; executable expansion can create a rule whose benefit appears after search or verification. Bellman value compares these trajectories on one scale.

These static limits sharpen the Ockham--Chatton dilemma. Static invariance says when continued realization has zero power against the active floor. Finite-budget retention says why an outward structural move must also be valued through the deployability of its successor. Dynamic control combines the two: spend on inherited capability while its realizable value remains high, and switch direction when changing the inherited boundary has greater complete trajectory value.

A two-step example shows why successor value is essential. Suppose the decisive bit lies outside $\phi$. A THINK action searches inside $\mathcal D_{\mathcal J}$ and remains above the structural floor. A REVEAL action changes the interface and creates $\mathcal D_{\mathcal J'}$. The next THINK action then operates inside $\mathcal D_{\mathcal J'}$ and can exploit the revealed distinction. The Bellman value $Q^*$ prices the capability inherited by this continuation, including gains that appear after the structural action.

Theorem~\ref{thm:main-budget-reversal} adds a second requirement to this trajectory view. A structural successor may unlock a lower terminal floor while dropping a useful finite policy. Bellman control therefore compares the complete successor realization pair together with the structural class it creates. This makes the dynamic state genuinely capability-valued: the same nominal action can be preferred or rejected depending on both the class it constructs and the finite envelope that the successor can realize.

\subsection{Saturation and persistent successor value}
For a scalar realization budget, fix $(\mathcal J,M,Q,\ell)$ and write
\begin{equation}
R(s):=\R^*_{s,\mathcal J,M}(Q,\ell),
\qquad
R_\infty:=\inf_{t\ge0}R(t).
\label{eq:main-saturation-risk}
\end{equation}
Nested finite envelopes make $R(s)$ nonincreasing and bounded below. Hence every fixed extra realization increment $h>0$ has vanishing marginal value:
\begin{equation}
\Delta_h(s):=R(s)-R(s+h)\longrightarrow0
\qquad (s\to\infty).
\label{eq:main-vanishing-realization-gain}
\end{equation}

\begin{proposition}[Persistent positive successor value dominates fixed realization increments]
\label{prop:main-eventual-boundary}
Let a boundary-changing action $c$ available at budget-indexed states $S_s$ have complete Bellman gain relative to stopping
\begin{equation}
G_C(s):=R(s)-\left(c(S_s,c)+\E[V^*(S'_s)\mid S_s,c]\right).
\label{eq:main-structural-gain}
\end{equation}
If $\liminf_{s\to\infty}G_C(s)>0$, then for every fixed $h>0$ there exists $s_0$ such that
\begin{equation}
G_C(s)>\Delta_h(s),\qquad s\ge s_0.
\label{eq:main-eventual-boundary-dominance}
\end{equation}
Thus a successor with persistent positive complete value eventually dominates spending another fixed increment of resource on realization inside the inherited class.
\end{proposition}
\begin{proof}
Nestedness gives $R(s)\downarrow R_\infty$, so both $R(s)$ and $R(s+h)$ converge to the same limit and $\Delta_h(s)\to0$. Choose $\eta>0$ below $\liminf_s G_C(s)$. For all sufficiently large $s$, $G_C(s)>\eta$ and $\Delta_h(s)<\eta$, proving the comparison.
\end{proof}

The proposition derives saturation from the finite-capability geometry itself. Its remaining premise has direct successor semantics: after action cost and future realization are priced, the constructed class continues to offer positive trajectory value. Theorem~\ref{thm:main-budget-reversal} determines whether that successor also retains predecessor finite capability, while restoration geometry measures the resource required when retention arrives later. Single-crossing or increasing-differences conditions sharpen eventual dominance into a unique threshold, and switching costs widen that threshold into hysteresis; Appendix~\ref{app:predictions} gives those refinements.

\paragraph{Two dynamic regimes.}
The saturation law yields an observable separation between realization-limited and boundary-limited trajectories. In a realization-limited regime, measured increments $R(s)-R(s+h)$ remain material and continued reasoning or search has direct finite value. Along a saturating path those increments collapse toward zero. A boundary-changing action enters the late-stage regime when its complete successor gain remains bounded away from zero after action cost and downstream realization are included. Restoration geometry then predicts how quickly that successor can carry predecessor finite capability forward. The empirical test is therefore carried by two measurable curves---marginal realization gain and successor restoration/value---whose crossing determines the transition.

\subsection{Recursive limit: the controller also inherits information}
\label{sec:main-recursive-limit}
Boundary control inherits the same information principle. Let $\mu(S)$ be the controller observation used to choose a direction. For two equiprobable states $S_0,S_1$ with $\mu(S_0)=\mu(S_1)$, the controller must use the same action distribution. If their optimal action sets are disjoint, every $\mu$-measurable controller has average optimal-action error at least
\begin{equation}
\R_{\rm dir}^*\ge\frac12.
\label{eq:main-meta-collision-floor}
\end{equation}
A controller observation that separates the states releases this floor. For a diagnostic action with induced probe laws $K_0^a,K_1^a$, the exact post-probe direction error is
\begin{equation}
\R_{\rm dir}^*(a)=\frac{1-\TV(K_0^a,K_1^a)}{2},
\label{eq:main-probe-error}
\end{equation}
and a wrong-direction cost $\lambda$ gives one-step diagnostic value
\begin{equation}
\operatorname{VoD}(a)=\frac{\lambda}{2}\TV(K_0^a,K_1^a)-c(a).
\label{eq:main-probe-value}
\end{equation}
The pair of formulas gives a recursive information law for control: task-level distinctions determine which action is useful, while meta-level distinctions determine whether the controller can recognize that direction. Additional deliberation through the same $\mu$ preserves the collision; diagnostic value scales exactly with the separation created by the probe. Appendix~\ref{app:predictions} develops this prediction, and Appendix~\ref{app:meta} develops substrate termination and the broader human-control comparison.

\paragraph{Meta-interface refinement.}
The recursive layer admits its own monotonicity law. Let a richer controller observation $\nu(S)$ refine $\mu(S)$ through $\mu=r\circ\nu$, and suppose every policy measurable through $\mu$ can be reproduced by a $\nu$-measurable controller. Then the refined controller policy class contains the original one and
\begin{equation}
V_\nu^*(S)\le V_\mu^*(S).
\label{eq:main-meta-refinement}
\end{equation}
When the additional signal in $\nu$ is reserved for meta-action selection, the task interface $\phi$ and executable support $H$ remain fixed, so the task-level structural floor $\R_{\mathcal J}^*$ is preserved while intervention value can improve. Making the same signal available to the terminal semantic decision promotes it into task evidence and can move $\phi$. The framework therefore distinguishes two uses of information within one recursive architecture: information that improves boundary selection and information that enlarges the boundary being selected over. This level-relative distinction gives a precise stopping point for the recursion whenever the current controller interface already separates the action families relevant to the state.

\subsection{Retention-aware boundary control}

The two risk floors locate the active direction relative to a target risk $\rho$. If
\begin{equation}
\R^*_{\mathcal J}(P,\ell)>\rho,
\label{eq:main-target-structural}
\end{equation}
then every structure-preserving trajectory remains above the target and a task distinction or executable consequence must change. If
\begin{equation}
\R^*_{\mathcal J}(P,\ell)\le\rho
<\R^*_{s,\mathcal J,M}(P,\ell),
\label{eq:main-target-realization}
\end{equation}
then the target lies inside inherited structural capability and outside current finite realization, directing control toward computation, search, realization redesign, or additional authorization.

When control moves structurally, Theorem~\ref{thm:main-budget-reversal} supplies the task-universal preservation criterion and $Q^*$ supplies the state-specific trajectory value. Envelope retention protects the predecessor's bounded-loss capability across workloads; Bellman value can still favor a nonretaining successor when its downstream task value compensates for the realization burden and action cost. The resulting control state prices three quantities together: the successor class created, the finite policies preserved, and the later reasoning enabled.

The two floors therefore induce three target-relative regimes:
\begin{center}
\small
\begin{tabular}{@{}p{2.25cm}p{2.55cm}p{2.25cm}@{}}
\toprule
Target location & Active limitation & Control direction \\
\midrule
$\rho<\R^*_{\mathcal J}$ & inherited structure & refine $\phi$ or expand $H$ \\
$\R^*_{\mathcal J}\le\rho<\R^*_{s,\mathcal J,M}$ & finite realization & compute, redesign $M$, or raise $s$ \\
$\R^*_{s,\mathcal J,M}\le\rho$ & finite-feasible target & reach, select, or stop \\
\bottomrule
\end{tabular}
\end{center}
The table gives the focused decision rule used in the main text. Appendix~\ref{app:broader-scope} refines the third regime into asymptotic realization, current reachability, and deployed selection.

\subsection{Predictions from the core theory}
\label{sec:main-predictions}

The static characterization and successor-class control state generate four direct signatures that can be tested across future model, context, retrieval, and tool upgrades.

\paragraph{Workload-stable reversal.}
For every prescribed full-support $X$-marginal, finite-envelope nonretention yields a bounded objective that reverses the same-budget order; retention or additional resource restores the predecessor envelope while terminal capability can remain expanded.

\paragraph{Restoration curves and upgrade overhead.}
$T^{\rm risk}_{1\to2}(s)$ and $T^{\rm exec}_{1\to2}(s)$ separate risk-equivalent recovery from literal simulation, turning compatibility into a resource curve whose composition tracks overhead across successive upgrades.

\paragraph{Saturation with persistent successor value.}
Nested finite realization makes every fixed extra resource increment vanish in risk value at saturation. Proposition~\ref{prop:main-eventual-boundary} then makes persistent positive successor value eventually dominate that fixed increment; single crossing and switching costs refine the comparison into a threshold and hysteresis band.

\paragraph{Recursive diagnostic scaling.}
When states requiring different action families collide at $\mu$, direction error remains at the collision floor; diagnostic probes release it according to induced total variation, with Eq.~\eqref{eq:main-probe-value} pricing the one-step value of that separation.

Appendix~\ref{app:predictions} develops these signatures together with switching-cost hysteresis and cross-level control imbalance; Appendix~\ref{app:meta} develops the human--LLM substrate comparison.

\section{Resulting capability system}
The system links inherited structure, finite deployment, measurement, and successor control. Theorem~\ref{thm:main-budget-reversal} gives the exact finite order; Corollary~\ref{cor:main-behavioral-identification} gives finite-slice recovery; Theorem~\ref{thm:main-workload-tail-radius} gives the workload-relative information radius beyond finite slices; restoration measures compatibility; and Bellman value prices the successor class together with its finite realization. Proposition~\ref{prop:main-eventual-boundary} and the recursive diagnostic law then determine when another fixed realization increment loses to persistent successor value and whether the controller can recognize that direction.

Measurement consequently has two complexity layers. The workload tail controls which task distinctions carry enough probability mass to enter a finite evaluation geometry; conditional on that retained head, $d_{\rm slice}$ and the smallest retained workload mass control directional and statistical shape recovery. Countable truncation makes these layers explicit, while non-atomic mass leaves a positive finite-information radius on the unrestricted singleton subclass.

\section{Related work}
\label{sec:relation-existing}
Relevant lines include emergence and test-time scaling \cite{wei2022,schaeffer2023,snell2025}; decision comparison and metareasoning \cite{blackwell1953,lecam1964,lieder2017}; capability evaluation and finite-information recovery \cite{hofstatter2025,brown2025task,traub1988}; and retrieval, tool use, and routing \cite{lewis2020,yao2023,schick2023,routellm2025}. Detailed theorem-level positioning appears in \hyperref[app:core-comparison]{Appendix Module V}, after the proofs, formal extensions, and experiments; that section states the theoretical increment of the representation, finite-budget order, measurement, restoration, and successor-class control results.

\section{Conclusion}
The emergence boundary separates additional realization from changes in inherited distinctions or executable semantics. Theorem~\ref{thm:main-budget-reversal} makes this operational: the closed convex finite envelope is the complete invariant of universal bounded-loss performance at a common resource ceiling, so terminal expansion can coexist with strict same-budget reversal when the successor drops predecessor capability from its finite envelope. Measurement recovers this object on finite slices and quantifies unresolved workload beyond them. Restoration prices compatibility, and successor-class Bellman value prices later inherited capability, organizing emergence around inheritance and boundary control.

\begin{textAtEnd}[category=foundation]
\section{Symbolization as a partial symbolic substructure}
\label{app:symbolization}

\subsection{Language as a consequence-preserving projection}

Fix a context $c$. Let $\mathcal T$ be a space of cognitive contents and $\Sigma^*$ the finite strings over an alphabet. Expression in that context is a partial map
\begin{equation}
E_c:\mathcal T\rightharpoonup\Sigma^*,\qquad
\mathcal T_{{\rm sym},c}=\operatorname{dom}(E_c)\subseteq\mathcal T.
\end{equation}
Let $\approx_c$ identify contents that are equivalent for the task in context $c$. We require it to be a congruence for the cognitive consequence relation: replacing premises or conclusions by $\approx_c$-equivalent contents does not change $\vdash_T$. Hence $\vdash_T$ is well-defined on $\mathcal T_{{\rm sym},c}/\!\approx_c$. For $M_c:\mathcal L\to\mathcal T_{{\rm sym},c}/\!\approx_c$, write $M_c(\Gamma)=\{M_c(\gamma):\gamma\in\Gamma\}$.

\begin{definition}[Symbolization--Substructure Thesis]
A language $(\mathcal L,\vdash_{\mathcal L})$ is a symbolic substructure of thought when:
\begin{align}
M_c(E_c(t))&=[t]_{\approx_c},\label{eq:realization}\\
\Gamma\vdash_{\mathcal L}s&\Longrightarrow M_c(\Gamma)\vdash_T M_c(s).\label{eq:sound-symbolization}
\end{align}
Equation~\eqref{eq:realization}, required whenever $t\in\mathcal T_{{\rm sym},c}$ and $E_c(t)\in\mathcal L$, makes language a context-indexed partial symbolic image. Equation~\eqref{eq:sound-symbolization} makes valid linguistic inference sound with respect to the task-relative inferential organization of the represented thought.
\end{definition}

Partiality, executability, and consequence preservation play distinct roles in the theory. Partiality determines which task-relevant distinctions survive into the effective quotient $\phi$. Executability determines the realizable rule class $\Pi_{\phi,H}$. Consequence preservation supplies the normative criterion operationalized through task-relative validity and loss. The interface determines what the system can distinguish, executable support determines what it can do, and task loss evaluates those actions against the represented inferential structure.

A language used as a vehicle of thought can be modeled, relative to a fixed context and task, as a formalized region of thought. Its expressions realize cognitive contents, its composition makes those contents publicly manipulable, and its consequence relation preserves a task-relevant part of their inferential organization. The distinction between form and meaning motivates partial symbolic preservation \cite{bender2020}; the distinction between formal and functional competence motivates the separation of symbolic fluency from world-directed ability \cite{mahowald2024}; and neurocognitive dissociations between language and non-linguistic thought motivate the architectural layering \cite{fedorenko2024}.

For an artificial system the inheritance chain is longer:
\begin{equation}
W\rightarrow T\rightarrow S\rightarrow D\rightarrow Z\rightarrow R\rightarrow A,
\label{eq:chain}
\end{equation}
where $W$ is the world, $T$ cognitive content, $S$ symbolized expression, $D$ recorded context, $Z$ representation, $R$ computation, and $A$ action.

\begin{proposition}[Information inheritance]
Assume the joint law factorizes according to Eq.~\eqref{eq:chain}, and fix its downstream Markov kernels. Then
\begin{align}
I(W;A)&\le I(W;R)\le I(W;Z)\le I(W;D),\\
I(W;D)&\le I(W;S)\le I(W;T).
\end{align}
If two conditional source laws---for example, two world or hypothesis conditions---induce the same law over $S$, then under the fixed downstream kernels they induce the same law over $A$.
\end{proposition}
\begin{proof}
Apply the data-processing inequality successively. Equality of the law at $S$ is preserved by composition with the same kernels from $S$ to $A$.
\end{proof}

This proposition supplies the bridge from philosophy to architecture: downstream computation can reorganize inherited distinctions; observation, recording, memory, and representation determine which distinctions enter the chain.

The inheritance chain has a quotient interpretation. The effective symbolic record identifies world states that produce the same downstream-accessible distinctions. For a deterministic interface, define
\begin{equation}
x\sim_\phi x'\quad\Longleftrightarrow\quad \phi(x)=\phi(x').
\end{equation}
The equivalence classes induced by $\phi$ form the downstream state representation. Computation may reorganize relations among these classes, combine preserved distinctions, and implement increasingly powerful rules over them. Splitting a class whose members remain identical at the effective interface requires a change in observation, memory, or representation that refines the effective interface. Expanding executable support can add mappings or actions over the existing classes, but it cannot by itself distinguish members of one interface fiber.

The decisive issue is the alignment between the symbolic partition and the task. A fiber is harmless when its members support the same optimal response; it becomes limiting when the interface merges states that require different actions. Collision floors therefore measure a mismatch between the distinctions preserved by the system and the distinctions demanded by the task.

\subsection{Three functional incompletenesses}

We organize symbolic intelligence through three functional questions. Each question identifies a distinct stage at which task-relevant structure can be lost or underused.

\emph{Substantive incompleteness} answers: what content enters the symbol system? It is the information loss produced when a situated world or cognitive content is projected into a symbolic record. Objects, practices, perception, action, history, and causal contact are only partially preserved; distinctions that never enter the effective interface cannot be recovered by later computation.

\emph{Substrate incompleteness} answers: by what structure does the system form task-relevant distinctions? Symbolic operations presuppose domains, objects, relations, similarity, salience, provenance, and executable semantics. The model may possess the relevant content yet lack the structure needed to individualize an object, route to the correct evidence, or map a program description to an executable action. Training, architectural bias, tools, and external observation provide functional routes by which token prediction acquires these structures. Their adequacy is measured by the task distinctions and executable consequences they support.

\emph{High-level incompleteness} answers: how does the system organize available distinctions into task behavior? Hypothesis formation, search, verification, calibration, action, and stopping are assembled over the distinctions already available. When fluent symbolic completion outruns those distinctions, the observable signatures include miscalibration \cite{mielke2022} and confabulation-type hallucination \cite{farquhar2024}, together with overthinking and missing-premise overreasoning \cite{chen2025overthink,fan2025,zhang2026overthink}. This layer has two sources: inherited defects from the first two layers, and control defects that persist even when relevant content and structure are present.

The three functional layers form a sequence,
\begin{equation}
\text{content entry}\rightarrow\text{distinction formation}\rightarrow\text{task organization and control}.
\end{equation}
Their task-level consequences are captured by four interacting mathematical objects. The effective interface $\phi$ records which distinctions are available; the executable class $\Pi_{\phi,H}$ records which interface--action mappings the current symbolic and computational environment can realize; the realization profile $M$ records the inference, search, optimization, compilation, scheduling, and runtime mechanism that maps inherited capability into realizable decisions; and the resource-indexed family $\F_s(\mathcal J,M)$ records how much of that capability is realizable under ceiling $s$. Substantive defects primarily restrict the distinctions entering $\phi$. Substrate defects can alter both the represented distinction structure and the executable mappings supported by the system. High-level defects appear in finite realization, search, objectives, stopping, and meta-control. The three philosophical layers therefore project into the mathematical system through overlapping operational effects.

\subsection{Four consequences for bounded symbolic agency}
The Symbolization--Substructure Thesis supplies a language-first route to bounded information, executable mediation, finite realization, and finite diagnostic access. These formal conditions, together with trajectory cost and deployment burden, induce the dynamic choice between managing inherited structural possibilities and selecting interventions that may introduce endpoint structural novelty.

First, symbolization induces a task-relative quotient. States that are distinct in the situated world may occupy the same symbolic fiber, and every fixed downstream computation inherits that identification. Collision floors are properties of the induced quotient.

Second, symbolic availability and executable support are different resources. A system may generate a predicate, program, tool description, or plan while lacking an interpreter--environment pair that realizes the corresponding interface--action mapping. Capability expansion occurs when the executable rule class expands.

Third, structural self-correction is symbolically mediated. A controller that chooses whether to compute, retrieve, observe, expand, prune, refine, or stop acts through a finite diagnostic interface. The collision logic therefore reappears at the control layer.

Fourth, bounded symbolic agency faces a recurrent directional dilemma. Control over inherited structure allocates finite realization and limits unnecessary commitment without enlarging the current structural possibility set. Structural enrichment supplies missing distinctions and mappings while increasing realization, routing, selection, and control burden. Finite intelligence must repeatedly choose between these directions.

Symbolic partiality explains bounded information; symbolic production interacting with executable mediation explains the separation between description and action; symbolically mediated self-correction supplies finite diagnostic access. Finite realization, trajectory cost, and deployment burden then turn these coordinates into a control problem. For language-first systems, symbolic substructure supplies an architectural route to bounded information, executable mediation, finite realization, and finite diagnostic access. Once these coordinates are explicit, Ockham--Chatton control is the trajectory-level choice between organizing inherited possibilities and transforming the structural possibility set.

These consequences motivate four task-level coordinates: an effective information interface, executable policy support, a realization profile, and a terminal realization ceiling, followed by a meta-interface for dynamic control. Consequence preservation supplies the evaluative bridge from symbolic inference to decision risk: executable mappings are judged by the task-relative consequences they preserve or violate. A cost-effective sufficient structure is one local outcome of this control process. The general principle concerns the persistence of the directional choice across task, architectural, and meta-control levels.

For tasks whose loss encodes violations of the required consequence structure, consequence preservation gives the control dilemma its task-relative norm. Ockhamian contraction preserves structural attainability when the contracted joint floor remains at or below the target. A task-level structural transformation is required when the inherited joint floor exceeds the target. Monotone structural enlargement is the direct case; a non-comparable replacement can be represented relative to a common envelope as outward enlargement followed by inward contraction.

For a contracted structure $(\phi',H')$, the target remains attainable when
\begin{equation}
\R_{\phi',H'}^*\le\rho.
\end{equation}
When
\begin{equation}
\R_{\phi,H}^*>\rho,
\end{equation}
reaching the target requires task-interface refinement or executable expansion. Structural attainability is therefore the joint-floor condition $\R_{\mathcal J}^*\le\rho$, while budget-realization attainability under a declared realization profile and ceiling is $\R_{s,\mathcal J,M}^*\le\rho$. Current pathwise attainability additionally depends on whether the relevant decision has entered $\F_{\rm stop}(S)$.

These formal coordinates generate the theorem structure that follows. Task-relative quotienting yields the collision and task-interface-refinement results in the static theory. The separation between symbolic production and executable support yields the joint floor and the criterion for genuine executable expansion. Bounded realization and control cost convert these static structures into the Ockham--Chatton necessity and finite-budget risk-reversal results. Symbolically mediated self-correction reappears in the extended meta-control analysis as meta-diagnosis, meta-action, and directional-control floors.

The Symbolization--Substructure Thesis supplies the language-first architectural route to the formal conditions. Its philosophical role is to explain why partial information, executable mediation, finite realization, and finite diagnostic access arise naturally in symbolic systems. The decision-theoretic layer begins once a deployed agent induces an effective interface $\phi$, executable kernel class $\Pi_{\phi,H}$, realization profile $M$, realization-cost map $c_{\mathcal J,M}$, meta-interface $\mu$, controller kernel $\pi(\cdot\mid\mu(S))$, and task objective. From that point onward, capability geometry and structural control operate on these explicit bounded-agent objects and apply beyond language-first architectures that instantiate the same conditions.

The history of deep learning can be read as a sequence of partial responses to these incompletenesses. Residual connections and persistent optimization changes can reshape the realization profile $M$, while chain of thought and adaptive computation alter pathwise reachability or terminal resource use depending on whether the profile and ceiling are held fixed \cite{he2016,wei2022cot,graves2016}. Retrieval, writable memory, active observation, and recurrent state refine the distinctions exposed to the system \cite{graves2014,lewis2020}. Tools, program interpreters, and extensible hypothesis languages expand executable support. Attention and positional structure shape both the effective distinction interface and the realization profile through which accessible distinctions are exploited \cite{vaswani2017,su2021}. These mechanisms occupy different locations in the theory, even when they produce similar surface improvements.

The central question is therefore sharper than whether larger systems become more capable. It is which gains arise from closing a compensation gap, which arise from refining the effective interface, and which arise from expanding executable support.

\section{Interfaces, executable support, and finite realization}
\label{app:interfaces}

We formalize the task-level consequences of symbolic incompleteness through four interacting coordinates: an effective information interface, an executable rule class, a realization profile, and a resource ceiling.

Let $(X,Y)\sim P$ be a fixed task joint distribution on standard Borel spaces $\mathcal X\times\mathcal Y$. Let $\A$ be a task-relative \emph{ambient semantic action universe}: every terminal semantic action used in architecture comparison is represented in this common standard Borel space. An architecture with an internal action space $\mathcal A_{\mathcal J}$ is compared through a measurable embedding $e_{\mathcal J}:\mathcal A_{\mathcal J}\hookrightarrow\A$ and the corresponding pushforward decision kernels. A semantic action that is unavailable to a particular architecture can therefore remain an element of $\A$ while lying outside that architecture's executable support. Let $\ell:\A\times\mathcal Y\to[0,\infty)$ be measurable. A measurable map $\phi:\mathcal X\to\mathcal Z$ is the \emph{effective interface}: it determines which distinctions the deployed system actually inherits. To treat deterministic decoding and sampling in one formalism, write $\Ker(\mathcal Z,\A)$ for the Markov kernels from $\mathcal Z$ to $\A$, and identify a deterministic rule $g:\mathcal Z\to\A$ with the Dirac kernel $\delta_{g(z)}$. The language--interpreter--environment complex matters through the executable kernel class
\begin{equation}
\Pi_{\phi,H}\subseteq\Ker(\mathcal Z,\A),
\end{equation}
whose members are the interface--action distributions that the parser, interpreter, tools, permissions, environment, and executable semantic randomization can realize. For $\kappa\in\Pi_{\phi,H}$, the induced task-level decision kernel is
\begin{equation}
(\kappa\circ\phi)(B\mid x):=\kappa(B\mid\phi(x)).
\end{equation}

\begin{definition}[Joint structural decision class and risk envelope]
For a joint structure $\mathcal J=(\phi,H)$, define its literal executable decision class on the task space by
\begin{equation}
\mathcal D_{\mathcal J}:=\{\kappa\circ\phi:\kappa\in\Pi_{\phi,H}\}\subseteq\Ker(\mathcal X,\A).
\label{eq:jointdecisionclass}
\end{equation}
The class $\mathcal D_{\mathcal J}$ records decision kernels the deployed language--interpreter--environment complex can actually realize in the common ambient semantic action universe. Executable expansion changes which kernels into $\A$ are realizable; it does not require changing the comparison space itself. When $\mathcal X$ and $\A$ are finite, define the associated expected-risk envelope
\begin{equation}
\widehat{\mathcal D}_{\mathcal J}:=\clco(\mathcal D_{\mathcal J}).
\label{eq:riskenvelope}
\end{equation}
Expected loss is linear and continuous in a finite decision kernel, so $\mathcal D_{\mathcal J}$ and $\widehat{\mathcal D}_{\mathcal J}$ have the same optimal bounded expected risk. The two objects carry different semantics: $\mathcal D_{\mathcal J}$ is literal executability, whereas $\widehat{\mathcal D}_{\mathcal J}$ is the capability envelope identified by universal expected-risk comparison.
\end{definition}

The measurable agent model is stated on standard Borel spaces. The finite-dimensional capability geometry used below has a sharper scope: whenever closed convex capability envelopes, universal finite-dimensional risk characterizations, directed deficiency, or capability-pair orders are invoked, the terminal task state and ambient semantic action spaces are taken finite unless an explicit extension is stated. The dynamic state space and controlled transition kernel may remain standard Borel. This separates general measurable dynamics from the finite terminal geometry used for exact risk-envelope comparison.

Let $(\mathcal R,\preceq)$ be a common externally chosen resource space and let $\overline{\mathcal R}=\mathcal R\cup\{+\infty\}$ adjoin a top element. Cross-architecture comparison is meaningful only after the resource coordinates have a shared operational interpretation, such as latency, energy, memory, FLOPs, tool cost, or a declared weighted vector. Structural capability and realization mechanism are separate coordinates. A \emph{realization profile} $M$ specifies the inference/search procedure, optimization dynamics, runtime implementation, internal scheduling, compilation strategy, realization-specific routing, and other mechanisms that determine how decisions in $\mathcal D_{\mathcal J}$ become realizable. The deployed realization object is therefore $(\mathcal J,M)$, with cost map
\begin{equation}
c_{\mathcal J,M}:\mathcal D_{\mathcal J}\to\overline{\mathcal R}.
\label{eq:realizationprofilecost}
\end{equation}
Its at-most-resource family is
\begin{equation}
\F_s(\mathcal J,M):=\{d\in\mathcal D_{\mathcal J}:c_{\mathcal J,M}(d)\preceq s\}.
\label{eq:resourcefamily}
\end{equation}
When $\mathcal X$ and $\A$ are finite, define the finite-budget expected-risk envelope
\begin{equation}
\widehat{\F}_s(\mathcal J,M):=\clco(\F_s(\mathcal J,M)).
\label{eq:finiteriskenvelope}
\end{equation}
Thus $\mathcal J=(\phi,H)$ determines structural capability, $M$ determines the resource-to-capability map within that structure, and $s$ selects an at-most-resource slice of the resulting realization family. The common external resource space permits cross-architecture comparison even when the same semantic decision has different implementation costs. Exact-resource families are recovered by taking $c_{\mathcal J,M}(d)$ to be the least resource level at which $d$ is realizable. A realization-mechanism transformation can therefore change $M$, $\F_s$, and $\F_\infty$ while keeping $\mathcal D_{\mathcal J}$ fixed.

\begin{definition}[Executable and risk-envelope orders]
For finite $\mathcal X$ and $\A$, define the terminal executable order and terminal risk-envelope order by
\begin{align}
\mathcal J_1\preceq_{\rm exec}\mathcal J_2
&\quad\Longleftrightarrow\quad
\mathcal D_{\mathcal J_1}\subseteq\mathcal D_{\mathcal J_2},\label{eq:execorder}\\
\mathcal J_1\preceq_{\rm risk}\mathcal J_2
&\quad\Longleftrightarrow\quad
\widehat{\mathcal D}_{\mathcal J_1}\subseteq\widehat{\mathcal D}_{\mathcal J_2}.
\label{eq:riskorder}
\end{align}
At a fixed common budget $s$, define finite realization orders on deployed pairs $(\mathcal J_i,M_i)$ by
\begin{align}
(\mathcal J_1,M_1)\preceq_{{\rm exec},s}(\mathcal J_2,M_2)
&\quad\Longleftrightarrow\quad
\F_s(\mathcal J_1,M_1)\subseteq\F_s(\mathcal J_2,M_2),\\
(\mathcal J_1,M_1)\preceq_{{\rm risk},s}(\mathcal J_2,M_2)
&\quad\Longleftrightarrow\quad
\widehat{\F}_s(\mathcal J_1,M_1)\subseteq\widehat{\F}_s(\mathcal J_2,M_2).
\label{eq:finiteriskorder}
\end{align}
Executable dominance implies risk-envelope dominance at both levels. Write $\mathcal J_1\sim_D\mathcal J_2$ when $\widehat{\mathcal D}_{\mathcal J_1}=\widehat{\mathcal D}_{\mathcal J_2}$. These are terminal capability orders; dynamic architecture comparison additionally depends on reachable structures, meta-actions, transition laws, diagnostic interfaces, and costs.
\end{definition}

For fixed $\mathcal J$, realization profiles can also be quotiented by the resource-indexed capability family they induce. Write
\begin{align}
M_1\sim^{\rm exec}_{\mathcal J}M_2
&\quad\Longleftrightarrow\quad
\F_s(\mathcal J,M_1)=\F_s(\mathcal J,M_2)\ \ \forall s,\\
M_1\sim^{\rm risk}_{\mathcal J}M_2
&\quad\Longleftrightarrow\quad
\widehat{\F}_s(\mathcal J,M_1)=\widehat{\F}_s(\mathcal J,M_2)\ \ \forall s.
\end{align}
Thus the formal identity of a realization profile is operational: implementations are equivalent at the literal or risk-envelope level when they induce the same resource-indexed capability family.

The four formal coordinates separate inherited capability from its realization into
\begin{equation}
\boxed{
\begin{array}{cc}
\phi:\text{available distinctions} & \Pi_{\phi,H}:\text{executable possibilities}\\[3pt]
M:\text{realization mechanism} & s:\text{terminal realization ceiling}
\end{array}}
\end{equation}

They generate a two-dimensional capability geometry. The horizontal axis passes from literal executability to the closed convex envelope identified by bounded expected-risk comparison; the vertical axis passes from finite-resource realization to the full capability of the inherited joint structure:
\begin{equation}
\boxed{
\begin{array}{ccc}
\mathcal D_{\mathcal J} & \longrightarrow & \widehat{\mathcal D}_{\mathcal J}\\[4pt]
\cup && \cup\\[-2pt]
\F_s(\mathcal J,M) & \longrightarrow & \widehat{\F}_s(\mathcal J,M).
\end{array}}
\label{eq:capabilitysquare}
\end{equation}
The upper row is the \emph{structural capability boundary}: literal executable kernels on the left and their universal expected-risk envelope on the right. The lower row is the \emph{realization boundary} at budget $s$: literal finite realizability on the left and finite-budget expected-risk capability on the right. Accordingly, $\preceq_{\rm exec}$ and $\preceq_{\rm risk}$ compare the structural row, while $\preceq_{{\rm exec},s}$ and $\preceq_{{\rm risk},s}$ compare the finite-resource row. Structural risk dominance can therefore coexist with a finite-budget reversal because the two comparisons occupy different vertical levels of the same capability square.

\begin{definition}[Capability boundaries and capability order]
The term \emph{boundary} denotes a set-valued capability object ordered by inclusion. Define
\begin{equation}
\mathfrak B_{\rm str}(\mathcal J)
:=\bigl(\mathcal D_{\mathcal J},\widehat{\mathcal D}_{\mathcal J}\bigr),
\qquad
\mathfrak B_s(\mathcal J,M)
:=\bigl(\F_s(\mathcal J,M),\widehat{\F}_s(\mathcal J,M)\bigr).
\label{eq:capboundaries}
\end{equation}
For capability pairs $\mathfrak B_1=(L_1,R_1)$ and $\mathfrak B_2=(L_2,R_2)$, write
\begin{equation}
\mathfrak B_1\preceq_{\rm cap}\mathfrak B_2
\quad\Longleftrightarrow\quad
L_1\subseteq L_2\ \text{ and }\ R_1\subseteq R_2.
\label{eq:capboundaryorder}
\end{equation}
On the canonical capability pairs used in the paper, $R_i=\clco(L_i)$, so
\begin{equation}
\mathfrak B_1\preceq_{\rm cap}\mathfrak B_2
\quad\Longleftrightarrow\quad
L_1\subseteq L_2.
\label{eq:strongcaporder}
\end{equation}
Thus $\preceq_{\rm cap}$ is the strong literal-preserving capability order: every old literal decision remains simulable after an outward move. A transition is \emph{strongly outward} when the pre-transition boundary is below the post-transition boundary in $\preceq_{\rm cap}$, \emph{strongly inward} when the reverse order holds, and a \emph{literal reconfiguration} when the literal components are incomparable. The column-specific risk-envelope order remains distinct: $R_1\subseteq R_2$ can hold even when $L_1\not\subseteq L_2$, producing a risk-outward but literal-reconfiguring transformation. Capability order geometry therefore records strong literal simulation, while $\preceq_{\rm risk}$ records the weaker universal bounded-risk comparison. Directed deficiency supplies approximation geometry when exact risk inclusion fails; $T^{\rm exec}$ and $T^{\rm risk}$ supply the corresponding resource geometries for finite capability simulation.
\end{definition}

For a task law $Q$ on $\mathcal X\times\mathcal Y$, a bounded or nonnegative measurable loss $\ell$, and a decision kernel $d\in\Ker(\mathcal X,\A)$, define
\begin{equation}
\mathcal L_{Q,\ell}(d):=\E_{(X,Y)\sim Q}\!\left[\int_{\A}\ell(a,Y)\,d(da\mid X)\right].
\label{eq:kernelrisk}
\end{equation}
When the loss is fixed by context, write $\mathcal L_Q=\mathcal L_{Q,\ell}$ and likewise suppress $\ell$ from the risk notation. The joint structural floor, finite-budget optimum, and compensation gap are
\begin{align}
\R_{\mathcal J}^*(Q)&:=\inf_{d\in\mathcal D_{\mathcal J}}\mathcal L_Q(d),\label{eq:jointfloor}\\
\R_{s,\mathcal J,M}^*(Q)&:=\inf_{d\in\F_s(\mathcal J,M)}\mathcal L_Q(d),\\
C_s(\mathcal J,M;Q)&:=\R_{s,\mathcal J,M}^*(Q)-\R_{\mathcal J}^*(Q).
\end{align}
When a realization profile is fixed by context, we retain the compact shorthand $\R_{s,\phi,H}^*=\R_{s,\mathcal J,M}^*$ and $C_s(\phi,H)=C_s(\mathcal J,M)$ alongside $\R_{\phi,H}^*=\R_{\mathcal J}^*$. Throughout the static theory the task law is the fixed $P$; the explicit law argument is restored in the dynamic theory, where observations and interventions change the conditional task law.

\begin{proposition}[Order hierarchy]
For finite $\mathcal X$ and $\A$,
\begin{equation}
\mathcal J_1\preceq_{\rm exec}\mathcal J_2
\Longrightarrow
\mathcal J_1\preceq_{\rm risk}\mathcal J_2.
\end{equation}
The analogous implication holds at every fixed budget $s$.
\end{proposition}
\begin{proof}
Set inclusion is preserved by convexification and closure.
\end{proof}

\subsection{Architectural grounding of the formal conditions}

The mathematical conditions correspond to concrete elements of current reasoning systems. Persistent changes to an inference or execution procedure alter the realization profile $M$ when they change the resource-to-capability map while leaving the inherited decision class fixed. Adaptive computation, cost-aware test-time reasoning, and model routing provide concrete examples \cite{graves2016,snell2025,desabbata2024,routellm2025,bestroute2025}; compiler, runtime, and scheduling changes occupy the same coordinate when they have the same formal effect. External test-time authorization changes the terminal realization ceiling $s$, such as a larger FLOP, latency, memory, or tool-call allowance. Computation performed under a fixed $(\mathcal J,M,s)$---decoding, sampling, verification, branch search, revision, reranking, or adaptive reasoning length---changes the pathwise state $q$ and the immediately reachable stop family while the ceiling stays fixed \cite{wei2022cot,adaptthink2025,refrain2026}. Retrieval systems provide concrete examples of interface refinement when they expose new task-relevant evidence \cite{lewis2020,flare2023,selfrag2024,adaptiverag2024,autosearch2026}; active observation and memory restoration occupy the same coordinate when they reveal distinctions previously unavailable to the terminal decision. Toolformer and ReAct illustrate tool-enabled executable interaction \cite{schick2023,yao2023}, while TECTON and MeCo illustrate adaptive tool selection or invocation over enabled semantics \cite{tecton2025,meco2025}. Error-source, confidence, validity, cost, and stopping signals form the diagnostic interface used by the meta-controller.

The structure--realization boundary is classified by its realized coordinate effect. Write $r^{\rm access}$ for routing that changes which task-relevant distinctions are accessible to the terminal semantic decision and $r^{\rm search}$ for routing that only searches, ranks, organizes, or selects computations over distinctions already accessible. Access routing acts through $\iota$ and $\phi$; search routing acts through the algorithmic state $q$, the immediately reachable stop family, and realization cost. Retrieval, memory, attention, context selection, compression, and candidate routing can instantiate either type. A retrieved document that exposes a previously unavailable fact is structural refinement; reranking the same already accessible documents is a realization intervention. This effect-based rule keeps the central distinction stable across changing agent mechanisms.

Finite digital systems provide the basic measurable structure. Token sequences over finite vocabularies form countable standard Borel spaces; finite action menus are measurable; deterministic decoders are measurable compositions of neural operations; stochastic decoders are Markov kernels. Standard Borel spaces preserve regular conditional probabilities, measurable factorization, and decision-theoretic comparison. The deterministic interface $\phi:X\to Z$ is the Dirac special case of a stochastic observation experiment $O_\phi\in\Ker(X,Z)$ with $O_\phi(\cdot\mid x)=\delta_{\phi(x)}$; noisy sensing, retrieval, and active observation fit the same information side through observation kernels when that generality is needed.

Benchmarks and deployed agents can enforce the comparison conditions directly. Fixed prompts, memory states, tool schemas, environment permissions, and intervention menus define controlled joint structures. Exact collisions hold the complete visible interface fixed while changing the correct label or action. Masking, gating, pruning, compression, and stopping instantiate contraction when every retained behavior remains simulable inside the inherited structure.

Markov sufficiency is obtained at the analyst level by augmenting the state with the conditional task law induced by the interaction history, together with active memory, tool state, environment state, architecture state, and remaining budget. Measurable costs and transitions then support the Bellman analysis. The boundary between realization and structure is operational: reorganizing already available information changes finite realization, while acquiring a task-visible observation, restoring an inaccessible task record, or enabling a new semantic action changes a structural coordinate. A signal used only for meta-action choice refines $\mu$ and the dynamic policy class without entering the task interface $\phi$.

\begin{table}[t]
\centering
\caption{Formal objects and their reasoning-architecture interpretation.}
\label{tab:assumptions}
\small\setlength{\tabcolsep}{4pt}
\begin{tabularx}{\textwidth}{@{}p{2.2cm}Xp{3.5cm}@{}}
\toprule
Object or condition & Architectural interpretation & Operational realization \\
\midrule
$s$ & terminal realization ceiling: authorized latency, tokens/FLOPs, memory, tool-call allowance & budget family $\F_s(\mathcal J,M)$ \\
$M$ & persistent inference/search/optimization/runtime mechanism & resource-to-capability map $c_{\mathcal J,M}$ and asymptotic family $\F_\infty(\mathcal J,M)$ \\
$\phi$ & selected context, retrieved memory, observations, access routing $r^{\rm access}$ & controlled interface / accessible distinctions \\
$\Pi_{\phi,H}$ & parser, tool schema, runtime, permissions, semantic actions & executable interface--action class \\
$q$ & traces, candidates, search routing $r^{\rm search}$, branch state & path-dependent reachability inside fixed $(\mathcal J,M,s)$ \\
$\F_{\rm stop}(S)$ & kernels selectable if the controller stops now & state-reachable decision family \\
$\mu(S)$ & error-source, confidence, cost, validity, stopping signals & diagnostic state for meta-control \\
$\pi(\cdot\mid\mu(S))$ & selector/controller policy acting on the observed meta-state & policy-induced terminal kernel $d_S^\pi$ or meta-action distribution \\
$\F_s$ nesting & preservation of lower-budget policies & emulate and stop at the old limit \\
Markov state & conditional task law, history, memory, environment, tools, budget & augmented analyst state \\
exact collision & identical visible state, different target & matched twins \\
approximate collision & nearby visible-state laws & total-variation comparison \\
\bottomrule
\end{tabularx}
\end{table}

\section{Static emergence invariance}
\label{app:static}

\begin{proposition}[Compensation decomposition and monotonicity]
For fixed $(\mathcal J,M)$,
\begin{equation}
\R_{s,\mathcal J,M}^*=\R_{\mathcal J}^*+C_s(\mathcal J,M),\qquad C_s(\mathcal J,M)\ge0.
\label{eq:decomp}
\end{equation}
If $s\preceq s'$, then $C_{s'}(\mathcal J,M)\le C_s(\mathcal J,M)$.
\end{proposition}
\begin{proof}
Every $d\in\F_s(\mathcal J,M)$ belongs to $\mathcal D_{\mathcal J}$, so its kernel risk is a candidate in the infimum defining $\R_{\mathcal J}^*$ and $\R_{s,\mathcal J,M}^*\ge\R_{\mathcal J}^*$. Because the at-most-resource construction gives $\F_s(\mathcal J,M)\subseteq\F_{s'}(\mathcal J,M)$ under $s\preceq s'$, the infimum over the larger family is bounded above by the infimum over the smaller family, while the joint structural floor remains fixed.
\end{proof}

The decomposition is the core systems statement for a fixed deployed pair $(\mathcal J,M)$. Under nested at-most-resource families, expanding the terminal realization ceiling weakly lowers the optimal compensation envelope. A change in training, optimization, or implementation can alter $M$ and therefore belongs to a different comparison; a longer computation under fixed $(\mathcal J,M,s)$ changes pathwise reachability at a fixed ceiling. The joint floor stays fixed whenever the effective interface and executable support preserve the same structural possibilities.
The resource-infimal compensation gap is
\begin{equation}
C_{\inf}(\mathcal J,M):=\inf_{s\in\mathcal R}C_s(\mathcal J,M).
\end{equation}
When $\mathcal R$ admits a countable cofinal sequence, an increasing sequence $s_1\preceq s_2\preceq\cdots$ is \emph{cofinal} when every $s\in\mathcal R$ satisfies $s\preceq s_n$ for some $n$. Because $C_s(\mathcal J,M)$ is monotone along the resource order, every cofinal path satisfies
\begin{equation}
\lim_{n\to\infty}C_{s_n}(\mathcal J,M)=\inf_n C_{s_n}(\mathcal J,M)=C_{\inf}(\mathcal J,M).
\end{equation}
Total compensation for a deployed pair $(\mathcal J,M)$ occurs when $\R_{\mathcal J}^*=0$ and $C_{\inf}(\mathcal J,M)=0$.

The resource family also has an asymptotic realization boundary. Define
\begin{align}
\F_\infty(\mathcal J,M)&:=\bigcup_{s\in\mathcal R}\F_s(\mathcal J,M)\subseteq\mathcal D_{\mathcal J},\label{eq:Finfinity}\\
\widehat{\F}_\infty(\mathcal J,M)&:=\clco\!\left(\F_\infty(\mathcal J,M)\right)\subseteq\widehat{\mathcal D}_{\mathcal J},\label{eq:Fhatinfinity}\\
\R_{\infty,\mathcal J,M}^*(Q)&:=\inf_{d\in\F_\infty(\mathcal J,M)}\mathcal L_Q(d)=\inf_s\R_{s,\mathcal J,M}^*(Q).
\label{eq:Rinfinity}
\end{align}
Thus the finite realization boundary expands monotonically inside an inherited structural boundary and approaches the portion of structural capability reachable by the chosen realization family. The inclusions
\begin{equation}
\F_s(\mathcal J,M)\subseteq\F_\infty(\mathcal J,M)\subseteq\mathcal D_{\mathcal J},
\qquad
\widehat{\F}_s(\mathcal J,M)\subseteq\widehat{\F}_\infty(\mathcal J,M)\subseteq\widehat{\mathcal D}_{\mathcal J}
\label{eq:asymnested}
\end{equation}
separate finite-budget reachability, asymptotic reachability under the chosen realization mechanism, and inherited structural capability.

\begin{definition}[Three levels of asymptotic completeness]
A deployed pair $(\mathcal J,M)$ has \emph{literal realization completeness} when
\begin{equation}
\F_\infty(\mathcal J,M)=\mathcal D_{\mathcal J}.
\end{equation}
For finite $\mathcal X$ and $\A$, it has \emph{risk-envelope realization completeness} when
\begin{equation}
\widehat{\F}_\infty(\mathcal J,M)=\widehat{\mathcal D}_{\mathcal J}.
\end{equation}
For a specified task law $Q$ and bounded loss $\ell$, it has \emph{task-effective asymptotic completeness} when
\begin{equation}
C^{\rm asym}(\mathcal J,M;Q,\ell)
:=\R_{\infty,\mathcal J,M}^*(Q,\ell)-\R_{\mathcal J}^*(Q,\ell)=0.
\label{eq:taskcomplete}
\end{equation}
\end{definition}

\begin{proposition}[Asymptotic realization completeness hierarchy]
For finite $\mathcal X$ and $\A$,
\begin{equation}
\F_\infty(\mathcal J,M)=\mathcal D_{\mathcal J}
\Longrightarrow
\widehat{\F}_\infty(\mathcal J,M)=\widehat{\mathcal D}_{\mathcal J}.
\end{equation}
Moreover,
\begin{equation}
\widehat{\F}_\infty(\mathcal J,M)=\widehat{\mathcal D}_{\mathcal J}
\quad\Longleftrightarrow\quad
C^{\rm asym}(\mathcal J,M;Q,\ell)=0
\end{equation}
for every finite target space, every finite task law $Q$ whose $X$-marginal has full support, and every bounded loss $\ell$. Risk-envelope completeness can therefore be identified by universal decision performance. Literal realization completeness remains stronger: different literal classes can share the same closed convex hull. For a single fixed decision problem, a zero asymptotic gap can also occur on a proper subset of the structural risk envelope.
\end{proposition}
\begin{proof}
Literal equality gives risk-envelope equality after closed convexification. Risk-envelope equality gives zero asymptotic gap for every bounded decision problem because bounded expected risk is a continuous linear functional of the finite decision kernel. Conversely, suppose the asymptotic envelope is a strict subset of the structural envelope. Compact convex separation gives a linear functional whose infimum is strictly smaller on the structural envelope. Fix any full-support $X$-marginal, take a finite target with deterministic $Y=X$, and apply statewise additive shifts and a common positive rescaling exactly as in Theorem~\ref{thm:main-budget-reversal}; the separating functional becomes a bounded nonnegative loss with a strictly positive asymptotic gap, contradicting universal zero gap. Finally, closed convexification can identify distinct literal classes, while one particular task can attain its optimum on a proper subset of an envelope.
\end{proof}

\begin{proposition}[Finite-resource and asymptotic realization gaps]
For every task law $Q$, bounded loss $\ell$, and budget $s$,
\begin{equation}
C_s(\mathcal J,M;Q,\ell)
=
\underbrace{\R_{s,\mathcal J,M}^*(Q,\ell)-\R_{\infty,\mathcal J,M}^*(Q,\ell)}_{C_s^{\rm fin}(\mathcal J,M;Q,\ell)}
+
\underbrace{\R_{\infty,\mathcal J,M}^*(Q,\ell)-\R_{\mathcal J}^*(Q,\ell)}_{C^{\rm asym}(\mathcal J,M;Q,\ell)},
\label{eq:twogapdecomp}
\end{equation}
where
\begin{align}
C_s^{\rm fin}(\mathcal J,M;Q,\ell)
&:=\R_{s,\mathcal J,M}^*(Q,\ell)-\R_{\infty,\mathcal J,M}^*(Q,\ell),\label{eq:finitegapexplicit}\\
C^{\rm asym}(\mathcal J,M;Q,\ell)
&:=\R_{\infty,\mathcal J,M}^*(Q,\ell)-\R_{\mathcal J}^*(Q,\ell).\label{eq:asymgapexplicit}
\end{align}
Both terms are nonnegative. Moreover,
\begin{equation}
C^{\rm asym}(\mathcal J,M;Q,\ell)=\inf_s C_s(\mathcal J,M;Q,\ell).
\label{eq:asymgapcinf}
\end{equation}
The first term is a finite-resource gap relative to the asymptotic capability of the current realization mechanism. The second is a task-effective asymptotic gap: the current task assigns value to structural capability that the asymptotic realization envelope does not attain.
\end{proposition}
\begin{proof}
Nested at-most-resource families give
\begin{equation}
\R_{s,\mathcal J,M}^*(Q,\ell)
\ge \R_{\infty,\mathcal J,M}^*(Q,\ell)
\ge \R_{\mathcal J}^*(Q,\ell),
\end{equation}
where $\R_{\infty,\mathcal J,M}^*=\inf_r\R_{r,\mathcal J,M}^*$. Adding and subtracting $\R_{\infty,\mathcal J,M}^*(Q,\ell)$ yields Eq.~\eqref{eq:twogapdecomp}. Taking the infimum of $C_s=\R_{s,\mathcal J,M}^*-\R_{\mathcal J}^*$ over $s$ gives Eq.~\eqref{eq:asymgapcinf}.
\end{proof}

Set-level risk-envelope realization incompleteness means $\widehat{\F}_\infty(\mathcal J,M)\subsetneq\widehat{\mathcal D}_{\mathcal J}$; it is a property of the structural architecture together with its realization profile. Task-effective asymptotic incompleteness means $C^{\rm asym}(\mathcal J,M;Q,\ell)>0$; it is the consequence of that geometry for a particular task projection. Set-level incompleteness can coexist with zero task-effective asymptotic gap when the omitted capability has no additional value under the specified $Q$ and $\ell$.

Equation~\eqref{eq:decomp} is the scalar risk projection of the capability square. The \emph{structural boundary} is the pair $(\mathcal D_{\mathcal J},\widehat{\mathcal D}_{\mathcal J})$: literal executable capability and its expected-risk envelope. The \emph{realization boundary} at budget $s$ is the pair $(\F_s(\mathcal J,M),\widehat{\F}_s(\mathcal J,M))$: the portion of literal and risk-equivalent capability reachable under that budget. A task law and loss project these set-valued boundaries to the scalars $\R_{\mathcal J}^*(Q)$ and $\R_{s,\mathcal J,M}^*(Q)$. The compensation gap is therefore a task-specific scalar separation between two capability boundaries. Its asymptotic component is a task projection and should be distinguished from set-level risk-envelope incompleteness. The executable/risk orders, directed deficiencies, and resource transformations describe the corresponding set-level comparative geometry. Increasing $s$ expands the realization boundary inside the inherited structural boundary; refinement or executable transformation moves the structural boundary and can simultaneously alter the difficulty of realizing it.

the structural risk-envelope extension in \hyperref[app:proofs]{Appendix Module I} proves that universal bounded-risk dominance is equivalent to inclusion of the closed convex capability envelope. This gives the joint information--execution coordinate two complementary geometries: literal executability is ordered by $\preceq_{\rm exec}$, while universal expected-risk capability is ordered by $\preceq_{\rm risk}$. Deterministic semantic rules appear as Dirac kernels, stochastic decoders appear directly as kernels, and Blackwell-style information comparison is recovered when executable kernels are unrestricted.

\begin{definition}[Task-relative approximate structural deficiency]
For finite $\mathcal X,\A$ and a fixed marginal $Q_X$, define the directed deficiency
\begin{equation}
\delta_{Q_X}(\mathcal J_1\!\to\!\mathcal J_2)
:=\sup_{d_1\in\widehat{\mathcal D}_{\mathcal J_1}}
\inf_{d_2\in\widehat{\mathcal D}_{\mathcal J_2}}
\E_{Q_X}\,\TV\!\left(d_1(\cdot\mid X),d_2(\cdot\mid X)\right).
\label{eq:jointdeficiency}
\end{equation}
This task-marginal deficiency measures how closely one joint information--execution structure can simulate another at the level of terminal action distributions. Taking the supremum over a deployment class of marginals yields a workload-class deficiency, and taking the supremum over all marginals on finite $\mathcal X$ yields a distribution-free version.
\end{definition}

the approximate structural-inheritance extension in \hyperref[app:proofs]{Appendix Module I} converts this deficiency into a bounded-loss regret guarantee and places approximate structural simulation on the same total-variation scale as approximate interface collision.

\begin{proposition}[Composition of directed structural deficiency]
For finite $\mathcal X,\A$ and a common marginal $Q_X$,
\begin{equation}
\delta_{Q_X}(\mathcal J_1\!\to\!\mathcal J_3)
\le
\delta_{Q_X}(\mathcal J_1\!\to\!\mathcal J_2)
+
\delta_{Q_X}(\mathcal J_2\!\to\!\mathcal J_3).
\label{eq:deficiencytriangle}
\end{equation}
Consequently,
\begin{equation}
d_{Q_X}(\mathcal J_1,\mathcal J_2)
:=\max\{\delta_{Q_X}(\mathcal J_1\!\to\!\mathcal J_2),
\delta_{Q_X}(\mathcal J_2\!\to\!\mathcal J_1)\}
\end{equation}
defines a pseudometric on $Q_X$-almost-sure risk-envelope equivalence classes. The same triangle inequality holds for $\delta_{Q_X,s}$ at any common budget $s$ with nonempty finite envelopes.
\end{proposition}
\begin{proof}
Let $\rho_{Q_X}(d,d'):=\E_{Q_X}\TV(d(\cdot\mid X),d'(\cdot\mid X))$. This is a pseudometric on decision kernels modulo $Q_X$-almost-sure equality. For any $d_1$ and $\varepsilon>0$, choose $d_2$ within $\delta_{12}+\varepsilon$ of $d_1$ and $d_3$ within $\delta_{23}+\varepsilon$ of $d_2$. The triangle inequality for total variation gives $\rho_{Q_X}(d_1,d_3)\le\delta_{12}+\delta_{23}+2\varepsilon$. Taking the infimum over $d_3$, supremum over $d_1$, and $\varepsilon\downarrow0$ yields Eq.~\eqref{eq:deficiencytriangle}. Symmetrization gives the pseudometric, and the finite-budget statement applies the same argument to $\widehat{\F}_s$.
\end{proof}

The directed deficiency is therefore an architecture-approximation geometry: the order records exact capability inclusion, $\delta$ measures behavioral approximation when exact inclusion fails, and the resource map below measures the budget required to reproduce a finite capability envelope.

\begin{definition}[Finite-resource deficiency]
For a common budget $s$ with nonempty finite envelopes and marginal $Q_X$, define
\begin{equation}
\delta_{Q_X,s}((\mathcal J_1,M_1)\!\to\!(\mathcal J_2,M_2))
:=
\sup_{d_1\in\widehat{\F}_s(\mathcal J_1,M_1)}
\inf_{d_2\in\widehat{\F}_s(\mathcal J_2,M_2)}
\E_{Q_X}\TV\!\left(d_1(\cdot\mid X),d_2(\cdot\mid X)\right).
\label{eq:finitedeficiency}
\end{equation}
It measures the same directed approximation error after the realization budget is imposed.
\end{definition}

\begin{corollary}[Finite-resource approximate risk bound]
If $0\le\ell\le L<\infty$, then
\begin{equation}
\R_{s,\mathcal J_2,M_2}^*(Q)
\le
\R_{s,\mathcal J_1,M_1}^*(Q)
+
L\,\delta_{Q_X,s}((\mathcal J_1,M_1)\!\to\!(\mathcal J_2,M_2)).
\label{eq:finitedeficiencyrisk}
\end{equation}
For full-support $Q_X$, closedness of the target envelope gives $\delta_{Q_X,s}=0$ exactly when $\widehat{\F}_s(\mathcal J_1,M_1)\subseteq\widehat{\F}_s(\mathcal J_2,M_2)$; for general $Q_X$, the equivalence is modulo $Q_X$-almost-sure equality. Theorem~\ref{thm:main-budget-reversal} then recovers exact same-budget universal risk dominance in the full-support case.
\end{corollary}
\begin{proof}
Apply the bounded-loss total-variation argument of the approximate structural-inheritance extension in \hyperref[app:proofs]{Appendix Module I} to the finite-budget envelopes $\widehat{\F}_s(\mathcal J_1,M_1)$ and $\widehat{\F}_s(\mathcal J_2,M_2)$.
\end{proof}

\begin{corollary}[Exact paired collision]
Under zero--one loss in the classification setting $\A=\mathcal Y$, consider an equiprobable pair with deterministic labels, $\phi(x_0)=\phi(x_1)$ and $y(x_0)\ne y(x_1)$. Every deterministic or randomized rule through $\phi$ has average error at least $1/2$ on the pair.
\end{corollary}
\begin{proof}
A deterministic rule assigns the same action to both points and is wrong on at least one. For a randomized rule, let $p_0,p_1$ be the probabilities of outputting the two correct labels; $p_0+p_1\le1$, so average correctness is at most $(p_0+p_1)/2\le1/2$.
\end{proof}

The collision result is the decision-theoretic form of symbolic quotienting. The interface partitions task states into fibers, while the task assigns actions to those states. A fiber becomes limiting precisely when it contains states requiring different actions. Additional computation can improve the rule assigned to the fiber, but cannot split its members without a newly available distinction.

The approximate collision formula and its binary-testing proof are included in the approximate structural-inheritance extension in \hyperref[app:proofs]{Appendix Module I}: the residual floor is exactly $(1-\TV(P_0^\phi,P_1^\phi))/2$.

\begin{definition}[Task-relevant refinement and executable expansion]
$\psi$ refines $\phi$ if there is a measurable $h$ with $\phi=h\circ\psi$ almost surely. An executable expansion is a change $H_t\to H_{t+1}$ with $\Pi_{\phi,H_t}\subsetneq\Pi_{\phi,H_{t+1}}$ when the interface is fixed. A refinement or expansion is task-relevant when it strictly lowers $\R_{\phi,H}^*$ for the specified law and loss.
\end{definition}

\begin{corollary}[Factorization certificate for executable dominance]
Suppose $\phi_t=h\circ\phi_{t+1}$ almost surely and
\begin{equation}
\{\kappa\circ h:\kappa\in\Pi_{\phi_t,H_t}\}\subseteq\Pi_{\phi_{t+1},H_{t+1}}.
\end{equation}
Then every old decision kernel is simulated under $P$ by the new structure and therefore
\begin{equation}
\R_{\phi_{t+1},H_{t+1}}^*\le\R_{\phi_t,H_t}^*.
\end{equation}
\end{corollary}
\begin{proof}
Every old decision kernel $\kappa\circ\phi_t$ equals $(\kappa\circ h)\circ\phi_{t+1}$ almost surely, and the lifted kernel $\kappa\circ h$ belongs to the new executable class. Hence the new structure literally simulates every old decision kernel and therefore dominates in both the executable and risk-envelope orders.
\end{proof}

\begin{proposition}[No-policy-gain expansion is vacuous]
Fix $\phi$. If $H\subseteq H'$ but $\Pi_{\phi,H^{\prime}}=\Pi_{\phi,H}$, then $\R_{\phi,H'}^*=\R_{\phi,H}^*$.
\end{proposition}
\begin{proof}
Both envelopes are infima over the same executable kernel class.
\end{proof}

We use ``semantic expansion'' for executable-policy expansion: generating new strings while leaving $\Pi_{\phi,H}$ unchanged leaves capability unchanged.

\begin{proposition}[Factorization of non-comparable structural replacement]
Let $\phi:\mathcal X\to\mathcal Z_\phi$ and $\psi:\mathcal X\to\mathcal Z_\psi$ be measurable interfaces on the same task space, and define the joint interface $\chi=(\phi,\psi)$. Let $p_\phi$ and $p_\psi$ denote the coordinate projections. Suppose a common executable envelope $\Pi_{\chi,H^\cup}$ satisfies
\begin{equation}
\{\kappa\circ p_\phi:\kappa\in\Pi_{\phi,H}\}\cup\{\kappa\circ p_\psi:\kappa\in\Pi_{\psi,H'}\}\subseteq\Pi_{\chi,H^\cup}.
\end{equation}
Then the non-comparable replacement $(\phi,H)\to(\psi,H')$ admits an order-theoretic representation as structural outward enlargement $(\phi,H)\to(\chi,H^\cup)$ followed by structural inward contraction $(\chi,H^\cup)\to(\psi,H')$.
\end{proposition}
\begin{proof}
Because $\phi=p_\phi\circ\chi$ and $\psi=p_\psi\circ\chi$, the joint interface refines both interfaces. The first lifted class inclusion lets the common envelope simulate every old decision kernel, so the first step is an enrichment. The second lifted class inclusion says that every final decision kernel is available in the common envelope; restricting the interface to $\psi$ and the support to $\Pi_{\psi,H'}$ is therefore a contraction. The factorization is structural. Its deployment value is determined separately by the realization cost and burden of the intermediate envelope.
\end{proof}

Surface expressivity and executable capability therefore obey different orders. A larger string language can leave the decision envelope unchanged, while a small executable extension can lower it sharply. A non-comparable redesign is a reconfiguration in the capability order. Relative to a declared common envelope, it admits an ordered enrichment--contraction representation: structural capability first moves outward to the envelope and then inward to the replacement.

\begin{theorem}[Total compensation under a deployed realization pair]
For every resource-indexed family generated by a fixed deployed pair $(\mathcal J,M)$,
\begin{equation}
\inf_s\R_{s,\mathcal J,M}^*=0
\quad\Longleftrightarrow\quad
\R_{\mathcal J}^*=0\ \text{ and }\ C_{\inf}(\mathcal J,M)=0.
\end{equation}
Under zero--one loss with a finite label space and deterministic target $Y=y(X)$, $\R_{\phi,H}^*=0$ if and only if, for every $\varepsilon>0$, there exists $\kappa_\varepsilon\in\Pi_{\phi,H}$ such that
\begin{equation}
\E\!\left[1-\kappa_\varepsilon(\{Y\}\mid\phi(X))\right]<\varepsilon.
\end{equation}
When the infimum defining $\R_{\phi,H}^*$ is attained, this condition is equivalent to an executable kernel $\kappa^*$ satisfying $\kappa^*(\{Y\}\mid\phi(X))=1$ almost surely.
\end{theorem}
\begin{proof}
The joint-floor decomposition gives $\R_{s,\mathcal J,M}^*=\R_{\mathcal J}^*+C_s(\mathcal J,M)$ with both terms nonnegative. Since $\R_{\mathcal J}^*$ is constant in $s$, $\inf_s\R_{s,\mathcal J,M}^*=\R_{\mathcal J}^*+C_{\inf}(\mathcal J,M)$; the left-hand side is zero exactly when both terms on the right vanish. Under zero--one loss with deterministic $Y$, the risk of $\kappa$ is $\E[1-\kappa(\{Y\}\mid\phi(X))]$. The stated approximation and attainment conditions are therefore exactly the conditions for the infimum of this nonnegative quantity to vanish.
\end{proof}

\begin{corollary}[Universal interface refinement]
Let $\phi:\mathcal X\to\mathcal Z_\phi$ and $\psi:\mathcal X\to\mathcal Z_\psi$ be measurable interfaces on the same probability space, with standard Borel ranges. Under unrestricted measurable decision rules, bounded losses, and finite action spaces, the following are equivalent up to null sets:
\begin{enumerate}
  \item $\overline{\sigma(\phi(X))}\subseteq\overline{\sigma(\psi(X))}$;
  \item there is a measurable $h$ such that $\phi=h\circ\psi$ almost surely;
  \item for every finite-valued target on the underlying probability space, every finite action space, and every bounded loss, the Bayes risk through $\psi$ is no larger than the Bayes risk through $\phi$.
\end{enumerate}
Moreover, under zero--one loss with a finite label space and deterministic target, $\R_\phi^*=0$ exactly when $Y$ is measurable with respect to $\overline{\sigma(\phi(X))}$.
\end{corollary}
\begin{proof}
Information inclusion and measurable factorization are equivalent by the Doob--Dynkin factorization theorem for standard Borel ranges. Factorization lets every $\phi$-rule be simulated through $\psi$, proving universal risk dominance. Conversely, take any event $B\in\sigma(\phi(X))$ and the binary zero--one decision problem with target $\ind_B$. The interface $\phi$ attains zero risk, so universal dominance gives rules $g_n(\psi(X))$ with error below $2^{-n}$. Borel--Cantelli implies eventual almost-sure agreement with $\ind_B$, placing $B$ in $\overline{\sigma(\psi(X))}$. Thus completed information inclusion follows. The label statement applies the same argument to the finitely many label events.
\end{proof}

Total compensation for a deployed pair $(\mathcal J,M)$ occurs when $\R_{\mathcal J}^*=0$ and $C_{\inf}(\mathcal J,M)=0$. A zero joint floor means that the interface preserves sufficient task information and executable support contains rules whose risks approach zero. When a cofinal path exists, a zero resource-infimal compensation gap means that every such path approaches the optimum available in that joint structure. These conditions separate information availability, executable availability, and finite realization.

\begin{takeaway}
\centering\bfseries Emergence acts on distinctions preserved by the architecture. A task-relevant distinction that remains collapsed at the effective interface contributes to the joint floor at every realization level subordinate to that structural boundary.
\end{takeaway}

The static theory is organized by the capability square in Eq.~\eqref{eq:capabilitysquare}. Horizontal comparison separates literal executability from risk-equivalent capability; vertical comparison separates structural capability from finite realization. The asymptotic row limit $\widehat{\F}_\infty$ further distinguishes temporary resource shortage from set-level realization-family incompleteness, while $C^{\rm asym}(\mathcal J,M;Q,\ell)$ records the task-effective consequence of that separation. Structural improvement changes $\phi$, $\Pi_{\phi,H}$, or both. Realization-profile improvement changes $M$ and can enlarge $\F_\infty$; ceiling expansion changes $s$ and enlarges $\F_s$ inside the current profile; pathwise computation changes $\F_{\rm stop}$ inside the current budget family. Their mismatch is exactly where an architecture can be structurally richer yet less realizable under a particular profile, ceiling, or trajectory. The resource transformation below quantifies the budget needed to reproduce another deployed pair's finite capability, turning a generic scaling question into a coordinate-specific control question: which realization or structural level should move, and in which direction?

\end{textAtEnd}

\begin{textAtEnd}[category=interpretation]
\section{Applications to established LLM results}
\label{app:reinterpretation}
This module uses the capability system as an application layer. Each case follows the same order: an established result supplies the phenomenon, the framework locates the active capability coordinate, and that location generates an additional diagnostic, retention, restoration, or control consequence. The mathematical provenance of the core results remains in Appendix~\ref{app:core-comparison}.

\subsection{Component-level application map}
\begin{table}[H]
\centering
\caption{Established phenomena as capability-coordinate applications.}
\label{tab:components}
\scriptsize\setlength{\tabcolsep}{3pt}
\begin{tabularx}{\textwidth}{@{}>{\raggedright\arraybackslash}p{3.0cm}>{\raggedright\arraybackslash}p{2.5cm}X>{\raggedright\arraybackslash}p{3.7cm}@{}}
\toprule
Established phenomenon & Coordinate & Framework diagnosis & Generated consequence \\
\midrule
Observational or memory twins & $\phi$ & unresolved interface collision & observation, retrieval, or memory restoration can release the floor \\
Executable tool/grammar support & $H$ & semantic action enters executable support & reasoning value can change after support expansion \\
Algorithmic reasoning / test-time compute & $M,s,q$ & finite realization and reachability & ceiling sweeps and retention separate realization burden from selection \\
Routing and retrieval & $\phi$, $M$, or $q$ & access and search create different endpoints & endpoint class determines control direction \\
Self-correction and overthinking & reachability / selection & candidate quality and selector quality can move separately & retained-candidate probes attribute degradation \\
Multi-agent protocols & finite envelope / selection & protocol topology changes available policies and aggregation & retained fallback tests universal monotonicity \\
\bottomrule
\end{tabularx}
\end{table}

\subsection{Capability elicitation and adaptive evaluation}
Capability-elicitation studies show that different prompting, steering, or fine-tuning procedures can expose different fractions of a model's latent behavior \cite{hofstatter2025}; adaptive task elicitation constructs new evaluations that target systematic failure modes across domains \cite{brown2025task}. In the capability system, elicitation enlarges the observed inner hull while adaptive evaluation selects new support directions. Corollary~\ref{cor:main-behavioral-identification} supplies the generated consequence: evaluation coverage can be reported as an inner--outer geometric gap, probe budgets admit an explicit recovery bound, and held-out objectives receive a risk interval before they are queried. This turns adaptive evaluation into capability-envelope acquisition, with each benchmark contributing a support direction.

\subsection{Bayesian and observational analyses}
Formal Bayesian accounts of in-context learning update latent hypotheses from available evidence \cite{xie2022}. Equal likelihood under two latent worlds preserves their posterior odds, while evidence that changes the likelihood ratio separates them. In the capability coordinates this is an application of interface inheritance: the available experiment determines $\phi$, and observation or retrieval that changes the experiment moves the structural boundary. The additional framework consequence is intervention choice: a collision directs value toward evidence acquisition, while a finite-realization gap directs value toward computation inside the inherited class.

Histories with the same terminally available suffix occupy one effective record class; restoring task-relevant memory refines that record. Attention, positional encoding, selective context control, and long-context adaptation \cite{vaswani2017,su2021,zhang2024ssa,bansal2026} change how relational distinctions become accessible or realizable. The framework turns these mechanisms into a common measurement question: whether the intervention moves the structural envelope, its finite realization, or both.

\subsection{Reasoning, stopping, and executable support}
Overthinking and missing-premise studies show that additional reasoning can become costly or unproductive \cite{chen2025overthink,fan2025,zhang2026overthink}; adaptive-computation work prices additional reasoning directly \cite{graves2016,snell2025}. The capability reading separates available policy quality from selection and stopping. Retained-candidate measurements then generate an attribution test: oracle-prefix improvement together with deployed degradation locates the loss in the selection component.

Executable-support results supply the complementary case. Tool schemas, interpreters, parsers, and constrained decoders determine which semantic mappings can be executed \cite{schick2023,yao2023,scholak2021picard,geng2023grammar}. Once an executable mapping enters $H$, additional reasoning can acquire value that was absent under the previous support. This yields a matched intervention design crossing evidence access, executable support, and realization effort.

\end{textAtEnd}

\begin{textAtEnd}[category=foundation]
\section{From the static boundary to Ockham--Chatton necessity}

The static theory turns an undifferentiated scaling question into a coordinate-specific boundary-selection question. Capability geometry is neutral: every realized non-stop transition is described by $\Gamma(S,u,S')$, while the transition kernel induces a distribution over strong outward, inward, reconfiguring, and stationary outcomes. Computation and search can move current reachability outward within a fixed terminal ceiling; a budget-ceiling transformation can move the budget boundary; pruning can move reachability, budget, or structural capability inward; evidence refinement and executable expansion move structural capability outward. Ockhamian control concerns how inherited structural possibilities are exploited, organized, restrained, or terminated, while Chattonian-containing actions have positive endpoint structural-novelty probability, typically through new distinctions or executable mappings.

For a fixed current task law and loss, the nested risk boundaries sharpen the structural trigger into a level-specific control criterion. Strict target location identifies the minimum capability level that must change; equality at a boundary additionally depends on attainment of the corresponding risk infimum. Once the target is attainable by the current stop family, stopping becomes a feasible terminal choice and trajectory value decides whether further control is worthwhile. The dynamic problem therefore combines task-risk projection, realized coordinate effects, capability geometry, and complete trajectory value.

\begin{proposition}[Structural contraction cannot lower the floor]
Suppose $\phi'=c\circ\phi$ almost surely for a measurable map $c$, and
\begin{equation}
\{\kappa\circ c:\kappa\in\Pi_{\phi',H'}\}\subseteq\Pi_{\phi,H}.
\end{equation}
Then
\begin{equation}
\R_{\phi',H'}^*\ge\R_{\phi,H}^*.
\end{equation}
\end{proposition}
\begin{proof}
Every kernel $\kappa\in\Pi_{\phi',H'}$ induces $\kappa\circ c\in\Pi_{\phi,H}$ with the same conditional action law and risk. Every risk attainable after contraction is therefore attainable before contraction, so the infimum over the contracted structure is no smaller.
\end{proof}

\begin{theorem}[Task-floor invariance and structural necessity]
Let $\iota_t$ denote the task evidence retained and accessible to the deployed system, let $\mathcal G_t=\sigma(\iota_t)$, and let $P_t=\mathcal L((X,Y)\mid\mathcal G_t)$ be a regular conditional task law. Analyst-only latent information is excluded from $\mathcal G_t$, and every terminal task interface satisfies $\sigma(\phi_t(X))\subseteq\mathcal G_t$ up to completion. Consider a finite trajectory with no exogenous task-law drift. Suppose every pre-stopping action preserves the accessible evidence state, is pure computation, structural contraction of the task interface or executable support, or controller-state refinement, and leaves the terminal semantic rule's task structure unchanged except for that contraction. Then $\mathcal G_t=\mathcal G_0$, $P_t=P_0$ almost surely, and
\begin{equation}
\R_{s_\tau,\phi_\tau,H_\tau}^*(P_0)\ge\R_{\phi_0,H_0}^*(P_0).
\label{eq:taskfloornecessity}
\end{equation}
Consequently, every finite trajectory reaching
\begin{equation}
\R_{s_\tau,\phi_\tau,H_\tau}^*(P_\tau)\le\rho<\R_{\phi_0,H_0}^*(P_0)
\end{equation}
contains either a change in system-accessible task evidence or an endpoint satisfying
\begin{equation}
\widehat{\F}_{s_\tau}(\mathcal J_\tau,M_\tau)
\not\subseteq
\widehat{\mathcal D}_{\mathcal J_0}.
\end{equation}
The latter certifies task-level capability outside the inherited risk envelope. Monotone task-interface refinement and executable expansion are canonical ways to create such an endpoint; a non-comparable replacement can arise through the preceding enrichment--contraction factorization.
\end{theorem}
\begin{proof}
Because every action preserves the accessible evidence state and no exogenous drift occurs, $\mathcal G_t=\mathcal G_0$ and the regular conditional law remains $P_0$ along the trajectory. Pure computation and controller-state refinement preserve the task-level structural pair. Structural contraction leaves the $P_0$-risk floor unchanged or raises it by the preceding proposition. Hence $\R_{\phi_\tau,H_\tau}^*(P_0)\ge\R_{\phi_0,H_0}^*(P_0)$. Adding the nonnegative finite-family gap under $P_0$ yields Eq.~\eqref{eq:taskfloornecessity}. For the consequence, if accessible evidence is unchanged then $P_\tau=P_0$. A terminal risk strictly below the initial structural floor implies that some decision in the terminal finite envelope has lower risk than every decision in $\widehat{\mathcal D}_{\mathcal J_0}$, so the terminal finite envelope cannot be contained in $\widehat{\mathcal D}_{\mathcal J_0}$. Evidence contraction falls outside the pathwise invariance premise: as an information coarsening, it is evaluated through its expected decision-risk effect. Evidence retained in a memory, summary, or terminally usable controller state remains part of $\iota_t$.
\end{proof}

\begin{corollary}[Bounded-drift structural necessity]
Assume the terminal loss satisfies $0\le\ell\le L$ and allow exogenous task-law drift satisfying
\begin{equation}
\TV(P_\tau,P_0)\le\varepsilon.
\end{equation}
If the trajectory otherwise contains only computation, controller refinement, or structural contraction, with no task-relevant evidence enrichment or executable expansion, then
\begin{equation}
\R_{s_\tau,\phi_\tau,H_\tau}^*(P_\tau)
\ge \R_{\phi_0,H_0}^*(P_0)-L\varepsilon.
\label{eq:boundeddriftnecessity}
\end{equation}
Hence any improvement below the inherited floor by more than $L\varepsilon$ certifies task-level boundary movement after the drift allowance.
\end{corollary}
\begin{proof}
For every fixed decision, bounded loss gives an $L\TV(P_\tau,P_0)$ change in expected risk; taking infima preserves the same Lipschitz bound. Apply this to the terminal class and combine it with structural-contraction monotonicity under $P_0$.
\end{proof}

The theorem converts the joint floor into a boundary criterion. Evidence-acquiring refinement supplies missing distinctions, executable expansion supplies missing mappings, and replacement reorganizes both through an enrichment--contraction path. Dynamic capability control has two axes: capability level and movement direction. The level ranges over current reachability, budget realizability, asymptotic realizability, and structural capability; the monotone direction is outward or inward, while realized endpoints may also be stationary or incomparable. Computation and search typically act on current reachability, ceiling transformations act on budget realizability, realization-profile transformations act on asymptotic realizability, and structural transformations act on the inherited capability set.

\section{Dynamic Ockham--Chatton control}
\label{app:dynamic}

\subsection{State, actions, and trajectory value}

At time $t$, let
\begin{equation}
S_t=(\iota_t,P_t,\phi_t,H_t,M_t,s_t,b_t,q_t),
\end{equation}
where $\iota_t$ represents all task evidence retained and accessible to the deployed system, $\mathcal G_t:=\sigma(\iota_t)$, and $P_t:=\mathcal L((X,Y)\mid\mathcal G_t)$ is a regular conditional task law. Write $\mathcal J_t=(\phi_t,H_t)$. The interface $\phi_t$ is the effective task interface available to the terminal semantic rule, $H_t$ the current executable support description, $M_t$ the current realization profile, $s_t$ the terminal realization ceiling that determines the at-most-resource terminal family, $b_t$ the remaining control-trajectory budget governing future intervention feasibility, and $q_t$ the controller's diagnostic and active algorithmic state, including traces, candidate sets, and search-routing state. The coordinates $s_t$ and $b_t$ need not share units or satisfy an order relation: $s_t$ parameterizes terminal realization capability, whereas $b_t$ prices and constrains future control actions. Access routing that changes terminally available task distinctions is represented in $\iota_t$ and $\phi_t$; search routing over already available distinctions is represented in $q_t$. We impose the deployment-consistency condition
\begin{equation}
\sigma(\phi_t(X))\subseteq\mathcal G_t
\quad\text{up to completion},
\label{eq:interfaceconsistency}
\end{equation}
so the terminal interface cannot contain analyst-only latent information absent from the system-accessible evidence. Controller-generated representations may reorganize $\iota_t$ but cannot violate this information inclusion. Analyst-only latent variables belong to the transition model, not to $\mathcal G_t$; conditioning on information unavailable to the system would create an artificial risk reduction. If a memory, summary, or controller state is usable by the terminal decision and contains task evidence, that evidential content belongs to $\iota_t$. If it is erased from every accessible component, the next conditional law is formed from the evidence that remains accessible.

For a state $S=(\iota,P_S,\phi,H,M,s,b,q)$, let $\F_{\rm stop}(S)$ be the measurable decision kernels that the system can select and execute by stopping in that state. The dynamic realization hierarchy is
\begin{equation}
\varnothing\ne
\F_{\rm stop}(S)
\subseteq
\F_s(\mathcal J(S),M(S))
\subseteq
\F_\infty(\mathcal J(S),M(S))
\subseteq
\mathcal D_{\mathcal J(S)}.
\label{eq:fourlevelrealization}
\end{equation}
The four levels have distinct causal ownership. $\mathcal D_{\mathcal J}$ is determined by the inherited joint structure $\mathcal J=(\phi,H)$. The asymptotic family $\F_\infty(\mathcal J,M)$ is determined by the deployed pair $(\mathcal J,M)$. The budget family $\F_s(\mathcal J,M)$ additionally depends on the terminal realization ceiling $s$. Finally, $\F_{\rm stop}(S)$ is current pathwise reachability after the realized trace, routing, candidate construction, and pruning history. Thus structural transformation, realization-mechanism transformation, ceiling transformation, and pathwise search act at distinct levels of the same inclusion chain. When the terminal task and ambient action spaces are finite, define
\begin{equation}
\widehat{\F}_{\rm stop}(S):=\clco\F_{\rm stop}(S),
\end{equation}
so the risk-envelope chain is
\begin{equation}
\widehat{\F}_{\rm stop}(S)
\subseteq
\widehat{\F}_s(\mathcal J(S),M(S))
\subseteq
\widehat{\F}_\infty(\mathcal J(S),M(S))
\subseteq
\widehat{\mathcal D}_{\mathcal J(S)}.
\label{eq:fourlevelriskenvelope}
\end{equation}

These inclusions induce four capability boundaries:
\begin{align}
\mathfrak B_{\rm reach}(S)
&:=\bigl(\F_{\rm stop}(S),\widehat{\F}_{\rm stop}(S)\bigr),\\
\mathfrak B_{\rm budget}(S)
&:=\mathfrak B_s(\mathcal J(S),M(S)),\\
\mathfrak B_{\rm asym}(S)
&:=\bigl(\F_\infty(\mathcal J(S),M(S)),\widehat{\F}_\infty(\mathcal J(S),M(S))\bigr),\\
\mathfrak B_{\rm structural}(S)
&:=\mathfrak B_{\rm str}(\mathcal J(S)).
\label{eq:dynamicboundaries}
\end{align}
They satisfy
\begin{equation}
\mathfrak B_{\rm reach}(S)
\preceq_{\rm cap}
\mathfrak B_{\rm budget}(S)
\preceq_{\rm cap}
\mathfrak B_{\rm asym}(S)
\preceq_{\rm cap}
\mathfrak B_{\rm structural}(S).
\label{eq:boundarychain}
\end{equation}
Thus resource availability and current pathwise reachability are separate coordinates. Search, verification, or reasoning can move $\mathfrak B_{\rm reach}$ outward at fixed $s$ without changing $\mathfrak B_{\rm budget}$; increasing the terminal realization ceiling moves $\mathfrak B_{\rm budget}$; improving the realization mechanism can move $\mathfrak B_{\rm asym}$; information or executable transformation can move $\mathfrak B_{\rm structural}$ and thereby reshape all lower boundaries.

The corresponding state-conditioned risks are
\begin{align}
\R_{\rm floor}(S)
&:=\R_{\mathcal J(S)}^*(P_S),\\
\R_{\rm asym}(S)
&:=\R_{\infty,\mathcal J(S),M(S)}^*(P_S),\\
\R_{\rm budget}(S)
&:=\R_{s(S),\mathcal J(S),M(S)}^*(P_S),\\
\R_{\rm reach}(S)
&:=\inf_{d\in\F_{\rm stop}(S)}\mathcal L_{P_S}(d)
=: \R_{\rm fin}(S).
\label{eq:dynamicrisks}
\end{align}
Define the three realization gaps
\begin{align}
C_{\rm reach}(S)&:=\R_{\rm reach}(S)-\R_{\rm budget}(S)\ge0,\\
C_{\rm fin}(S)&:=\R_{\rm budget}(S)-\R_{\rm asym}(S)\ge0,\\
C_{\rm asym}(S)&:=\R_{\rm asym}(S)-\R_{\rm floor}(S)\ge0.
\label{eq:threegaps}
\end{align}
Then the state-conditioned compensation gap has the exact decomposition
\begin{equation}
C_{\rm gap}(S)
:=\R_{\rm fin}(S)-\R_{\rm floor}(S)
=C_{\rm reach}(S)+C_{\rm fin}(S)+C_{\rm asym}(S).
\label{eq:dynamicthreegap}
\end{equation}
$C_{\rm reach}$ measures pathwise search and pruning shortfall relative to what the current budget permits; $C_{\rm fin}$ measures finite-resource shortfall relative to the asymptotic realization family; $C_{\rm asym}$ measures the task-effective limit of that realization family relative to inherited structural capability. When $\F_{\rm stop}(S)=\F_s(\mathcal J(S),M(S))$, the reachability gap vanishes and the dynamic decomposition reduces to the static two-gap decomposition under the conditional law $P_S$.

\begin{proposition}[Strict four-level task-risk intervention criterion]
Fix a state $S$, its current conditional task law $P_S$, and task loss $\ell$. Let $\varrho\ge0$ be a target task-risk level, and suppose $\varrho$ is strictly separated from the four nested risk boundaries
\begin{equation}
\R_{\rm floor}(S)\le \R_{\rm asym}(S)\le \R_{\rm budget}(S)\le \R_{\rm reach}(S).
\label{eq:targetriskchain}
\end{equation}
Then the location of $\varrho$ identifies the lowest capability level that must change for the target $R\le\varrho$ to become attainable under the fixed pair $(P_S,\ell)$:
\begin{align}
\R_{\rm reach}(S)<\varrho
&\Longrightarrow \text{current stop family sufficient},\label{eq:targetcase1}\\
\R_{\rm budget}(S)<\varrho<\R_{\rm reach}(S)
&\Longrightarrow \text{reachability-or-higher change},\label{eq:targetcase2}\\
\R_{\rm asym}(S)<\varrho<\R_{\rm budget}(S)
&\Longrightarrow \text{ceiling-or-higher change},\label{eq:targetcase3}\\
\R_{\rm floor}(S)<\varrho<\R_{\rm asym}(S)
&\Longrightarrow \text{profile-or-structural change},\label{eq:targetcase4}\\
\varrho<\R_{\rm floor}(S)
&\Longrightarrow \text{structural change}.\label{eq:targetcase5}
\end{align}
Here ``higher'' follows the inclusion chain $\F_{\rm stop}\subseteq\F_s\subseteq\F_\infty\subseteq\mathcal D_{\mathcal J}$.
\end{proposition}
\begin{proof}
All four risks are infima of the same expected-loss functional $\mathcal L_{P_S,\ell}$ over nested decision classes. If $\varrho$ is strictly larger than the infimum over a class, the definition of infimum gives a decision in that class with risk strictly below $\varrho$. If $\varrho$ is strictly smaller than the infimum, no decision in that class can attain risk at most $\varrho$. Applying these two facts successively to $\F_{\rm stop}$, $\F_s$, $\F_\infty$, and $\mathcal D_{\mathcal J}$ yields the five cases without an attainment assumption.
\end{proof}

\begin{proposition}[Boundary exact attainment]
Fix the same $(S,P_S,\ell)$. For any one of the four capability classes
\[
\mathcal F\in\{\F_{\rm stop}(S),\F_{s(S)}(\mathcal J(S),M(S)),\F_\infty(\mathcal J(S),M(S)),\mathcal D_{\mathcal J(S)}\},
\]
write $R_{\mathcal F}:=\inf_{d\in\mathcal F}\mathcal L_{P_S,\ell}(d)$. At the equality target $\varrho=R_{\mathcal F}$, exact target feasibility within $\mathcal F$ holds if and only if the infimum $R_{\mathcal F}$ is attained in $\mathcal F$. For every $\varepsilon>0$, an $\varepsilon$-optimal decision with risk below $R_{\mathcal F}+\varepsilon$ exists by the definition of infimum.
\end{proposition}
\begin{proof}
If the infimum is attained, its optimizer is feasible at the equality target. Conversely, any $d\in\mathcal F$ with $\mathcal L_{P_S,\ell}(d)\le R_{\mathcal F}$ must attain the infimum. The final statement is the defining approximation property of an infimum.
\end{proof}

The two propositions separate capability geometry from optimization regularity. For a fixed current $(P_S,\ell)$, strict target location determines the minimum capability level whose movement is necessary for task-risk attainability; equality at a risk boundary additionally depends on exact attainment at that level. Once the current stop family contains a target-feasible decision, stopping is feasible with respect to the task-risk target, while the Bellman value $Q^*$ still determines whether stopping is trajectory-optimal.

Boundary direction is fundamentally a property of a realized non-stop stochastic transition. Stop is trajectory termination and is treated separately from capability-order movement. Let
\begin{equation}
\mathcal B_{\rm cap}:=\{\mathrm{reachability},\mathrm{budget},\mathrm{asymptotic},\mathrm{structural}\},
\qquad
\mathcal D_{\rm move}:=\{\mathrm{outward},\mathrm{inward}\},
\end{equation}
and identify $\mathfrak B_b(S)$ with the corresponding boundary in Eq.~\eqref{eq:dynamicboundaries}. For a realized non-stop transition $(S,u,S')$, define
\begin{align}
(b,\mathrm{outward})\in\Gamma(S,u,S')
&\Longleftrightarrow
\mathfrak B_b(S)\prec_{\rm cap}\mathfrak B_b(S'),\label{eq:gammaout}\\
(b,\mathrm{inward})\in\Gamma(S,u,S')
&\Longleftrightarrow
\mathfrak B_b(S')\prec_{\rm cap}\mathfrak B_b(S).\label{eq:gammain}
\end{align}
If the two endpoints are equal, the realized transition is \emph{stationary} on boundary $b$ and receives neither monotone label. If they are incomparable in the strong capability order, it is a \emph{reconfiguration} on boundary $b$. A risk-envelope outward movement may still occur when the right-column envelope expands while literal capability reconfigures; $\preceq_{\rm risk}$ records this weaker comparison and $\preceq_{\rm cap}$ records literal capability inclusion.

The action-level geometry is induced by the transition kernel $P(dS'\mid S,u)$. Define
\begin{align}
p_{b,+}(S,u)
&:=P\!\left(\mathfrak B_b(S)\prec_{\rm cap}\mathfrak B_b(S')\mid S,u\right),\\
p_{b,-}(S,u)
&:=P\!\left(\mathfrak B_b(S')\prec_{\rm cap}\mathfrak B_b(S)\mid S,u\right),\\
p_{b,\parallel}(S,u)
&:=P\!\left(\mathfrak B_b(S')\parallel_{\rm cap}\mathfrak B_b(S)\mid S,u\right),\\
p_{b,0}(S,u)
&:=P\!\left(\mathfrak B_b(S')=\mathfrak B_b(S)\mid S,u\right).
\label{eq:directionprobabilities}
\end{align}
These four events partition the realized boundary outcomes, so
\begin{equation}
p_{b,+}(S,u)+p_{b,-}(S,u)+p_{b,\parallel}(S,u)+p_{b,0}(S,u)=1.
\label{eq:fourwayboundaryprobability}
\end{equation}
The corresponding outcome labels are outward, inward, reconfiguration, and stationary. An action is almost-sure outward on boundary $b$ when $\mathfrak B_b(S)\preceq_{\rm cap}\mathfrak B_b(S')$ almost surely and strict enlargement has positive probability; almost-sure inward is defined symmetrically. Pure compute/search is typically stationary on structural, asymptotic, and budget boundaries while changing current reachability. This makes an action a probability distribution over outward, inward, reconfiguration, and stationary boundary outcomes, exactly the object evaluated by the Bellman expectation.

When a structural reconfiguration admits a common intermediate envelope $S^{\cup}$, its order-theoretic representation is
\begin{equation}
\Gamma^{\rm repr}(S,S';S^{\cup})
=\left\langle
(\mathrm{structural},\mathrm{outward}),
(\mathrm{structural},\mathrm{inward})
\right\rangle,
\label{eq:factorizedgamma}
\end{equation}
recording an enrichment--contraction representation of the endpoint replacement. This representation need not be the realized intervention path. If an engineering intervention actually traverses internal deployed states
\[
S=S^{(0)}\to S^{(1)}\to\cdots\to S^{(m)}=S',
\]
those internal states define the realized path and are treated separately.

\begin{definition}[Endpoint novelty, path enrichment, and Ockham--Chatton control]
For a realized non-stop transition define the \emph{endpoint Chattonian indicator}
\begin{equation}
\chi_C^{\rm end}(S,u,S')
:=
\mathbf 1\!\left\{
\mathcal D_{\mathcal J(S')}
\not\subseteq
\mathcal D_{\mathcal J(S)}
\right\}.
\label{eq:chattonindicator}
\end{equation}
Thus $\chi_C^{\rm end}=1$ exactly when the successor structural decision class contains literal executable capability that the inherited structural class cannot simulate. Monotone structural enlargement and incomparable structural replacement with endpoint novelty are both Chattonian at the endpoint; equality and pure contraction have $\chi_C^{\rm end}=0$.

When the realized intervention contains an explicit internal structural path $S^{(0:m)}$, define the \emph{path-enrichment indicator}
\begin{equation}
\chi_C^{\rm path}(S,u,S')
:=
\mathbf 1\!\left\{
\exists k<m:
\mathcal D_{\mathcal J(S^{(k)})}
\subsetneq
\mathcal D_{\mathcal J(S^{(k+1)})}
\right\}.
\label{eq:chattonpathindicator}
\end{equation}
Endpoint novelty asks whether the final structure introduces inherited structural possibility. Path enrichment records actual strong structural-outward steps along the deployed intervention path, using the same literal-preserving order as $\Gamma$. A common-envelope representation such as Eq.~\eqref{eq:factorizedgamma} contributes to $\chi_C^{\rm path}$ only when its intermediate state is actually traversed.

For a non-stop action define the endpoint structural-novelty probability
\begin{equation}
p_C(S,u)
:=p_C^{\rm end}(S,u)
:=\E\!\left[\chi_C^{\rm end}(S,u,S')\mid S,u\right],
\label{eq:chattonprobability}
\end{equation}
and, when an internal path is part of the action model, the path-enrichment probability
\begin{equation}
p_C^{\rm path}(S,u)
:=\E\!\left[\chi_C^{\rm path}(S,u,S')\mid S,u\right].
\label{eq:chattonpathprobability}
\end{equation}
A \emph{non-stop Ockhamian action} at state $S$ satisfies $p_C(S,u)=0$: every successor structural decision class remains literally contained in the inherited structural class. Ockhamian control also includes the stop action separately as trajectory termination. Within inherited structural possibility, non-stop Ockhamian control may expand current reachability through computation and search, expand budget realizability through a larger terminal ceiling, improve asymptotic realizability through $M\to M'$ at fixed $\mathcal J$, or contract retained capability; stopping terminates further control. A \emph{Chattonian-containing action} has $p_C(S,u)>0$ and therefore a positive probability of endpoint structural novelty. Pure Chattonian actions have $p_C=1$, while $0<p_C<1$ describes mixed stochastic structural novelty.
\end{definition}

Let
\begin{equation}
\mathcal C_{\rm coord}:=\{\mathcal J,M,s,q\}
\end{equation}
denote the four capability-related formal coordinates. For a realized non-stop transition $(S,u,S')$, define its formal coordinate signature
\begin{equation}
\operatorname{coord}(S,u,S')\subseteq\mathcal C_{\rm coord}
\label{eq:coordsignature}
\end{equation}
by including $\mathcal J$, $M$, $s$, or $q$ exactly when the corresponding primitive state coordinate changes; in particular, $q\in\operatorname{coord}(S,u,S')$ if and only if $q(S')\ne q(S)$. Changes in the derived stop family $\F_{\rm stop}$ are recorded by the reachability component of $\Gamma$, not used to infer a primitive $q$-change. At the stochastic action level, define
\begin{equation}
p_c(S,u):=P\!\left(c\in\operatorname{coord}(S,u,S')\mid S,u\right),
\qquad c\in\mathcal C_{\rm coord}.
\label{eq:coordprobability}
\end{equation}

Engineering mechanism labels are recorded separately. Let
\begin{equation}
\mathcal K_{\rm impl}:=\{\mathrm{compute},\mathrm{search},\mathrm{retrieve},\mathrm{route},\mathrm{prune},\mathrm{compile},\mathrm{expand},\mathrm{refine},\mathrm{stop}\},
\end{equation}
with $\operatorname{impl}(u)\subseteq\mathcal K_{\rm impl}$. The available menu is
\begin{equation}
\U(S_t)=\{u^{\rm stop}\}\cup\U^{\rm non-stop}(S_t).
\label{eq:actions}
\end{equation}
The mechanism signature states what engineering operation is executed; $\operatorname{coord}(S,u,S')$ states which formal coordinates actually change; $\Gamma(S,u,S')$ states the resulting capability-order movement; and $\chi_C^{\rm end}$ records representation-independent endpoint structural novelty. Thus dynamic action semantics follow
\begin{equation}
\boxed{
\operatorname{impl}(u)
\longrightarrow
\operatorname{coord}(S,u,S')
\longrightarrow
\Gamma(S,u,S')
\longrightarrow
\chi_C^{\rm end}(S,u,S')
\longrightarrow
Q^*(S,u).}
\label{eq:dynamicactionsemantics}
\end{equation}
Typical implementation effects are:
\begin{itemize}
  \item \textbf{Compute/search}: at fixed $(\mathcal J,M,s)$, change within-ceiling allocation or the active algorithmic state $q$, thereby potentially moving the derived reachability boundary $\mathfrak B_{\rm reach}$ and its stop family $\F_{\rm stop}$.
  \item \textbf{Retrieve/refine}: expose new task-relevant evidence when the acquired observation enters $\iota$ and $\phi$; retrieval over already accessible material instead changes pathwise state.
  \item \textbf{Route}: access routing can change $\mathcal J$; routing among already available models, tools, documents, or branches can change $M$, $s$, or $q$ depending on its persistent and budget effects.
  \item \textbf{Compile/redesign}: change $M$ when the persistent realization engine or resource-to-capability map changes while $\mathcal J$ is fixed.
  \item \textbf{Prune}: contract any coordinate actually removed---pathwise candidates, terminal ceiling, realization machinery, accessible evidence, or executable support.
  \item \textbf{Expand}: enlarge executable support or other structural capability when new semantic mappings enter $\mathcal D_{\mathcal J}$.
  \item \textbf{Stop}: terminate the trajectory and accept terminal loss. Pure stopping has no capability-order direction; if a stopping protocol also unloads memory, disables tools, or decommissions machinery, those additional state changes receive their own coordinate and boundary labels.
\end{itemize}

Composite actions are classified first by their realized endpoint relation and, when the deployed execution contains internal stages, by their actual path. A tool call can both exercise executable support and acquire evidence; retrieve--then--compress can refine and prune. If a replacement gains some distinctions or mappings while discarding others, its structural endpoints can be incomparable under $\preceq_{\rm cap}$ and the net transition is a structural reconfiguration. A common joint envelope may represent that replacement as an ordered enrichment--contraction pair through $\Gamma^{\rm repr}$, while $\chi_C^{\rm end}$ remains determined only by the endpoint decision classes and $\chi_C^{\rm path}$ only by actually traversed intermediate states. Evidence acquisition carries a structural effect only when the observation becomes accessible to the deployed system, enlarging $\mathcal G_t$ before any subsequent compression. Routing is classified by the same effect-based rule: access routing that changes terminally available distinctions changes $\iota_t$ or $\phi_t$, while search routing that only ranks or traverses already accessible material changes $q_t$, $\F_{\rm stop}(S_t)$, cost, and future option value. Context or memory pruning changes $\iota_t$ and possibly $\phi_t$; tool or policy gating changes $H_t$; branch or candidate pruning changes $q_t$ and the immediately reachable stop family. Algorithmic pruning leaves the joint floor, asymptotic boundary, and budget boundary unchanged unless it also changes $M_t$, accessible evidence, $\phi_t$, $H_t$, or $s_t$; it can nevertheless raise the current stopping risk by shrinking $\F_{\rm stop}$. Early exit is the stop action. The geometric language is neutral: outward and inward are monotone capability-order movements, incomparable endpoints are reconfigurations, and stop is treated as trajectory termination. Ockhamian control manages inherited structural possibility without endpoint novelty; Chattonian-containing control carries positive endpoint structural-novelty probability. Mechanism names do not determine coordinates by themselves; transition semantics determine the capability relation.

Four capability-related transformations are distinguished by their formal coordinate effects. A \emph{structural transformation} changes $\mathcal J=(\phi,H)$ and hence $\mathcal D_{\mathcal J}$. A \emph{realization-mechanism transformation} holds $\mathcal J$ fixed and changes $M$, moving $\F_\infty(\mathcal J,M)$ and potentially every finite family. A \emph{budget-ceiling transformation} holds $(\mathcal J,M)$ fixed and changes $s$, moving $\F_s(\mathcal J,M)$. A \emph{pathwise reachability transformation} holds $(\mathcal J,M,s)$ fixed and changes $q$ or $\F_{\rm stop}$. The same engineering action may change several coordinates. Separately, an \emph{evaluative transformation} changes the task or meta-level loss used to project capabilities into value. An \emph{environment transformation} changes the task law or transition kernel. Objective revision changes how a fixed capability set is valued; structural-boundary movement is determined by the capability coordinates, while environment drift changes the world against which capability is exercised. Dynamic Ockham--Chatton control operates primarily on capability transformations while its stochastic boundary effects are induced by the environment transition kernel and its action values are determined jointly by capability, objective, and environment.

An admissible controller policy $\pi$ is measurable through the meta-interface: its state-conditioned action kernel satisfies $\pi(\cdot\mid S)=\pi(\cdot\mid\mu(S))$ and is supported on the feasible menu $\U(S)$. Let $\tau^\pi=\inf\{t\ge0:u_t=u^{\rm stop}\}$, with $\inf\varnothing=\infty$. The terminal realization ceiling $s_t$ and remaining trajectory-control budget $b_t$ are separate resource coordinates. Within-ceiling compute, search, verification, resampling, or reranking can consume $b_t$ and change $q_t$ or $\F_{\rm stop}$ while leaving $s_t$ fixed; a ceiling transformation changes $s_t$ itself. Their units may differ, and no relation such as $s_t\le b_t$ is imposed. More generally their stochastic update may be written $(s_{t+1},b_{t+1})=\Psi(S_t,u_t,\xi_{t+1})$, with feasibility encoded by $u_t\in\U(S_t)$. Let $k(S,u)$ be dynamic objective cost, $c(S,u)$ trajectory-budget consumption, and $L_{\rm term}(S)$ terminal loss; $k$ and $c$ represent distinct quantities. For $0<\beta\le1$, define
\begin{equation}
J^\pi(S)=
\begin{cases}
\E_\pi^S\!\left[\sum_{t=0}^{\tau-1}\beta^t k(S_t,u_t)+\beta^\tau L_{\rm term}(S_\tau)\right],&\tau<\infty\text{ a.s. and the expectation is finite},\\
+\infty,&\text{otherwise}.
\end{cases}
\label{eq:dynamiccost}
\end{equation}
The admissible policies are $\mathfrak P_{\rm adm}(S)=\{\pi:J^\pi(S)<\infty\}$. Canonical Ockham--Chatton control uses $\beta=1$: a finite-budget condition such as the appendix theorem, or another explicit stopping mechanism, supplies termination without discounting away a positive terminal loss; admissibility specifies the policy class being compared, while the stopping mechanism supplies termination. A discounted model $0<\beta<1$ represents genuine temporal preference. For every state-preserving delay action $u^{\rm delay}$, the delay-consistency condition
\begin{equation}
k(S,u^{\rm delay})\ge(1-\beta)L_{\rm term}(S)
\label{eq:delayconsistency}
\end{equation}
prevents terminal-loss discounting by itself from making idle waiting valuable.

The running cost $k$ carries acquisition, switching, tool-call, communication, and in-trajectory holding costs. The terminal deployment burden prices the persistent structural and realization machinery that remains active or retained after the controller stops.

\begin{definition}[Admissible deployment burden]
A deployment burden functional $K_{\rm dep}$ prices the deployment state that must remain available for future use after stopping: retained memory, routing machinery, tool and interpreter state, verification infrastructure, and their coordination. It orders the cost of deployed structure: retained memory, routing, tools, verification infrastructure, and their coordination. It is admissible on a declared deployed-pair space when it satisfies four conditions.
\begin{enumerate}
  \item \textbf{Isomorphism invariance.} Operationally equivalent renamings of symbols, actions, or internal identifiers preserve burden.
  \item \textbf{Operational grounding.} Burden is a measurable function of declared deployment quantities. Writing
  \begin{align}
  z(J,M)&=(m_{\rm mem},r_{\rm route},\chi_{\rm tool},h_{\rm verify},\ldots)(J,M),\\
  K_{\rm dep}(J,M)&=\mathcal K(z(J,M)).
  \end{align}
  typical coordinates include memory footprint, routing latency, tool-selection load, interpreter state, realization-engine state, and verification cost.
  \item \textbf{Lower semicontinuity.} Under the declared topology for deployed-pair convergence $(J_n,M_n)\to(J,M)$,
  \begin{equation}
  K_{\rm dep}(J,M)\le\liminf_{n\to\infty}K_{\rm dep}(J_n,M_n).
  \end{equation}
  \item \textbf{Controlled composition.} For an admissible composition of deployed pairs,
  \begin{equation}
  K_{\rm dep}((J_1,M_1)\oplus(J_2,M_2))
  \le K_{\rm dep}(J_1,M_1)+K_{\rm dep}(J_2,M_2)+K_{\rm join}((J_1,M_1),(J_2,M_2)),
  \end{equation}
  where $K_{\rm join}$ records the additional routing and coordination overhead created by composition.
\end{enumerate}
A useful decomposition is
\begin{equation}
K_{\rm dep}(J,M)=K_{\rm struct}(J)+K_{\rm real}(M)+K_{\rm coord}(J,M),
\label{eq:deploymentburdendecomp}
\end{equation}
where $K_{\rm struct}$ prices retained interface and executable-support infrastructure, $K_{\rm real}$ prices persistent realization machinery, and $K_{\rm coord}$ prices their deployment coupling. Capability order and burden order are orthogonal: a richer risk envelope may raise maintenance cost, lower routing cost, or trade one for the other. Burden changes are therefore intervention-relative, and the action-specific quantity $\Delta K_{\rm dep}(u)$ carries this deployment effect into dynamic control.
\end{definition}

A canonical separable instantiation sets the coordination term to zero and uses
\begin{align}
K_{\rm dep}(\mathcal J,M)&:=\lambda_\phi K_\phi(\phi)+\lambda_H K_H(H)+\lambda_M K_M(M),\\
L_{\rm OC}^{\star}(S)&:=\R_{\rm fin}(S)+K_{\rm dep}(\mathcal J,M),
\label{eq:octerminal}
\end{align}
with $\lambda_\phi,\lambda_H,\lambda_M\ge0$. Here $K_\phi$ and $K_H$ instantiate structural burden through retained representation and executable-support maintenance, while $K_M$ instantiates realization burden through runtime, scheduling, compilation, and realization-specific routing infrastructure. The benchmark uses the risk infimum over the current stop family. If $d_S^\pi\in\F_{\rm stop}(S)$ is the terminal decision induced by $\pi(\cdot\mid\mu(S))$, define its realized-control gap and operational loss by
\begin{align}
G_{\rm sel}^{\pi}(S)&:=\mathcal L_{P_S}(d_S^\pi)-\R_{\rm fin}(S)\ge0,\\
L_{\rm OC}^{\rm op}(S)&:=L_{\rm OC}^{\star}(S)+G_{\rm sel}^{\pi}(S).
\label{eq:operationalterminal}
\end{align}

\begin{proposition}[Failure anatomy decomposition]
Fix a reference risk $\R_{\rm ref}(S)\le \R_{\rm floor}(S)$ under the same task law, loss, and ambient action semantics. Define
\begin{align}
C_{\rm struct}(S)&:=\R_{\rm floor}(S)-\R_{\rm ref}(S),\\
C_{\rm asym}(S)&:=\R_{\rm asym}(S)-\R_{\rm floor}(S),\\
C_{\rm fin}(S)&:=\R_{\rm budget}(S)-\R_{\rm asym}(S),\\
C_{\rm reach}(S)&:=\R_{\rm reach}(S)-\R_{\rm budget}(S).
\end{align}
For the deployed decision $d_S^\pi$, the excess risk has the exact identity
\begin{equation}
\boxed{
\mathcal L_{P_S}(d_S^\pi)-\R_{\rm ref}(S)
=C_{\rm struct}(S)+C_{\rm asym}(S)+C_{\rm fin}(S)+C_{\rm reach}(S)+G_{\rm sel}^{\pi}(S).
}
\label{eq:failureanatomy}
\end{equation}
Every term is nonnegative.
\end{proposition}
\begin{proof}
Insert the four nested risks between $\mathcal L_{P_S}(d_S^\pi)$ and $\R_{\rm ref}(S)$ and telescope. Nonnegativity follows from the nested decision classes, the reference assumption, and the definition of $G_{\rm sel}^{\pi}$.
\end{proof}

The identity turns surface error into a coordinate diagnosis. Structural failure is relative to a declared reference system; asymptotic, finite-resource, and reachability failures describe progressively narrower realization limits; selection failure is excess risk of the deployed choice over the best retained stop option. The same visible error can therefore require different interventions, while errors in different domains can share one failure coordinate.

the retained-candidate attribution identity in \hyperref[app:proofs]{Appendix Module I} gives the trajectory-level candidate-retention identity. It separates two mechanisms often merged under \emph{overthinking}: a selection/stopping gap despite preserved options, and option contraction through pruning or memory overwrite. It complements empirical findings that intrinsic self-correction can fail or turn correct answers into incorrect ones \cite{huang2023selfcorrect,tyen2024,zhang2025dark}.

The state-conditioned compensation gap $C_{\rm gap}(S)$ is the sum of pathwise reachability, finite-resource, and asymptotic realization-family gaps in Eq.~\eqref{eq:dynamicthreegap}. The gap $G_{\rm sel}^{\pi}$ prices the deployed policy's excess risk over that infimum, while $K_{\rm dep}$ prices the burden of carrying the final deployed pair $(\mathcal J,M)$. An Ockhamian contraction or realization-profile simplification can lower $K_{\rm dep}$ while increasing one or more realization gaps; a Chattonian structural enrichment can lower the joint floor while raising deployment burden, finite-realization difficulty, or selection difficulty. In the canonical deployed instantiation, $L_{\rm term}=L_{\rm OC}^{\rm op}$; oracle analysis may set $L_{\rm term}=L_{\rm OC}^{\star}$. Other applications may supply a different measurable terminal objective. Acquisition is charged in $k$, while $K_{\rm dep}$ charges only post-stopping structural and realization retention, so the canonical accounting avoids double counting by construction.

Under the dynamic-programming assumptions,
\begin{align}
V^*(S)&=\inf_\pi J^\pi(S)=\inf_{u\in\U(S)}Q^*(S,u),\\
Q^*(S,u)&=k(S,u)+\beta\E[V^*(S')\mid S,u],\qquad u\ne u^{\rm stop},\\
Q^*(S,u^{\rm stop})&=L_{\rm term}(S).
\end{align}
A one-step risk reduction is insufficient in general: expanding an interpreter may have no immediate benefit but can enable a valuable executable action one step later.

\begin{proposition}[Optimal stopping under the dynamic-programming assumptions]
Stopping is optimal at $S$ exactly when
\begin{equation}
L_{\rm term}(S)\le
\inf_{u\in\U(S)\setminus\{u^{\rm stop}\}}
\left[k(S,u)+\beta\E[V^*(S')\mid S,u]\right],
\end{equation}
with the extended-real convention $\inf\varnothing=+\infty$.
\end{proposition}
\begin{proof}
By the Bellman equation, stopping attains the infimum exactly when its value is no larger than the infimum of all non-stop action values. Since $Q^*(S,u^{\rm stop})=L_{\rm term}(S)$, substituting the non-stop definition of $Q^*$ gives the condition.
\end{proof}

\subsection{Architecture transforms the dilemma}

Architecture design can refine $\phi$, enlarge $\Pi_{\phi,H}$, replace the realization profile $M$, transform the terminal ceiling $s$, lower action costs, and improve the meta-interface. Each intervention changes a declared subset of formal coordinates and thereby moves one or more capability boundaries, costs, or future options.

Static decision equivalence $\mathcal J_1\sim_D\mathcal J_2$ identifies equality of terminal risk envelopes. Dynamic architectural equivalence is finer: it also compares reachable structural states, feasible meta-actions, transition kernels, resource and switching costs, diagnostic interfaces, and future finite families. The static terminal orders form the static component of the finer dynamic option topology.

Theorem~\ref{thm:main-budget-reversal} gives the finite analogue of universal envelope dominance: same-budget bounded-risk monotonicity holds exactly when the later closed convex finite envelope retains the earlier one. Its separating-task direction shows that nonretention generates a bounded-loss reversal witness under a full-support $X$-marginal even when terminal risk-envelope dominance holds.

\begin{example}[Strict floor improvement with finite-budget reversal]
Let $X=(U,V)$ be uniform on $\{0,1\}^2$, $Y=U\oplus V$, $\phi(X)=U$, and $\psi(X)=(U,V)$. Let both executable supports contain all deterministic binary rules on their interfaces, represented as Dirac kernels, so $\R_{\phi,H}^*=1/2$ and $\R_{\psi,H'}^*=0$. Let $\mathcal J_0=(\phi,H)$ and $\mathcal J_1=(\psi,H')$. At a deliberately limited budget, choose realization profiles $M_0,M_1$ such that $\F_s(\mathcal J_0,M_0)$ contains only the constant-zero predictor, with risk $1/2$, while $\F_s(\mathcal J_1,M_1)$ contains only the complement of XOR, with risk $1$. The refined structure has the strictly better floor but the worse finite-budget optimum because it cannot reproduce the old predictor at budget $s$.
\end{example}

The characterization locates finite-budget reversal precisely. Information or executable enrichment can improve the terminal risk order while failing to improve the finite order. A risk increase under a common finite budget certifies failure of finite risk-envelope simulation; once simulation is restored, universal finite-budget monotonicity returns. The realization-cost maps $c_{\mathcal J,M}$ anchor every such comparison to an explicitly chosen external resource space and an explicit realization mechanism.

\begin{definition}[Literal and risk-equivalent resource transformations]
For scalar common budgets $\mathcal R=\mathbb R_+$ and deployed realization pairs $(\mathcal J_i,M_i)$, define
\begin{align}
T^{\rm exec}_{(\mathcal J_1,M_1)\to(\mathcal J_2,M_2)}(s)
&:=
\inf\left\{t\ge0:
\F_s(\mathcal J_1,M_1)
\subseteq
\F_t(\mathcal J_2,M_2)
\right\},\label{eq:simulationfunctionexec}\\
T^{\rm risk}_{(\mathcal J_1,M_1)\to(\mathcal J_2,M_2)}(s)
&:=
\inf\left\{t\ge0:
\widehat{\F}_s(\mathcal J_1,M_1)
\subseteq
\widehat{\F}_t(\mathcal J_2,M_2)
\right\},
\label{eq:simulationfunction}
\end{align}
with $\inf\varnothing=+\infty$. The first map is the least target budget for literal finite-policy simulation; the second is the least budget for reproducing the source finite expected-risk envelope. Since literal inclusion implies risk-envelope inclusion,
\begin{equation}
T^{\rm risk}_{(\mathcal J_1,M_1)\to(\mathcal J_2,M_2)}(s)
\le
T^{\rm exec}_{(\mathcal J_1,M_1)\to(\mathcal J_2,M_2)}(s).
\end{equation}
The two maps align with the two columns of the capability square: $T^{\rm exec}$ is literal resource geometry and $T^{\rm risk}$ is risk-equivalent resource geometry. The remainder of the resource-geometry results use $T:=T^{\rm risk}$ unless a superscript is shown. Fixed startup costs, constant-factor overheads, and nonlinear realization burdens appear directly in the shape of these architecture--realization-to-architecture--realization resource transformations.
\end{definition}

\begin{proposition}[Resource geometry of architectural simulation]
For $\star\in\{\mathrm{exec},\mathrm{risk}\}$, the transformation $T^\star_{(\mathcal J_1,M_1)\to(\mathcal J_2,M_2)}$ is nondecreasing. Whenever $T^\star(s)<\infty$, every $t>T^\star(s)$ reproduces the corresponding source capability at budget $s$: literal inclusion for $\star=\mathrm{exec}$ and risk-envelope inclusion for $\star=\mathrm{risk}$. In the risk case this gives, for every bounded finite decision problem,
\begin{equation}
\R_{t,\mathcal J_2,M_2}^*(Q)
\le
\R_{s,\mathcal J_1,M_1}^*(Q).
\label{eq:T-risk}
\end{equation}
If the relevant infimum is attained, the same statements hold at $t=T^\star(s)$. For the risk transformation, the asymptotic multiplicative summary
\begin{equation}
\alpha_\infty((\mathcal J_1,M_1)\!\to\!(\mathcal J_2,M_2))
:=\limsup_{s\to\infty}\frac{T^{\rm risk}_{(\mathcal J_1,M_1)\to(\mathcal J_2,M_2)}(s)}{s}
\end{equation}
recovers constant-factor overhead, while an affine bound $T^{\rm risk}(s)\le\alpha s+\beta$ separates multiplicative cost from fixed startup cost.
\end{proposition}
\begin{proof}
For either column of the capability square, source families are nested in $s$ and target families are nested in $t$. If $s_1\le s_2$, every target budget that simulates the source family at $s_2$ also simulates it at $s_1$, proving monotonicity. For fixed $s$, admissible target budgets form an upward-closed set; if $t>T^\star(s)$, some admissible $t'<t$ exists and nesting yields simulation at $t$. In the risk column, Eq.~\eqref{eq:T-risk} then follows from finite-budget risk-envelope dominance.
\end{proof}

\begin{proposition}[Composition of architecture resource transformations]
For either $T^{\rm exec}$ or $T^{\rm risk}$ and three deployed architecture--realization pairs on a common scalar external resource axis, write
\begin{equation}
T^\star_{23}(r+):=\inf_{\varepsilon>0}T^\star_{(\mathcal J_2,M_2)\to(\mathcal J_3,M_3)}(r+\varepsilon),
\qquad \star\in\{\mathrm{exec},\mathrm{risk}\}.
\end{equation}
Whenever $T^\star_{(\mathcal J_1,M_1)\to(\mathcal J_2,M_2)}(s)<\infty$,
\begin{equation}
T^\star_{(\mathcal J_1,M_1)\to(\mathcal J_3,M_3)}(s)
\le
T^\star_{23}\!\left(T^\star_{(\mathcal J_1,M_1)\to(\mathcal J_2,M_2)}(s)+\right).
\label{eq:Tcomposition}
\end{equation}
If the intermediate simulation infimum is attained, or if $T^\star_{(\mathcal J_2,M_2)\to(\mathcal J_3,M_3)}$ is right-continuous at that budget, the right-limit can be replaced by ordinary composition. In particular, for either fixed $\star$, affine simulation bounds with $\alpha_{12},\alpha_{23}\ge0$,
\begin{equation}
T^\star_{12}(s)\le\alpha_{12}s+\beta_{12},
\qquad
T^\star_{23}(t)\le\alpha_{23}t+\beta_{23}
\end{equation}
compose to
\begin{equation}
T^\star_{13}(s)
\le
\alpha_{23}\alpha_{12}s
+\alpha_{23}\beta_{12}+\beta_{23}.
\end{equation}
\end{proposition}
\begin{proof}
For every $r>T^\star_{12}(s)$, the source budget-$s$ capability of the selected column is contained in the intermediate budget-$r$ capability. For every $t>T^\star_{23}(r)$, the latter is contained in the target budget-$t$ capability. Hence $T^\star_{13}(s)\le T^\star_{23}(r)$. Taking the infimum over $r>T^\star_{12}(s)$ gives Eq.~\eqref{eq:Tcomposition}. The affine statement follows by composing the two upper bounds and sending any slack in the intermediate budget to zero.
\end{proof}

Thus $T^{\rm risk}$ is a directed architecture-to-architecture resource transformation with monotonicity and composition. The same arguments apply to $T^{\rm exec}$ after replacing risk envelopes by literal finite families, so both columns of the capability square carry corresponding resource geometry. Order, deficiency, and resource transformation answer complementary questions: which architecture is capability-dominant, how closely one can approximate another, and how much resource is required to reproduce its finite capability.

\subsection{Risk decomposition and action value}

For a feasible action $u$, define the expected one-step state-conditioned risk decrease
\begin{equation}
\Delta_t(u)=\R_{\rm reach}(S_t)-
\E[\R_{\rm reach}(S_{t+1})\mid S_t,u].
\end{equation}
For each level of the capability chain, let
\begin{align}
\Delta_t^{\rm floor}(u)&=\R_{\rm floor}(S_t)-
\E[\R_{\rm floor}(S_{t+1})\mid S_t,u],\\
\Delta_t^{\rm reach}(u)&=C_{\rm reach}(S_t)-
\E[C_{\rm reach}(S_{t+1})\mid S_t,u],\\
\Delta_t^{\rm fin}(u)&=C_{\rm fin}(S_t)-
\E[C_{\rm fin}(S_{t+1})\mid S_t,u],\\
\Delta_t^{\rm asym}(u)&=C_{\rm asym}(S_t)-
\E[C_{\rm asym}(S_{t+1})\mid S_t,u].
\end{align}
Also write
\begin{equation}
\Delta_t^{\rm comp}(u)
:=\Delta_t^{\rm reach}(u)+\Delta_t^{\rm fin}(u)+\Delta_t^{\rm asym}(u).
\end{equation}

\begin{proposition}[Four-level risk-change decomposition]
Whenever the displayed quantities are measurable and integrable,
\begin{equation}
\Delta_t(u)
=\Delta_t^{\rm floor}(u)
+\Delta_t^{\rm reach}(u)
+\Delta_t^{\rm fin}(u)
+\Delta_t^{\rm asym}(u).
\label{eq:fourlevelriskchange}
\end{equation}
A pure search or compute action evaluated under a fixed task loss $\ell$ and preserving $P_t,\phi_t,H_t,M_t$, and the terminal realization ceiling $s_t$ has
\begin{equation}
\Delta_t^{\rm floor}(u)=\Delta_t^{\rm fin}(u)=\Delta_t^{\rm asym}(u)=0,
\end{equation}
so its task-risk effect is entirely a movement of current pathwise reachability. A terminal-ceiling transformation can change the finite-resource term while preserving floor and asymptotic terms. A structural or realization-mechanism transformation can change several terms simultaneously.
\end{proposition}
\begin{proof}
Apply Eq.~\eqref{eq:dynamicthreegap} at times $t$ and $t+1$ and use linearity of conditional expectation. The zero terms for pure search or computation follow because the structural decision class, realization profile, asymptotic family, budget family, and conditional task law are preserved while only $q_t$ and $\F_{\rm stop}(S_t)$ change.
\end{proof}

For the oracle deployment benchmark in Eq.~\eqref{eq:octerminal}, define
\begin{align}
\Delta_t^{\rm OC\star}(u)&=L_{\rm OC}^{\star}(S_t)-\E[L_{\rm OC}^{\star}(S_{t+1})\mid S_t,u],\\
\Delta_t^{\rm dep}(u)&=K_{\rm dep}(\mathcal J_t,M_t)-\E[K_{\rm dep}(\mathcal J_{t+1},M_{t+1})\mid S_t,u].
\end{align}

\begin{proposition}[Deployment benchmark decomposition]
Whenever the displayed quantities are measurable and integrable,
\begin{equation}
\Delta_t^{\rm OC\star}(u)
=\Delta_t^{\rm floor}(u)
+\Delta_t^{\rm reach}(u)
+\Delta_t^{\rm fin}(u)
+\Delta_t^{\rm asym}(u)
+\Delta_t^{\rm dep}(u).
\label{eq:deploymentbenchmarkfourlevel}
\end{equation}
Thus the one-step benchmark resolves into structural-floor value, current reachability, finite-resource saturation, asymptotic realization-family value, and deployment burden.
\end{proposition}
\begin{proof}
Expand $L_{\rm OC}^{\star}=\R_{\rm reach}(S)+K_{\rm dep}(\mathcal J,M)$, substitute Eq.~\eqref{eq:dynamicthreegap}, and apply linearity of conditional expectation.
\end{proof}

\begin{definition}[Inherited-structure reachability saturation path]
Fix a task loss $\ell_0$. A state family $\rho\mapsto S_\rho$, $\rho\in I=[\rho_-,\rho_+]$, is an \emph{inherited-structure reachability saturation path} relative to $\ell_0$ when
\begin{equation}
\mathcal J(S_\rho)=\mathcal J_0,\qquad M(S_\rho)=M_0,\qquad s(S_\rho)=s_0,\qquad
P_{S_\rho}=P_0,\qquad \ell_\rho=\ell_0,
\label{eq:saturationfixedcoords}
\end{equation}
so the task law, task loss, structural boundary, asymptotic boundary, and budget boundary remain fixed, while current reachability is monotone:
\begin{equation}
\rho_1\le\rho_2
\Longrightarrow
\mathfrak B_{\rm reach}(S_{\rho_1})
\preceq_{\rm cap}
\mathfrak B_{\rm reach}(S_{\rho_2})
\preceq_{\rm cap}
\mathfrak B_{\rm budget}.
\label{eq:saturationpath}
\end{equation}
All risk projections along the path are evaluated under the common pair $(P_0,\ell_0)$. Hence inclusion of the stop families directly gives
\begin{equation}
\rho_1\le\rho_2
\Longrightarrow
\R_{\rm reach}(S_{\rho_2})\le\R_{\rm reach}(S_{\rho_1}),
\qquad
C_{\rm reach}(S_{\rho_2})\le C_{\rm reach}(S_{\rho_1}),
\label{eq:saturationriskmonotonicity}
\end{equation}
because $\R_{\rm budget}=\R^*_{s_0,\mathcal J_0,M_0}(P_0,\ell_0)$ is constant. The coordinate $\rho$ therefore measures pure reachability saturation: how fully the realized trajectory has exploited capability already available under a fixed $(P_0,\mathcal J_0,M_0,s_0,\ell_0)$.
\end{definition}

For an inherited-structure reachability saturation path $\rho\mapsto S_\rho$ relative to $\ell_0$, consider a pure compute/search action $u_C$ that preserves $P_0,\mathcal J_0,M_0,s_0$ and is evaluated under the same task loss $\ell_0$, together with a structural enrichment action $u_E$, each followed by stopping. Write
\begin{align}
\Delta_C^{\rm reach}(\rho)
&:=C_{\rm reach}(S_\rho)-\E[C_{\rm reach}(S_\rho^C)],\\
\Delta_E^{\rm floor}(\rho)
&:=\R_{\rm floor}(S_\rho)-\E[\R_{\rm floor}(S_\rho^E)],\\
\Delta_E^{\rm reach}(\rho)
&:=C_{\rm reach}(S_\rho)-\E[C_{\rm reach}(S_\rho^E)],\\
\Delta_E^{\rm fin}(\rho)
&:=C_{\rm fin}(S_\rho)-\E[C_{\rm fin}(S_\rho^E)],\\
\Delta_E^{\rm asym}(\rho)
&:=C_{\rm asym}(S_\rho)-\E[C_{\rm asym}(S_\rho^E)],\\
\Delta_E^{\rm dep}(\rho)
&:=K_{\rm dep}(\mathcal J_\rho,M_\rho)-\E[K_{\rm dep}(\mathcal J_\rho^E,M_\rho^E)],
\end{align}
and let $k_C(\rho)=k(S_\rho,u_C)$ and $k_E(\rho)=k(S_\rho,u_E)$. Their net gains relative to stopping now are
\begin{align}
G_C(\rho)&=\Delta_C^{\rm reach}(\rho)-k_C(\rho),\label{eq:compute-gain}\\
G_E(\rho)&=\Delta_E^{\rm floor}(\rho)+\Delta_E^{\rm reach}(\rho)+\Delta_E^{\rm fin}(\rho)
+\Delta_E^{\rm asym}(\rho)+\Delta_E^{\rm dep}(\rho)-k_E(\rho).
\label{eq:enrich-gain}
\end{align}

A sufficient primitive condition for the gain decomposition in Eqs.~\eqref{eq:compute-gain}--\eqref{eq:enrich-gain} is
\begin{align}
\frac{d}{d\rho}\Delta_C^{\rm reach}(\rho)-k_C'(\rho)&<0,\label{eq:primitive-compute}\\
\frac{d}{d\rho}\left[\Delta_E^{\rm floor}(\rho)+\Delta_E^{\rm reach}(\rho)+\Delta_E^{\rm fin}(\rho)+\Delta_E^{\rm asym}(\rho)+\Delta_E^{\rm dep}(\rho)\right]-k_E'(\rho)&\ge0.
\label{eq:primitive-enrich}
\end{align}
It makes the compute-minus-enrichment advantage
\begin{equation}
A(\rho):=G_C(\rho)-G_E(\rho)
\end{equation}
strictly decreasing on $I$. the primitive switching derivation in \hyperref[app:proofs]{Appendix Module I} gives the derivative and crossing argument.

The first condition captures diminishing returns from further exploitation of capability already available under the inherited structure: as the current trajectory approaches the budget-realization boundary, additional search or computation closes less of $C_{\rm reach}$. The second allows structural-enrichment value to arise from floor reduction, improved pathwise reachability, a better finite-resource profile, a richer asymptotic realization family, lower deployment burden through structural or realization-profile change, or any combination whose net value does not decline with saturation.

With opposite endpoint signs, saturation of current inherited-structure reachability and sustained value from structural change therefore generate the unique compute--enrich threshold proved in the primitive switching derivation in \hyperref[app:proofs]{Appendix Module I}. This is the reachability-versus-structural instance of Ockham--Chatton control. the hysteresis prediction in Appendix~\ref{app:predictions} further shows that asymmetric switching costs split this threshold into a persistence band whose selected family depends on the currently active mode.

The one-step result is a two-action local specialization of the general Ockham--Chatton problem: it compares the marginal value of expanding current state reachability inside inherited structural capability with the marginal value of changing the inherited capability set itself. The full dynamic problem compares state-relative action families. Define
\begin{align}
\U_O(S)
&:=\{u^{\rm stop}\}\cup\{u\in\U^{\rm non-stop}(S):p_C(S,u)=0\},\\
\U_{C+}(S)
&:=\{u\in\U^{\rm non-stop}(S):p_C(S,u)>0\},\\
\U_C^{\rm pure}(S)
&:=\{u\in\U^{\rm non-stop}(S):p_C(S,u)=1\},\\
\U_M(S)
&:=\{u\in\U^{\rm non-stop}(S):0<p_C(S,u)<1\}.
\end{align}
Thus $\U_{C+}(S)=\U_C^{\rm pure}(S)\cup\U_M(S)$, and the three families $\U_O$, $\U_C^{\rm pure}$, and $\U_M$ are mutually exclusive on non-stop actions. The Bellman envelopes for the general inherited-structure-versus-structural-novelty comparison are
\begin{equation}
Q_O^*(S):=\inf_{u\in\U_O(S)}Q^*(S,u),
\qquad
Q_{C+}^*(S):=\inf_{u\in\U_{C+}(S)}Q^*(S,u).
\label{eq:ocfamilyenvelopes}
\end{equation}
When pure structural-enrichment actions are of separate interest, write
\[
Q_C^{*,\rm pure}(S):=\inf_{u\in\U_C^{\rm pure}(S)}Q^*(S,u).
\]
Here family membership is state-relative through endpoint novelty: $Q_O^*$ is the best control whose successor structural classes remain literally simulable by inherited structural capability, whereas $Q_{C+}^*$ is the best action carrying positive-probability endpoint structural novelty. Future option value, stochastic successor geometry, later structural actions, switching costs, and future resource states enter through these Bellman envelopes.

A \emph{control-family-stable saturation path} is an inherited-structure reachability saturation path on which
\begin{equation}
\U_O(S_\rho)=\U_O^0,\qquad
\U_{C+}(S_\rho)=\U_{C+}^0
\qquad\text{for every }\rho\in I.
\label{eq:familystablesaturation}
\end{equation}
On such a path define
\[
A_{\rm OC}(\rho):=Q_O^*(S_\rho)-Q_{C+}^*(S_\rho),
\]
so the Bellman comparison is taken over fixed action families. More general extensions can replace family stability with a monotone action-correspondence condition.

A single-crossing refinement assumes cross-family increasing differences along this path. Under that condition $A_{\rm OC}$ is monotone even when the minimizing action changes inside either family, and a sign change gives a unique threshold. The corresponding $\epsilon$-optimal envelope argument is collected in \hyperref[app:proofs]{Appendix Module I}.

For the operational loss in Eq.~\eqref{eq:operationalterminal}, let
\begin{equation}
\Delta_t^{\rm sel}(u)=G_{\rm sel}^{\pi}(S_t)-\E[G_{\rm sel}^{\pi}(S_{t+1})\mid S_t,u].
\end{equation}
\begin{proposition}[Operational terminal-loss decomposition]
Whenever the displayed quantities are measurable and integrable,
\begin{equation}
L_{\rm OC}^{\rm op}(S_t)-\E[L_{\rm OC}^{\rm op}(S_{t+1})\mid S_t,u]
=\Delta_t^{\rm floor}(u)+\Delta_t^{\rm reach}(u)+\Delta_t^{\rm fin}(u)
+\Delta_t^{\rm asym}(u)+\Delta_t^{\rm dep}(u)+\Delta_t^{\rm sel}(u).
\end{equation}
\end{proposition}
\begin{proof}
Substitute $L_{\rm OC}^{\rm op}=L_{\rm OC}^{\star}+G_{\rm sel}^{\pi}$ and apply the preceding decomposition.
\end{proof}

The decompositions identify which term an action changes. The Bellman value $Q^*$ completes the comparison by adding immediate transition and holding costs, future option value, and stochastic transitions.

\begin{proposition}[Risk cost of pruning]
If $\Pi_{\phi,H'}\subseteq\Pi_{\phi,H}$, then
\begin{equation}
\R_{\phi,H'}^*\ge\R_{\phi,H}^*.
\end{equation}
More precisely, if there is a $\delta>0$ such that every retained kernel $\kappa\in\Pi_{\phi,H'}$ has risk at least $\R_{\phi,H}^*+\delta$, then $\R_{\phi,H'}^*\ge\R_{\phi,H}^*+\delta$. Removing an attained optimizer alone need not raise the infimum when a retained near-optimal sequence remains. Whether pruning has positive trajectory value is determined by the full Bellman comparison, because pruning may also change the compensation gap, deployment burden, selection gap, transition law, and future option set. In the one-step special case with no further state effects, its savings must exceed the induced increase in terminal loss plus the intervention cost.
\end{proposition}
\begin{proof}
The infimum of the same risk functional over a subset is at least as large as the infimum over the original set.
\end{proof}

\section{Meta-control, recursive limits, and substrate termination}
\label{app:meta}

A controller faces a latent bottleneck state $B\in\mathcal B$ and chooses an intervention $u\in\mathcal U(S)$ under a meta-loss $L_{\rm meta}(u,B)$ that prices task error, resource use, deployment burden, and future option value. Computation, observation, memory restoration, representation change, tool use, contraction, and stopping are different interventions in this common decision problem. The meta-interface map $\mu:\mathcal S\to\mathcal Z_\mu$ produces the observation $\mu(S)$ available for structural-action selection; the controller policy is the distinct kernel $\pi(\cdot\mid\mu(S))$.

\subsection{Diagnostic and action-selection floors}

Let the coarse diagnostic categories be
\begin{equation}
\mathcal C:=\{\mathrm{comp},\mathrm{obs},\mathrm{mem},\mathrm{repr},\mathrm{obj}\},
\end{equation}
and let $c:\mathcal B\to\mathcal C$ map a latent bottleneck state to its coarse error-source category. Diagnosis predicts $c(B)$, whereas action selection chooses an optimal member of $\U(S)$ under $L_{\rm meta}(u,B)$. These are different decision problems: the same coarse category can contain bottleneck states with different optimal actions under different budgets, costs, and future option values, while distinct bottlenecks can share one optimal intervention.

\begin{theorem}[Meta-diagnostic collision floor]
Consider two equiprobable states $S_0,S_1$ with latent bottlenecks $B_0,B_1$. If $\mu(S_0)=\mu(S_1)$ but $c(B_0)\ne c(B_1)$, every diagnosis kernel measurable through $\mu$ has average error at least $1/2$. For induced meta-interface laws $P_0^\mu,P_1^\mu$, the optimal binary diagnostic error is $(1-\TV(P_0^\mu,P_1^\mu))/2$.
\end{theorem}
\begin{proof}
Apply the exact and approximate collision result with meta-interface $\mu$ as the observation and coarse bottleneck label $c(B)$ as the target.
\end{proof}

\begin{theorem}[Meta-action collision floor]
Suppose $\mu(S_0)=\mu(S_1)$, $\U(S_0)=\U(S_1)$, each optimal-action set
\begin{equation}
\Opt(S)=\arg\min_{u\in\U(S)}Q^*(S,u)
\end{equation}
is nonempty and measurable, and $\Opt(S_0)\cap\Opt(S_1)=\varnothing$. Any controller policy $\pi(\cdot\mid\mu(S))$ has average optimal-action error at least $1/2$ on the pair.
\end{theorem}
\begin{proof}
The interface equality forces $\pi(\cdot\mid\mu(S_0))=\pi(\cdot\mid\mu(S_1))$. Because the two optimal sets are disjoint, the common distribution assigns total probability at most one to them; average optimal-action success is therefore at most $1/2$.
\end{proof}

\begin{definition}[State-relative control-family label and optimal-family set]
For a feasible state--action pair define
\begin{equation}
d(S,u):=
\begin{cases}
O, & u=u^{\rm stop}\ \text{or}\ \bigl(u\in\U^{\rm non-stop}(S),\ p_C(S,u)=0\bigr),\\[3pt]
C, & u\in\U^{\rm non-stop}(S),\ p_C(S,u)=1,\\[3pt]
M, & u\in\U^{\rm non-stop}(S),\ 0<p_C(S,u)<1.
\end{cases}
\label{eq:statefamilylabel}
\end{equation}
Here $C$ denotes pure/almost-sure endpoint structural novelty, while the notation $C+$ denotes the broader positive-probability family $p_C>0$. For nonempty
\[
\Opt(S):=\arg\min_{u\in\U(S)}Q^*(S,u),
\]
define the \emph{optimal-family set}
\begin{equation}
G(S):=\{d(S,u):u\in\Opt(S)\}\subseteq\{O,C,M\}.
\label{eq:optimalfamilyset}
\end{equation}
Thus $G(S)=\{M\}$ means that a mixed stochastic action is optimal, whereas $G(S)=\{O,C\}$ means that pure Ockhamian and pure Chattonian actions tie in Bellman value. These cases remain distinct.
\end{definition}

the recursive diagnostic prediction in Appendix~\ref{app:predictions} proves the disjoint-family collision floor, characterizes breakability by the probe-induced total variation, and derives the exact net diagnostic value $\lambda\TV(K_0^a,K_1^a)/2-c(a)$. This turns ``diagnose the bottleneck'' into an interventional requirement: a useful diagnostic action induces observably different response laws under competing bottleneck hypotheses, while the Bellman controller compares that information value with acquisition cost.

For a general bottleneck space and intervention menu, let $R_{\rm meta}^*(\mu)$ denote the Bayes meta-risk under the current diagnostic interface and let a probe $a$ return $Z_a$. Its general value of diagnostic information is
\begin{equation}
\operatorname{VoD}_{\rm gen}(a)
:=R_{\rm meta}^*(\mu)
-\E\!\left[R_{\rm meta}^*(\mu,Z_a)\right]
-c(a).
\label{eq:generalvod}
\end{equation}
The general meta-decision problem is the primary form: latent bottleneck state $B$ is viewed through the finite meta-interface $\mu(S)$ and mapped to an intervention $u$. The binary total-variation formula above is its closed-form two-bottleneck, two-control-family projection. Equation~\eqref{eq:generalvod} places clarification, retrieval for diagnosis, active testing, introspective probes, and tool probing in the same meta-decision language.

This separates \emph{identifiability} from \emph{economic diagnosability}: a probe can reveal the bottleneck yet still be irrational when its information value is smaller than its acquisition cost. In the full dynamic model the same comparison is absorbed into the Bellman action value, where diagnosis can also alter future option value.

These results produce the recursive structure:
\begin{equation}
\boxed{\text{task-interface limit}\ \longrightarrow\ \text{structural control}\ \longrightarrow\ \text{meta-interface limit}.}
\end{equation}

\begin{corollary}[Level-relative meta-interface refinement]
Fix the task law $P$, task interface $\phi$, executable support $H$, dynamic environment, and cost functional. Reserve the additional signal in $\nu$ for meta-action selection and exclude it from the terminal semantic decision rule; direct terminal availability assigns the signal to the task interface $\phi$. Under this separation, refining the controller interface from $\mu$ to $\nu$, with $\mu=r\circ\nu$ almost surely, leaves the task-level joint floor $\R_{\phi,H}^*(P)$ unchanged. If every policy measurable through $\mu$ can be reproduced through $\nu$, then the optimal dynamic value under $\nu$ is no larger than that under $\mu$.
\end{corollary}
\begin{proof}
Because the new signal is excluded from the terminal semantic rule, the task-level floor depends on the fixed $(P,\phi,H)$. The refinement can ignore its added signal and reproduce every old control policy, so the refined policy class contains the old class; taking the infimum of the same trajectory cost over the larger class cannot increase value.
\end{proof}

Meta-interface refinement is therefore level-relative: it is Chattonian enrichment for the controller and a realization improvement for task-level intervention selection, but it cannot by itself lower the task-level joint floor. The refined controller must still decide when to use additional signals, tools, or internal computation. This recurrent structure explains why architecture can relocate the dilemma without removing the need for control. Finite budgets, external feedback, hard constraints, and specialized controllers terminate particular recursions by supplying an external stopping rule or a richer discriminating state.

\subsection{Substrate termination and human control}

Human cognition provides an architectural comparison. Core systems for objects, actions, number, and space organize task-relevant distinctions early in cognition \cite{spelke2007}; affordance competition and sensorimotor selection bias action before a fully specified symbolic plan \cite{cisek2007}; epistemic action changes the world to simplify cognition \cite{kirsh1994}; and effort allocation supplies a specialized control mechanism \cite{shenhav2013,lieder2017}. The dissociation between language and non-linguistic thought supports the architectural separation \cite{fedorenko2024,mahowald2024}, while metacognitive sensitivity remains a distinct and imperfect capacity \cite{fleming2012,maniscalco2012}. Together, these findings motivate the architectural premises of the \emph{substrate-termination hypothesis}. The hypothesis states that non-linguistic substrate organization resolves many coarse Ockham--Chatton direction choices before explicit language-mediated meta-control.

Let $S\sim\varrho$ for a specified distribution over control states, and let
\begin{equation}
\eta:\mathcal S\to\mathcal B_{\rm sub}
\end{equation}
denote the task-relevant non-linguistic substrate state supplied by perceptual, sensorimotor, memory, and domain-specific components. The controller observation $\mu(S)$ and architectural substrate $\eta(S)$ occupy distinct levels; $\eta(S)$ can supply a direct route to action-family choice. On states with singleton pure optimal-family requirements, define $D_{\rm OC}(S)=O$ when $G(S)=\{O\}$ and $D_{\rm OC}(S)=C$ when $G(S)=\{C\}$. States with $G(S)=\{M\}$ or $|G(S)|>1$ remain in the richer family-selection problem. The substrate is directionally sufficient relative to $\varrho$ on the singleton pure-family states when a measurable map $h:\mathcal B_{\rm sub}\to\{O,C\}$ satisfies
\begin{equation}
D_{\rm OC}(S)=h(\eta(S))\quad\varrho\text{-almost surely}.
\end{equation}
A direct architectural route from $\eta(S)$ to the correct action family terminates the directional recursion at the substrate level. Selection within that family can still recruit effort allocation, confidence monitoring, and higher-order control.

A substrate directional collision occurs when $\eta(S_0)=\eta(S_1)$ while $D_{\rm OC}(S_0)\ne D_{\rm OC}(S_1)$. Its resolution requires a refined substrate state, a higher-level diagnostic interface, or an external interaction that exposes the missing distinction. Within this framework, substrate and diagnostic collisions generate a principal class of human metacognitive errors; mis-specified objectives, noisy control costs, and finite realization generate additional errors above those floors. Substrate termination remains the architectural mechanism that resolves many coarse directional choices before explicit symbolic control.

The substrate-termination hypothesis predicts architecturally truncated recursion across many human tasks. When $\eta$ is directionally sufficient, perceptual, sensorimotor, memory, and domain-specific systems resolve the coarse action family before explicit language-mediated meta-reasoning. When substrate collisions occur, higher-order diagnosis or external interaction becomes necessary. Language-first AI concentrates more distinction formation and intervention selection at symbolic and meta-control interfaces, making Ockham--Chatton recursion more explicit and increasing the value of multimodal observation, persistent memory, active experimentation, and specialized control components.

\paragraph{Human agency and boundary authorship.}
Machine capability unfolds inside interfaces, goals, executable semantics, evaluators, and institutions that humans construct and remain responsible for. As control within supplied boundaries becomes increasingly effective, human agency concentrates on boundary authorship: defining worthwhile ends, creating task-relevant distinctions, grounding norms in lived and social contexts, and accepting responsibility for structural change. The recursive analysis sharpens this role. Human contribution enters through substrate interaction, objective revision, social commitment, and institutional construction that redefine the state and action space inherited by subsequent control.

\end{textAtEnd}

\begin{textAtEnd}[category=predictions]
\section{Module II: Predictions and new research directions}
\label{app:predictions}
The predictions in this module are consequences of the inherited/finite capability distinction, Theorem~\ref{thm:main-budget-reversal}, and successor-class control; they are ordered from direct representation consequences to broader architectural forecasts.

\subsection{Prediction 1: workload-relative recovery and held-out risk transfer}
At a fixed resource level, elicited executable policies give an inner approximation to $\widehat\F_s$ and bounded-loss probes give an outer approximation on a declared finite evaluation slice. Corollary~\ref{cor:main-behavioral-identification} supplies the finite-slice Hausdorff and statistical rates, while Theorem~\ref{thm:main-workload-tail-radius} identifies the residual coverage radius when the task space is open-ended. Countable workloads admit tail truncation followed by finite-head recovery; non-atomic workload mass leaves a finite-information floor on the unrestricted singleton subclass. The inner--outer Hausdorff gap remains a stopping certificate inside the selected slice, and Eq.~\eqref{eq:main-heldout-risk-certificate} transfers that geometric certificate to previously unqueried bounded-loss objectives at the same evaluation resolution.

\subsection{Prediction 2: workload-stable finite reversals}
Failure of finite-envelope inclusion under a structurally richer successor gives every prescribed workload with full-support $X$-marginal a bounded objective that exposes a strict reversal. Architecture enrichment can therefore improve terminal capability while degrading some finite-budget deployment objectives until a compatibility path or additional budget restores the lost envelope.

\subsection{Prediction 3: restoration-budget curves for capability upgrades}
For a scalar external resource axis, the map $T^{\rm risk}_{1\to2}(s)$ defined in Eq.~\eqref{eq:main-resource-restoration} gives the minimum target-side budget required for a successor to recover the predecessor's risk-equivalent finite capability. Comparing restoration curves across model, context, retrieval, and tool upgrades separates structural enrichment from realization overhead and yields a measurable development target.

\subsection{Prediction 4: persistent positive successor value dominates fixed realization increments after saturation}
Nested finite realization makes the benefit of any fixed extra resource increment converge to zero. A boundary-changing successor with persistent positive complete Bellman gain therefore eventually dominates that fixed realization increment. Single-crossing conditions sharpen eventual dominance into one threshold; direction-change costs widen it into a hysteresis band.

\subsection{Prediction 5: recursive control is limited by diagnostic information}
When capability states requiring different directions collide at the controller interface, meta-deliberation through the same interface preserves the collision. Diagnostic probes reduce direction error according to the total variation they create between induced observation laws. This predicts monotone gains from genuinely separating diagnostic information, while additional processing of an unchanged meta-interface remains at the same collision floor.

\subsection{Prediction 6: development and deployment can exhibit a cross-level control imbalance}
\section{Cross-level Ockham--Chatton imbalance}

Across the design families surveyed here, system development changes both realization profiles and structural capability, while deployment exposes controls over terminal ceilings and pathwise activation inside the resulting system. We define the asymmetry operationally as a recurrent allocation of control across levels. Parameters, training data, learned optimization, compilers, inference procedures, and runtime design can alter the realization profile $M$; architectural additions such as modalities, memory channels, tools, interpreters, and action schemas can additionally change $\phi$ or $H$. Deployment then exposes the inherited pair $(\mathcal J,M)$ through prompts, selected context, tool gateways, terminal realization ceilings, and trajectory-level search state.

\subsection{Deployment concentrates Ockhamian control}

Deployment methods compress prompts, prune context, filter retrieval, sparsify attention, route among models and tools, adapt reasoning length, and stop search when additional work has low value. LLMLingua compresses prompts \cite{jiang2023}; selective attention learns query-dependent contextual sparsity \cite{zhang2024ssa}; RouteLLM routes requests across models, while BEST-Route jointly chooses a model and a response-sampling budget \cite{routellm2025,bestroute2025}; AdaptThink and REFRAIN adapt the amount of reasoning and the stopping decision \cite{adaptthink2025,refrain2026}. These methods primarily move current reachability or the budget-realization boundary when accessible evidence, executable support, and the asymptotic realization mechanism remain fixed.

Structural enrichment supplies the complementary direction. FLARE, Self-RAG, Adaptive-RAG, and AutoSearch make retrieval timing, retrieval depth, or retrieval strategy state dependent \cite{flare2023,selfrag2024,adaptiverag2024,autosearch2026}; when retrieval exposes a new task-relevant distinction, it moves the structural boundary. ReAct couples reasoning with environment interaction \cite{yao2023}; Toolformer enlarges executable API support \cite{schick2023}; TECTON and MeCo adapt tool choice or invocation from metareasoning signals \cite{tecton2025,meco2025}. The same engineering mechanism can therefore act on different capability coordinates: selection among already exposed options changes current reachability, ceiling authorization changes budget realizability, persistent routing or inference redesign can change $M$ and asymptotic realizability, and exposure of new evidence or semantics changes structural capability. The control-family status is state-relative: the transition is Ockhamian when its successor structural decision class remains contained in the inherited class and is Chattonian-containing when endpoint structural novelty has positive probability.

\subsection{System development combines realization-profile change and structural construction}

Frontier development contains distinct formal operations. Training and system redesign can change the realization profile $M$ by changing the learned inference mechanism, optimization dynamics, compilation path, or runtime resource-to-capability map; compute-optimal scaling studies one important development-time allocation of model size and training tokens that can reshape this profile \cite{hoffmann2022}. Architectural additions such as modalities, persistent memory access, tools, interpreters, and action schemas can change effective distinctions or executable possibilities and therefore change $\mathcal J$. Deployment-time authorization of additional terminal resources changes $s$, while computation under fixed $(\mathcal J,M,s)$ changes $q$ and current reachability. Compression, distillation, pruning, sparsity, and efficient implementations are classified by which of these formal coordinates they actually change.

The cross-level problem appears when a broad trained support is controlled through narrow prompt-level selection at deployment. Two recurrent mismatches follow:
\begin{enumerate}
  \item a broad trained support remains underused because the runtime interface hides when a specialized capability is relevant;
  \item an expanded runtime interface overloads a controller that has not learned to price structural actions against computation, contraction, and stopping.
\end{enumerate}

\begin{takeaway}
\centering\bfseries
Generality requires bidirectional control at both levels: learn how to acquire data, tools, distinctions, and strategies, and learn when to use, ignore, preserve, prune, replace, or stop them.
\end{takeaway}

We call this the \emph{cross-level Ockham--Chatton imbalance}. It predicts direction errors in systems that combine broad capacity construction with restrictive trajectory-level activation: excessive computation under missing information, excessive pruning under novelty, indiscriminate expansion under realization difficulty, and delayed switching among these regimes.

\end{textAtEnd}

\begin{textAtEnd}[category=interpretation]
\subsection{Safety, inverse scaling, and multi-agent protocols}
Inverse-scaling and U-shaped scaling results show that larger systems can rank worse on particular tasks even as other capabilities improve \cite{mckenzie2023inverse,wei2022ushape}. The capability application separates terminal expansion, finite-envelope retention, and deployed selection. The additional consequence is diagnostic: common-budget reversal points to missing finite retention, while degradation after retention points to selection or stopping.

Constrained decoding supplies an executable-support application. Grammar- and parser-constrained generation exclude inadmissible continuations through the executable mechanism \cite{scholak2021picard,geng2023grammar}; learned refusal and reranking operate through selection. The framework therefore assigns different intervention coordinates to hard executable exclusion and learned control, which suggests separate stress tests for support and selector behavior.

Multi-agent sampling, voting, consensus, and debate alter the complete protocol envelope and its selector \cite{li2024agents,kaesberg2025,cui2026}. A retained single-agent or earlier-protocol fallback embeds a predecessor route into the richer protocol and directly tests the retention principle; aggregation experiments then isolate the additional selector effect. This application converts team-size comparison into the same finite-envelope and retention measurements used for model, tool, and context upgrades.
\end{textAtEnd}

\begin{textAtEnd}[category=experiments]
\section{Module IV: Controlled identification}
\label{app:experiments}
Corollary~\ref{cor:main-behavioral-identification} defines the behavioral measurement target. The controlled studies here identify coordinate movement through matched interventions: elicited executable decisions populate inner capability approximations, while collision certificates, common-ceiling comparisons, retention interventions, endpoint crosses, and diagnostic probes constrain how those approximations can move. This design targets the generating coordinates $\phi,H,M,s$ and $\mu$ without requiring exhaustive reconstruction of the complete finite envelope.
\subsection{Matched signatures of boundary movement}

\begin{table*}[t!]
\caption{\textbf{Matched intervention signatures.} Controlled rows use 16 seeds, with 64 paired seeds for budget and retention. Values are exact accuracy, identification score, or correct counts.}
\label{tab:main-results}
\centering
\scriptsize
\setlength{\tabcolsep}{3pt}
\renewcommand{\arraystretch}{1.03}
\begin{tabularx}{\textwidth}{@{}p{0.065\textwidth}p{0.105\textwidth}X>{\centering\arraybackslash}p{0.10\textwidth}>{\centering\arraybackslash}p{0.14\textwidth}p{0.18\textwidth}@{}}
\toprule
Source & Family & Matched contrast & Baseline & Intervention & Signature \\
\midrule
Legacy & Interface & hidden; state-exposing pairs & $.500$ paired & $1.000$ exposed & floor and refinement \\
 & Executable & restricted; extensible interpreter & $0/16$; $2/16$ & $1/16$; $16/16$ & support--reasoning complementarity \\
 & Computation & pointer; affine; finite-state; 3SUM & $0,5,9,8/16$ & $14,16,16,16/16$ & realization inside structure \\
\midrule
Control & Modular & 12-step computation & $0/16$ disabled & $16/16$ low; $13/16$ high & reachability; budget exhaustion \\
 & Budget & high/low across ceilings & 1K: $48/35$ & 2K: $52/62$; 8K: $64/63$ & reversal; retained recovery \\
 & Retention & external prefix retention & $\Delta R_{\rm oracle}=-1/64$ & $\Delta G_{\rm sel}^{\pi}=+3/64$ & $\Delta R=+2/64$ exactly \\
 & Endpoint & label $\times$ successor class & search $\leq .500$ & access $1.000$ & endpoint effect; label gap $\leq.031$ \\
 & Meta & coarse; $.75$; exact probe & $\leq.500$ coarse & $.656$--$.719$; $.844$--$1$ & recursive collision release \\
 & Routing & saturation $\rho:.1\rightarrow.9$ & THINK through $.5$ & structural from $.7$ & threshold switching \\
\bottomrule
\end{tabularx}
\end{table*}

The experiments calibrate the theory through matched boundary movements. DeepSeek-V4-Flash protocols vary reasoning, evidence, executable support, memory, and a common output ceiling under fixed scoring rules. Table~\ref{tab:main-results} assembles the resulting coordinate-identification signatures. Appendix~\ref{app:legacy-experiments}, \emph{Legacy API reconstruction and routing stress tests}, gives prompts, seeds, rerun policy, validity audits, confidence analyses, and routing stress tests.

The design uses matched boundary interventions. Computation changes the trajectory; collision pairs opposite targets under one effective input; refinement exposes the decisive bit; executable expansion changes interpreter semantics; ceiling and retention manipulate the finite envelope and selector. Two theory-specific crosses complete the chain. The \emph{endpoint cross} assigns both RETRIEVE and TOOL labels to search-only and state-revealing successors. The \emph{meta-interface cross} holds an ambiguous controller card fixed, then supplies $75\%$-reliable or exact bottleneck probes. Each contrast targets one coordinate of the broader decomposition in Eq.~\eqref{eq:app-broader-failure-anatomy} or one transition in the inherited-boundary loop.

The matched contrasts reveal the typed hierarchy and its recursion directly. With decisive state and execution rules visible, thinking closes algorithmic and serial-computation gaps; hidden twins stay at the structural ceiling and executable expansion changes the value of the same reasoning intervention. In the endpoint cross, both mechanism labels attain $1$ after state access and never exceed $.5$ after search-only successors, so realized successor class predicts capability across both labels. At the control layer, every effort mode moves monotonically from its coarse collision score through the $.75$ probe to the exact probe. The same information geometry therefore governs task decisions and intervention direction.

\end{textAtEnd}

\begin{textAtEnd}[category=experiments]
\section{Legacy API reconstruction and routing stress tests}
\label{app:legacy-experiments}
\subsection{Design}

The study tested one production model, DeepSeek V4-Flash, through an API in direct and thinking modes. Prompts, seeds, decoders, scoring, stage identifiers, and output schemas were frozen before aggregation. The protocol contains 778 jobs and 1,652 authoritative stages. Records are append-only; 95 initial length terminations were rerun with the same prompt and a 16,384-token ceiling, and the latest record per stage is authoritative.

The task families isolate four theoretical signatures:
\begin{enumerate}
  \item \textbf{Compensation}: hold the external interface fixed and add test-time reasoning.
  \item \textbf{Collision}: construct observational or memory twins with identical visible input and opposite targets.
  \item \textbf{Refinement}: reveal an intervention, restore omitted memory, or enlarge executable semantics.
  \item \textbf{Routing stress}: compare saturated retrieval with relational and authority-sensitive long-context selection.
\end{enumerate}

\begin{figure}[t]
\centering
\begin{tikzpicture}
\begin{groupplot}[
  group style={group size=3 by 1,horizontal sep=0.85cm},
  width=0.28\textwidth,height=5.05cm,
  ymin=0,ymax=1.08,
  ytick={0,0.25,0.5,0.75,1},
  yticklabel style={font=\sffamily\scriptsize},
  xticklabel style={font=\sffamily\tiny,align=center},
  title style={font=\sffamily\bfseries\scriptsize,text=ink,align=center},
  label style={font=\sffamily\scriptsize},
  grid=major,grid style={draw=gray!16},
  axis line style={draw=gray!45},
  tick style={draw=gray!45},
  legend style={font=\sffamily\scriptsize,draw=none,fill=none,at={(0.5,-0.29)},anchor=north,legend columns=2},
  enlarge x limits=0.18,ybar]
\nextgroupplot[title={(a) Computation\\closes a gap},symbolic x coords={Pointer,Affine,FSM,3SUM},xtick=data]
\addplot[fill=blue!70,draw=blue] coordinates {(Pointer,0) (Affine,0.3125) (FSM,0.5625) (3SUM,0.5)};
\addplot[fill=teal!78,draw=teal] coordinates {(Pointer,0.875) (Affine,1) (FSM,1) (3SUM,1)};
\legend{Direct,Thinking}
\nextgroupplot[title={(b) Interfaces\\determine the floor},symbolic x coords={Obs. twin,Passive,Active ID},xtick=data]
\addplot[fill=blue!70,draw=blue] coordinates {(Obs. twin,0.5) (Passive,0.81) (Active ID,0.7292)};
\addplot[fill=teal!78,draw=teal] coordinates {(Obs. twin,0.5) (Passive,1) (Active ID,1)};
\nextgroupplot[title={(c) Semantic support\\is an interface},symbolic x coords={Restricted,Extensible},xtick=data]
\addplot[fill=blue!70,draw=blue] coordinates {(Restricted,0) (Extensible,0.125)};
\addplot[fill=teal!78,draw=teal] coordinates {(Restricted,0.0625) (Extensible,1)};
\end{groupplot}
\end{tikzpicture}
\caption{Matched signatures. Thinking improves executable computation when decisive distinctions and strategies are available, remains at the exact twin floor when they are not, and becomes effective after executable support is expanded.}
\label{fig:mainresults}
\end{figure}

\subsection{Compensation inside a fixed structure}

Thinking improves pointer chasing from $0/16$ to $14/16$, affine induction from $5/16$ to $16/16$, finite-state execution from $9/16$ to $16/16$, and 3SUM from $8/16$ to $16/16$. Passive causal intervention accuracy rises from $162/200$ to $200/200$, and simulator identification from $35/48$ to $48/48$. These results are consistent with a reduced compensation gap: the relevant distinctions were present, and extra serial computation realized better rules over them.

\subsection{Collision floors under more thinking}

The observation-twin construction pairs two causal worlds with identical visible evidence and opposite labels. Across 100 stages per mode, the paired score is exactly $0.5$ in both direct and thinking modes. The memory experiment shows the same pattern. Full-context instances score $48/48$ in both modes. After truncation creates paired histories with the same visible suffix and different decisive earlier bits, valid responses receive a $0.5$ paired score: $24/24$ direct responses and all 13 valid thinking responses. Restoring memory moves the matched score from $0.5$ to $1.0$.

\subsection{Executable support and finite-realization complementarity}

Under a restricted interpreter, direct decoding produces no valid held-out executable prediction and thinking produces $1/16$. Under an extensible interpreter, validity rises to $2/16$ direct and $16/16$ with thinking. Under the observed deployed profile and ceiling, expansion lowers the structural barrier and reasoning closes the remaining pathwise and finite-realization gap. Their conjunction produces the full gain.

\subsection{Long-context routing as an Ockham--Chatton stressor}

Ordinary key--value and rule retrieval remains perfect through 2,048 distractors in both modes. A relational suite then varies competitors over 2,048, 8,192, and 32,768. End-to-end direct accuracy declines from $.714$ to $.167$; thinking reaches $.500$ at 32,768 and preserves $2/2$ accuracy for relative anchors and eight-hop graphs, while unnumbered ordinal access is $0/6$ direct across scales.

\begin{figure}[t]
\centering
\begin{tikzpicture}
\begin{axis}[
  width=\columnwidth,height=5.2cm,
  xmode=log,log basis x=2,
  xmin=1500,xmax=43000,
  xtick={2048,8192,32768},xticklabels={2K,8K,32K},
  ymin=0,ymax=1.05,ytick={0,.25,.5,.75,1},
  xlabel={number of competitors},ylabel={end-to-end accuracy},
  label style={font=\sffamily\scriptsize},ticklabel style={font=\sffamily\scriptsize},
  grid=both,grid style={gray!16},axis line style={gray!45},
  legend style={font=\sffamily\scriptsize,draw=none,fill=none,at={(0.5,-0.30)},anchor=north,legend columns=3}]
\addplot[very thick,blue,mark=*,mark options={fill=blue}] coordinates {(2048,.7143) (8192,.5714) (32768,.1667)};
\addplot[very thick,teal,mark=square*,mark options={fill=teal}] coordinates {(2048,.5714) (8192,.7692) (32768,.5)};
\addplot[very thick,coral,dashed,mark=triangle*,mark options={fill=coral}] coordinates {(2048,1) (8192,1) (32768,1)};
\legend{Stress direct,Stress thinking,Ordinary retrieval}
\end{axis}
\end{tikzpicture}
\caption{Saturated retrieval and relational control are different tests. Adding context expands the nominal interface, but routing under the deployed profile and ceiling can enlarge the realization burden.}
\label{fig:stress}
\end{figure}

The legacy data establish three static signatures; the theory-specified extension tests the same coordinates under controlled interventions.
\end{textAtEnd}

\begin{textAtEnd}[category=experiments]
The 3,248-record controlled study uses frozen generators, exact-match scoring, and requests disabled, low-, and high-effort modes \cite{deepseekapi2026}. Condition labels record the requested API configuration; the generator and append-only responses preserve the model, thinking toggle, effort field, token ceiling, and response metadata for each run. Controlled families use 16 seeds; budget and retention use 64 paired seeds. Appendix~\ref{app:retained} gives the candidate trajectory identity. Appendix~\ref{app:extensions} gives endpoint and diagnostic cells, ceiling pairings, confidence intervals, termination audits, factorial controls, topology checks, and routing curves.

\paragraph{Observed boundary signature.}
Modular computation moves $q$; hidden twins preserve $\mathcal J$ and its $.5$ floor; one-bit reveal refines $\phi$; interpreter expansion changes $H$; and a common ceiling isolates finite completion. Their conjunction traces realization, structural invariance, discontinuous refinement, executable complementarity, and finite reversal. The endpoint and meta-interface crosses carry the same identification logic from deployed capability to structural transitions and recursive control.

The targeted extensions expose transition, recursion, finite realization, and selection. Crossing RETRIEVE/TOOL labels with search/access endpoints yields access score $1$, search score at most $.5$, access label gap $0$, and maximum search label gap $.031$. Coarse meta-collisions stay at or below $.5$; $.75$-reliable probes reach $.656$--$.719$, and exact probes $.844$--$1$, monotonically in every mode. Low effort leads at 2,048 tokens ($62/64$ versus $52/64$, paired $p=.00635$), while high effort reaches $64/64$ at 8,192; retention raises the 2,048 envelope to $63/64$. Misleading feedback gives $\Delta R=2/64=\Delta R_{\rm oracle}+\Delta G_{\rm sel}^{\pi}=(-1+3)/64$, assigning the residual exactly to selection.
\end{textAtEnd}

\begin{textAtEnd}[category=experiments]
\section{Retained-candidate and constraint checks}
\label{app:retained}

The retained-candidate protocol stores every earlier answer and measures oracle-prefix risk and deployed selection risk over four revision rounds. Under neutral feedback, the first revision improves both from $1/64$ to $0$ with zero selection gap. Under misleading feedback, round 0 has $(R_{\rm oracle},G_{\rm sel}^{\pi})=(1/64,0)$ and round 1 has $(0,3/64)$, yielding the exact increment $2/64=-1/64+3/64$. Round 3 similarly has deployed risk $2/64=0+2/64$, while exact verification and best-of-$k$ attain $64/64$. Verifier-grounded feedback keeps both terms at zero. The contrast isolates selection from candidate generation exactly as the retained-candidate attribution identity in \hyperref[app:proofs]{Appendix Module I} predicts; related feedback-sensitive correction results appear in \cite{huang2023selfcorrect,tyen2024,zhang2025dark}.

The protected-marker suites record zero soft-policy violations in all three modes under both ordinary and instruction-conflict prompts. Grammar projection and the deterministic external filter also record zero violations by construction. The observed suite therefore supplies a compliance control, while the executable mechanisms instantiate the set-inclusion guarantee independently of learned compliance.
\end{textAtEnd}

\begin{textAtEnd}[category=originality]
\section{Theoretical positioning and nearest-neighbor comparison}
\label{app:originality}
\label{app:core-comparison}
This section places the sharp main-text results beside their closest mathematical and LLM-theoretic neighbors, then records the contribution hierarchy used throughout the appendix.

\subsection{Finite-budget deployment as a decision representation}
Theorem~\ref{thm:main-budget-reversal} combines a classical convex-analytic proof engine with a resource-indexed capability representation. A deployed pair $(\mathcal J,M)$ induces one finite envelope $\widehat\F_s(\mathcal J,M)$ at every externally declared resource level $s$. Universal bounded-loss comparison identifies exactly this envelope, so the family $s\mapsto\widehat\F_s(\mathcal J,M)$ gives a decision-theoretic representation of finite deployment capability.

Finite-dimensional convex separation converts envelope noninclusion into a separating linear objective. Statewise shifts and positive rescaling realize that objective as a bounded nonnegative loss under any prescribed workload with full-support $X$-marginal. The capability construction supplies the deployment object and its semantics: $\mathcal J=(\phi,H)$ specifies inherited information--execution structure, $M$ specifies the architecture's resource-to-policy realization map, and $s$ places architectures on one common external resource scale.

This representation produces four linked consequences. Terminal comparison on $\widehat{\mathcal D}_{\mathcal J}$ and finite comparison on $\widehat\F_s(\mathcal J,M)$ form distinct capability orders and can reverse direction. Finite nonretention is workload-stable because every prescribed full-support $X$-marginal admits a bounded objective exposing the missing capability. Equality of all bounded-loss optima identifies equality of finite risk envelopes, giving the behavioral identification law in Corollary~\ref{cor:main-behavioral-identification}. The restoration transformation $T^{\rm risk}_{1\to2}$ is simultaneously a set-inclusion threshold and the least target-side resource at which universal bounded-loss dominance over the predecessor's budget-$s$ capability is recovered.

\subsection{Nearest-neighbor theory}
\paragraph{Blackwell--Le Cam experiment comparison.}
Blackwell and Le Cam order observation experiments through universal decision value \cite{blackwell1951,blackwell1953,lecam1964}. The present capability square carries universal comparison through a joint information--execution structure and then along a resource-indexed realization dimension. The terminal envelope $\widehat{\mathcal D}_{\mathcal J}$ incorporates effective distinctions together with executable support; the deployed envelope $\widehat\F_s(\mathcal J,M)$ adds an architecture-specific realization map evaluated at a common external resource ceiling. Theorem~\ref{thm:main-budget-reversal} identifies the latter from universal same-budget performance, while $T^{\rm risk}$ traces the resource required to recover one deployed envelope inside another. The resulting terminal/finite pair supplies the order reversal used by the static and dynamic theory.

\paragraph{Convex separation and theorem content.}
Convex separation furnishes the exact witness once the deployed envelope has been constructed. The theorem's representation content is the identification of $\widehat\F_s(\mathcal J,M)$ as the complete invariant of bounded-loss performance at a declared common resource scale. The workload-stable witness and the restoration-resource representation then convert that invariant into two operational consequences: nonretention can be exposed without changing the workload marginal, and compatibility can be measured by the resource level at which universal dominance returns.

\paragraph{Noisy support-function recovery.}
Convex-body reconstruction studies recovery from finitely many noisy support-function measurements \cite{gardner2006,guntuboyina2012}. Corollary~\ref{cor:main-behavioral-identification} connects that geometry to deployed capability: Theorem~\ref{thm:main-budget-reversal} realizes every unit support direction as a bounded-loss decision probe under a full-support workload, while finite-net discretization and fresh-sample concentration give an explicit capability-envelope certificate. The upper bound exposes the declared slice dimension $d_{\rm slice}=|\mathcal X|(|\A|-1)$, the workload-conditioning factor $q_{\min}$, probe optimization error, and statistical error on one scale; the adaptive exact-support lower bound shows an unavoidable dimension-dependent probe cost over unrestricted finite envelopes. Equation~\eqref{eq:main-heldout-risk-certificate} then converts Hausdorff recovery into out-of-probe decision-risk prediction.

\paragraph{Finite information and infinite-dimensional recovery.}
Information-based complexity studies the residual radius left by finitely many information functionals and relates optimal recovery to width-type quantities \cite{traub1988,hinrichs2023}. Structured infinite-dimensional recovery becomes possible when additional compression assumptions such as sparsity or multilevel structure reduce the effective information radius \cite{adcock2021}. Theorem~\ref{thm:main-workload-tail-radius} specializes this viewpoint to capability probes: bounded-loss decision objectives are the information functionals, and the workload $q$ induces an explicit average-TV radius. The countable truncation result then separates residual workload coverage $T_N(q)$ from finite-head shape recovery.

\paragraph{Metareasoning and adaptive computation.}
Rational metareasoning and expected-value-of-control models allocate effort among available computations \cite{shenhav2013,lieder2017}; adaptive computation and LLM test-time scaling make the amount of processing state- or prompt-dependent \cite{graves2016,snell2025,desabbata2024,adaptthink2025,refrain2026}. Endpoint-defined successor-class control adds capability-changing transitions to that allocation problem: evidence acquisition and executable expansion can change the decision class inherited by subsequent reasoning, while Theorem~\ref{thm:main-budget-reversal} prices the finite capability carried through those transitions. The Bellman state therefore couples current realization value with successor-class option value.

\paragraph{Retrieval, tools, and routing.}
RAG, ReAct, Toolformer, adaptive retrieval, tool selection, and model routing supply concrete mechanisms for changing evidence access, executable support, and realization paths \cite{lewis2020,yao2023,schick2023,flare2023,selfrag2024,adaptiverag2024,tecton2025,meco2025,routellm2025,bestroute2025}. Endpoint semantics compares these heterogeneous mechanisms by the successor decision class and finite envelope they create. This gives one control coordinate across mechanism labels and directly links system design to retention and restoration geometry.

\subsection{Contribution map}
\begin{table}[h]
\centering
\small
\setlength{\tabcolsep}{5pt}
\renewcommand{\arraystretch}{1.15}
\begin{tabularx}{\textwidth}{@{}p{3.2cm}p{3.1cm}X@{}}
\toprule
Item & Role in the paper & Contribution \\
\midrule
$\mathcal D_{\mathcal J}$ with $\mathcal J=(\phi,H)$ and $\widehat\F_s(\mathcal J,M)$ & Core capability objects & Couple task distinctions with executable semantics and separate terminal structural possibility from resource-indexed deployed capability. \\
Theorem~\ref{thm:main-budget-reversal} & Core static characterization & Identifies the finite risk envelope as the complete representation of universal same-budget decision performance; terminal/finite disagreement yields reversal and restoration geometry. \\
Endpoint-defined successor-class control & Core dynamic formulation & Classifies heterogeneous interventions by the capability class realized at the endpoint and makes that successor class inheritable by later control. \\
Corollaries~\ref{cor:main-behavioral-identification}--\ref{cor:main-countable-truncation} and Theorem~\ref{thm:main-workload-tail-radius} & Core measurement results & Give finite-slice Hausdorff/statistical recovery, exact workload-tail information radius, and constructive head--tail recovery on countable workloads. \\
Restoration budgets, saturation dominance, switching/hysteresis, recursive diagnostic limits & Theory-generated consequences & Give upgrade overheads, late-stage successor comparisons, control persistence, and information for intervention direction. \\
Full hierarchy, deficiencies, resource transforms, measurable dynamics & Extended formal system & Extend the focused structural/finite object to asymptotic realization, reachability, selection, approximation, and cross-architecture resource geometry. \\
DeepSeek matched interventions & Controlled coordinate identification & Calibrate invariance, reversal, endpoint, retention, and diagnostic signatures under owned interventions. \\
Established LLM mechanisms and phenomena & Framework applications & Use the capability coordinates to diagnose established results and generate retention, restoration, identification, and control consequences. \\
\bottomrule
\end{tabularx}
\end{table}

\paragraph{Core contribution chain.}
The static contribution begins with the resource-indexed capability object and culminates in Theorem~\ref{thm:main-budget-reversal}, which identifies its universal decision meaning exactly. Corollary~\ref{cor:main-behavioral-identification} adds finite-slice Hausdorff recovery, statistical complexity, and held-out risk transfer; Theorem~\ref{thm:main-workload-tail-radius} extends measurement to open-ended workloads through an exact finite-information radius; Corollary~\ref{cor:main-countable-truncation} recovers a constructive positive regime by separating workload-tail coverage from finite-head shape recovery; restoration geometry measures the resource cost of compatibility. The class-valued successor formulation is the dynamic center because an action can change the decision class inherited by subsequent reasoning, while Proposition~\ref{prop:main-eventual-boundary} derives late-stage dominance over fixed realization increments from saturation plus persistent successor value. The predictive module develops the closest refinements; the extended formal module expands their mathematical closure; Module V gives detailed theoretical positioning and applies the coordinates to established LLM results.
\end{textAtEnd}

\begin{textAtEnd}[category=interpretation]
\subsection{Mechanism-level applications}
\label{app:comparison}
The cases below apply the same coordinates to established mechanisms. Their role is interpretive and predictive: cited work supplies the mechanism or empirical phenomenon; the capability state supplies a common location; retention, restoration, or endpoint control supplies the additional consequence.

\paragraph{Chain of thought and Tree of Thoughts.}
Chain of thought increases serial realization and Tree of Thoughts expands and prunes intermediate search \cite{wei2022cot,yao2023tot}. Under fixed task evidence and executable support they primarily move finite reachability; an endpoint that acquires new evidence or executable semantics moves structural capability. The framework therefore predicts different returns to more search across those endpoint classes and directs boundary-changing actions toward the cases where the structural floor is active.

\paragraph{Prompt, context, retrieval, and pruning.}
Prompt compression, retrieval, and branch pruning change active context or search \cite{jiang2023,lewis2020,yao2023tot}. Their capability location follows the realized endpoint: access to a new task-relevant distinction moves $\phi$, persistent realization redesign moves $M$, and selection over already accessible content moves reachability. The generated consequence is a matched access/search comparison under a common resource contract.

\paragraph{Retrieval and tool-using agents.}
RAG, ReAct, Toolformer, adaptive retrieval, and adaptive tool systems expose external information and executable actions \cite{lewis2020,yao2023,schick2023,flare2023,selfrag2024,adaptiverag2024,autosearch2026,tecton2025,meco2025}. Endpoint semantics places their heterogeneous calls on the same successor-class axis. Retention and restoration then measure whether an expanded agent carries predecessor finite capability through the new routing burden.

\paragraph{Transformer expressivity.}
Transformer expressivity results characterize representable functions or languages under architectural and computational assumptions \cite{yun2020,strobl2024,merrill2024}. The capability application adds deployment slices and transitions: structural expressivity determines the terminal class, $M$ and $s$ determine finite realization, and control determines when a successor class is worth constructing. This yields separate empirical targets for representability, deployability, and control value.

\begin{takeaway}
\centering\bfseries
Established LLM results enter the framework as applications to a shared capability state. Their added value here comes from cross-mechanism diagnosis, common-budget retention and restoration tests, and successor-class control predictions.
\end{takeaway}
\end{textAtEnd}

\begin{textAtEnd}[category=experiments]
\section{Completed experimental extensions}
\label{app:extensions}

The extensions contribute 2,880 records to the controlled study. The 64-pair ceiling sweep crosses disabled, low, and high reasoning with 512, 1,024, 2,048, 4,096, and 8,192 output tokens. Low-effort completion accuracy is $0$, $.547$, $.969$, $.984$, and $.984$; high-effort accuracy is $0$, $.750$, $.813$, $.969$, and $1$. The paired high-minus-low difference therefore changes from $+.203$ at 1,024 (23 high-only versus 10 low-only, exact McNemar $p=.0351$) to $-.156$ at 2,048 (1 high-only versus 11 low-only, $p=.00635$), approaches zero at 4,096, and becomes $+.016$ at 8,192. High-effort conditional accuracy is $1$ throughout, and its completion failures through 4,096 are exactly length terminations. The finite ordering changes with the common ceiling while completed-answer quality remains fixed.

The explicitly retained policy menu takes the taskwise better low/high outcome at each common ceiling. Its accuracy is $0$, $.906$, $.984$, $1$, and $1$, compared with high effort's $0$, $.750$, $.813$, $.969$, and $1$. At 2,048, retaining the earlier finite policy raises the empirical envelope by $.172$; at 4,096 the retained menu covers all 64 tasks; at 8,192 high effort itself covers all 64. The sequence realizes the retention prediction of Theorem~\ref{thm:main-budget-reversal}: nonretention permits finite reversal, explicit retention restores the old option, and ceiling expansion removes the observed reversal.

The $2\times2\times3$ factorial crosses hidden versus revealed input, missing versus enabled executable tables, and disabled versus low/high reasoning. With support enabled, reveal changes identification from the hidden $.5$ ceiling to $.500$, $1.000$, and $.969$ for disabled, low, and high modes. With revealed evidence but missing tables, the corresponding scores are $.563$, $.438$, and $.094$; length termination rates are $0$, $.125$, and $.781$. Reasoning becomes effective when both the decisive distinction and executable mapping are present, exposing an information--support--realization interaction beyond a scalar scale summary.

The endpoint experiment crosses two implementation labels, RETRIEVE and TOOL, with two realized successor classes. Search endpoints return a digest computed from the visible record and preserve the hidden-state collision; access endpoints return the decisive external bit. Each access cell scores $32/32$ for every label and effort mode. Search identification scores are $.5$ for both labels with thinking disabled; reasoning-enabled modes frequently emit no admissible decision under the unresolved collision, and no search cell exceeds $.5$. The within-endpoint label gap is zero in every access cell and at most $.031$ across search cells. By contrast, the access-minus-search effect is at least $.5$ in every label--effort cell. The crossed design makes endpoint structural class the stable predictor while the mechanism name varies orthogonally.

The meta-interface experiment reuses the binary collision at the controller layer. A coarse card reports identical confidence, candidate count, and feasible families for a computation bottleneck and a missing-evidence bottleneck, whose optimal controls are THINK and REVEAL. Its scored direction accuracy is $.500$, $.188$, and $.094$ for disabled, low, and high effort; the lower reasoning-enabled scores arise from abstention and remain inside the collision ceiling. A $75\%$-reliable probe raises the same modes to $.656$, $.719$, and $.656$; an exact probe raises them to $.844$, $.969$, and $1.000$. Thus every mode moves monotonically along the information curve from collision through partial separation to exact separation. For the symmetric probe channel, $\TV=2\alpha-1$ and the Bayes direction ceiling is $\alpha$, making reliability $\alpha$ the experimental coordinate of the recursive diagnostic prediction in Appendix~\ref{app:predictions}.

Candidate retention crosses neutral, misleading, and exact-verifier feedback over 64 four-round trajectories. For every round $t$, the measured identity residual
\[
R_t^{\rm deployed}-R_t^{\rm oracle}-G_{{\rm sel},t}
\]
is exactly zero. Under misleading feedback, the roundwise triples $(R^{\rm deployed},R^{\rm oracle},G_{\rm sel}^{\pi})$ are $(1,1,0)/64$, $(3,0,3)/64$, $(0,0,0)/64$, and $(2,0,2)/64$. Thus the first additional-reasoning step gives $\Delta R=2/64=(-1+3)/64$: candidate quality improves by $1/64$ while selection worsens by $3/64$. Final self-selection scores $62/64$, whereas the retained oracle, exact verifier, and best-of-$k$ majority score $64/64$. Neutral and verifier-grounded final selection also score $64/64$. The adversarial protected-marker control records $0/16$ violations in every reasoning mode; grammar projection and post-filtering also record zero.

The multi-agent protocol obtains $16/16$ for a low-effort single agent, retained fallback, verified aggregation, and forced consensus because both reasoning-enabled proposals are correct, supplying a retained-policy consistency check. The routing study supplies the directional contrast: across THINK versus RETRIEVE, TOOL, or CLARIFY, the analytic optimum changes family between $\rho=.5$ and $.7$. All modes select the structural family on all 16 tasks at $\rho=.9$; at $\rho=.7$ the rates are $1$, $1$, and $1$ for disabled, low, and high, while at $\rho=.5$ they are $0$, $.0625$, and $0$. The empirical switch therefore tracks the increasing-differences threshold.

All generators use deterministic seeds, append-only records, exact-match schemas, finish-reason audits, and the last record per stable suite--task--stage--mode identifier. Budget and retention use 64 tasks; the remaining controlled families use 16. Wilson 95\% intervals accompany binomial cells: $52/64$ has interval $[.700,.889]$, $62/64$ has $[.893,.991]$, and $64/64$ has $[.943,1]$. The generated summary additionally reports exact paired McNemar tests, retained-envelope accuracy, conditional accuracy, token use, the complete roundwise risk identity, endpoint main effects and label gaps, the meta-interface information curve, factorial contrasts, topology outcomes, violation rates, and the routing curve.

The theoretical program derives primitive conditions that propagate increasing differences through the Bellman operator, permits changing feasible action correspondences, and models hysteresis with explicit switching costs. Dynamic architecture simulation, stochastic observation experiments, endogenous resource prices, and minimal common envelopes extend the same order, approximation, and resource geometries.
\end{textAtEnd}

\begin{textAtEnd}[category=finalappendix]
\clearpage
\onecolumn
\appendix
\restoreappendixenvironments

\section*{Module I. Core proofs and formal results}
\phantomsection\label{app:proofs}
\addcontentsline{toc}{section}{Module I. Core proofs and formal results}
This module gives the complete main-text proofs and the supporting derivations closest to the core. Formal extensions and validation follow; theoretical positioning and applications are collected in the final module.

\section{Proofs of main-text results}
\printProofs

\subsection{How the core results propagate}

The focused main results generate a larger mathematical closure. Structural invariance extends to approximate experiments; finite-envelope retention induces canonical reversal witnesses and resource comparisons; endpoint-defined Bellman control yields switching and diagnostic consequences. These extensions support the predictive module, and the application module places established cases in the same capability coordinates.

\subsection{Universal and approximate envelope consequences}

\paragraph{Formal extension: structural risk-envelope order.}
Let $\mathcal X$ and $\A$ be finite. For joint structures $\mathcal J_1,\mathcal J_2$ on common task and action spaces, the following are equivalent:
\begin{enumerate}
  \item $\widehat{\mathcal D}_{\mathcal J_1}\subseteq\widehat{\mathcal D}_{\mathcal J_2}$;
  \item for every $X$-marginal $Q_X$ with full support, every finite target space, every conditional law $Q(Y\mid X)$, and every bounded loss,
  \begin{equation}
  \R_{\mathcal J_2}^*(Q)\le\R_{\mathcal J_1}^*(Q).
  \label{eq:app-universal-risk-order}
  \end{equation}
\end{enumerate}
Literal executable inclusion is a stronger sufficient condition. Thus universal bounded-risk comparison identifies the closed convex decision envelope, while literal inclusion identifies exact simulation.
\emph{Derivation.} Expected loss is a continuous linear functional of a finite decision kernel, so optimizing over $\mathcal D_{\mathcal J}$ or its closed convex hull gives the same value. Envelope inclusion therefore implies Eq.~\eqref{eq:app-universal-risk-order}.

For necessity, suppose $C_1:=\widehat{\mathcal D}_{\mathcal J_1}\nsubseteq C_2:=\widehat{\mathcal D}_{\mathcal J_2}$. Choose $d_1\in C_1\setminus C_2$. Compact convex separation gives coefficients $w_{x,a}$ such that
\[
\langle w,d_1\rangle<\inf_{d_2\in C_2}\langle w,d_2\rangle.
\]
Fix a prescribed $q_X$ with full support on $\mathcal X$. Statewise additive shifts leave kernel comparisons unchanged because every row sums to one; a common positive rescaling then makes $w_{x,a}/q_X(x)$ a bounded nonnegative loss. With deterministic target $Y=X$, its expected loss is an affine transform of $\langle w,d\rangle$. Hence a bounded decision problem under that full-support $X$-marginal has strictly smaller optimal risk on $C_1$ than on $C_2$, contradicting universal dominance. Therefore universal dominance implies envelope inclusion.

For a fixed marginal $Q_X$, define the directed structural deficiency
\begin{equation}
\delta_{Q_X}(\mathcal J_1\!\to\!\mathcal J_2)
:=
\sup_{d_1\in\widehat{\mathcal D}_{\mathcal J_1}}
\inf_{d_2\in\widehat{\mathcal D}_{\mathcal J_2}}
\E_{Q_X}\TV\!\left(d_1(\cdot\mid X),d_2(\cdot\mid X)\right).
\label{eq:app-directed-deficiency}
\end{equation}

\paragraph{Formal extension: approximate structural inheritance.}
If $0\le\ell\le L$, then
\begin{equation}
\R_{\mathcal J_2}^*(Q)
\le
\R_{\mathcal J_1}^*(Q)
+L\,\delta_{Q_X}(\mathcal J_1\!\to\!\mathcal J_2).
\label{eq:app-deficiency-risk}
\end{equation}
For two equiprobable worlds with interface laws $P_0^\phi,P_1^\phi$, the optimal binary distinction error is
\begin{equation}
\R_{\rm test}^*=\frac{1-\TV(P_0^\phi,P_1^\phi)}{2}.
\label{eq:app-approx-collision}
\end{equation}
Thus exact fibers, approximate fibers, and approximate cross-architecture simulation lie on one quantitative scale.
\emph{Derivation.} Let $h_x(a)=\E[\ell(a,Y)\mid X=x]$. Since $0\le h_x\le L$, for any two action kernels,
\[
\left|\int h_x\,d_1(da\mid x)-\int h_x\,d_2(da\mid x)\right|
\le L\TV(d_1(\cdot\mid x),d_2(\cdot\mid x)).
\]
Integrate over $Q_X$, choose an arbitrarily near-optimal $d_1$ for $\mathcal J_1$, then an arbitrarily near-best simulator $d_2$ for that kernel, and let both approximation errors vanish. This proves Eq.~\eqref{eq:app-deficiency-risk}.

For a binary decision region $B$, equal-prior success is $\tfrac12[P_0^\phi(B)+P_1^\phi(B^c)]$. Maximizing over $B$ gives $\tfrac12(1+\TV(P_0^\phi,P_1^\phi))$ by the variational definition of total variation. Subtracting from one proves Eq.~\eqref{eq:app-approx-collision}.

\subsection{Infinite-space recovery and effective geometry}
\label{app:infinite-recovery}

Theorem~\ref{thm:main-workload-tail-radius} is the workload-specific radius-of-information counterpart of the finite support-function geometry. Its exact binary-singleton calculation separates state-space coverage from convex-shape recovery: finite slices carry a directional approximation problem, while open-ended workloads first require enough information to cover the probability mass that matters.

For a constructive countable extension, let
\[
q=\sum_{i\ge1}p_i\delta_{x_i},
\qquad p_1\ge p_2\ge\cdots>0,
\qquad |\A|=a<\infty,
\]
and define the workload-average kernel metric
\[
d_q(d,e)
:=
\sum_{i\ge1}p_i\,
\TV\!\left(d(\cdot\mid x_i),e(\cdot\mid x_i)\right).
\]
Let $d_{H,q}$ be the induced Hausdorff distance, let
$T_N(q):=\sum_{i>N}p_i$, and let $\pi_N C$ denote projection of a closed envelope $C$ onto the first $N$ task states.

Corollary~\ref{cor:main-countable-truncation} gives the constructive head--tail bound in the main text. Its proof is collected with the other main-text results above. The bound makes the infinite-space extension operational: $T_N(q)$ sets the unresolved workload mass, while the retained-head dimension $N(a-1)$ and normalized smallest mass $p_N/(1-T_N(q))$ determine the directional and statistical cost of applying finite-slice tomography to the head. This is the positive counterpart to Theorem~\ref{thm:main-workload-tail-radius}: countable tail decay supplies an effective finite evaluation geometry, whereas non-atomic mass leaves a non-vanishing finite-information radius on the unrestricted singleton subclass.

\subsection{Canonical witnesses for finite-envelope nonretention}

the workload-stable reversal prediction in Appendix~\ref{app:predictions} is the canonical workload form of Theorem~\ref{thm:main-budget-reversal}; its proof is printed above with the other main-text results. The separating construction fixes an arbitrary full-support $X$-marginal, uses statewise shifts and positive rescaling to obtain a bounded nonnegative loss, and thereby isolates finite realization as the source of reversal whenever terminal dominance also holds.

\subsection{Formal identity for the overthinking reinterpretation}

\paragraph{Derived attribution identity for retained candidates.}
Let $\mathcal C_0\subseteq\cdots\subseteq\mathcal C_T$ be externally retained candidate families under one task law and loss. Define
\[
R_t=\inf_{d\in\mathcal C_t}\mathcal L(d),
\qquad
G_t=\mathcal L(d_t)-R_t.
\]
Then $R_{t+1}\le R_t$ and, for every $0\le r<t\le T$,
\begin{equation}
\mathcal L(d_t)-\mathcal L(d_r)
=
(R_t-R_r)+(G_t-G_r).
\label{eq:app-sequence-attribution}
\end{equation}
In particular, a positive deployed-loss increment at step $t$ occurs exactly when
\begin{equation}
G_{t+1}-G_t>R_t-R_{t+1}\ge0.
\label{eq:app-selector-dominates}
\end{equation}
An observed increase in $R_t$ certifies failure of candidate retention; selector degradation alone cannot raise the oracle infimum.
\emph{Derivation.} Nested candidate families make the oracle infima nonincreasing. The identity $\mathcal L(d_t)=R_t+G_t$ holds at every step; subtracting it at times $r$ and $t$ gives Eq.~\eqref{eq:app-sequence-attribution}. Taking adjacent times and using $R_{t+1}-R_t\le0$ gives Eq.~\eqref{eq:app-selector-dominates}. If $R_{t+1}>R_t$ under unchanged evaluation, the new candidate family cannot contain the old one, because an infimum over a superset cannot increase.

The four-round external-retention experiment is an exact finite-sample instantiation: every measured residual $\mathcal L(d_t)-R_t-G_t$ is zero, and the first misleading-feedback transition gives $2/64=-1/64+3/64$.

\subsection{Formal derivation behind saturation and switching}

\paragraph{Saturation law.}
For fixed $(\mathcal J,M,Q,\ell)$ on a scalar resource axis, nestedness $\F_s\subseteq\F_t$ for $s\le t$ makes $R(s)=\R^*_{s,\mathcal J,M}(Q,\ell)$ nonincreasing. Bounded loss supplies a lower bound, so $R(s)$ converges to $R_\infty=\inf_tR(t)$. For every fixed $h>0$, both $R(s)$ and $R(s+h)$ converge to $R_\infty$, giving $R(s)-R(s+h)\to0$. This is the geometric source of Proposition~\ref{prop:main-eventual-boundary}; its proof is printed with the other main-text results above.

\paragraph{Uniqueness refinement.}
Let $G_O(\rho)$ and $G_C(\rho)$ be complete net gains from inherited-class realization and structural construction along a scalar saturation path, and write $A_G=G_O-G_C$. If $A_G$ is continuous, strictly decreasing, and changes sign, the preferred family has one threshold. A sufficient differentiable condition is $G_O'(\rho)<0$ and $G_C'(\rho)\ge0$. Increasing differences of the optimized Bellman family costs gives the same single-crossing property while allowing the minimizing action to change within each family. These conditions refine eventual dominance into a unique switching topology; asymmetric switching costs widen the crossing into a two-threshold hysteresis band.

\subsection{Formal derivation behind the diagnostic prediction}

The predictive module records the exact two-state diagnostic formula and its value-of-information interpretation.

\printProofs[predictions]

\section*{Module III. Extended formal system}
\addcontentsline{toc}{section}{Module III. Extended formal system}
The full Symbolization--Substructure construction, information--execution model, dynamic program, resource transformations, and technical capability geometry are collected here after the core results and their closest predictions.

\printProofs[foundation]

\section{Broader capability geometry and application space}
\label{app:broader-scope}

The main text isolates the structural/finite distinction needed for the two static limits and the resulting dynamic direction choice. The same inheritance principle resolves into a finer capability hierarchy when asymptotic realization, current reachability, and deployed selection are tracked explicitly:
\begin{equation}
\F_{\rm stop}(S)
\subseteq
\F_s(\mathcal J,M)
\subseteq
\F_\infty(\mathcal J,M)
\subseteq
\mathcal D_{\mathcal J}.
\label{eq:app-full-envelope-chain}
\end{equation}
For a fixed task law and loss, let $\R_{\rm reach}$, $\R_{\rm budget}$, $\R_{\rm asym}$, and $\R_{\rm floor}$ be the corresponding risk infima, let $d_S^\pi$ be the deployed decision, and let $\R_{\rm ref}\le\R_{\rm floor}$. The full hierarchy gives the intervention-owned decomposition
\begin{align}
\mathcal L(d_S^\pi)-\R_{\rm ref}
={}&\underbrace{\R_{\rm floor}-\R_{\rm ref}}_{C_{\rm struct}}
+\underbrace{\R_{\rm asym}-\R_{\rm floor}}_{C_{\rm asym}}
\nonumber\\
&+\underbrace{\R_{\rm budget}-\R_{\rm asym}}_{C_{\rm fin}}
+\underbrace{\R_{\rm reach}-\R_{\rm budget}}_{C_{\rm reach}}
\nonumber\\
&+\underbrace{\mathcal L(d_S^\pi)-\R_{\rm reach}}_{G_{\rm sel}^{\pi}}.
\label{eq:app-broader-failure-anatomy}
\end{align}
The coordinates separate structural impossibility, asymptotic realization limits, finite authorization, pathwise search, and selection. Evidence or executable expansion moves $\mathcal J$; persistent inference redesign moves $M$; authorization moves $s$; search moves the current path; diagnostic access and reranking move the selector. This finer geometry is useful for attribution once the focused structural/finite distinction has identified the relevant direction.

\begin{figure}[t]
\centering
\includegraphics[width=0.96\textwidth]{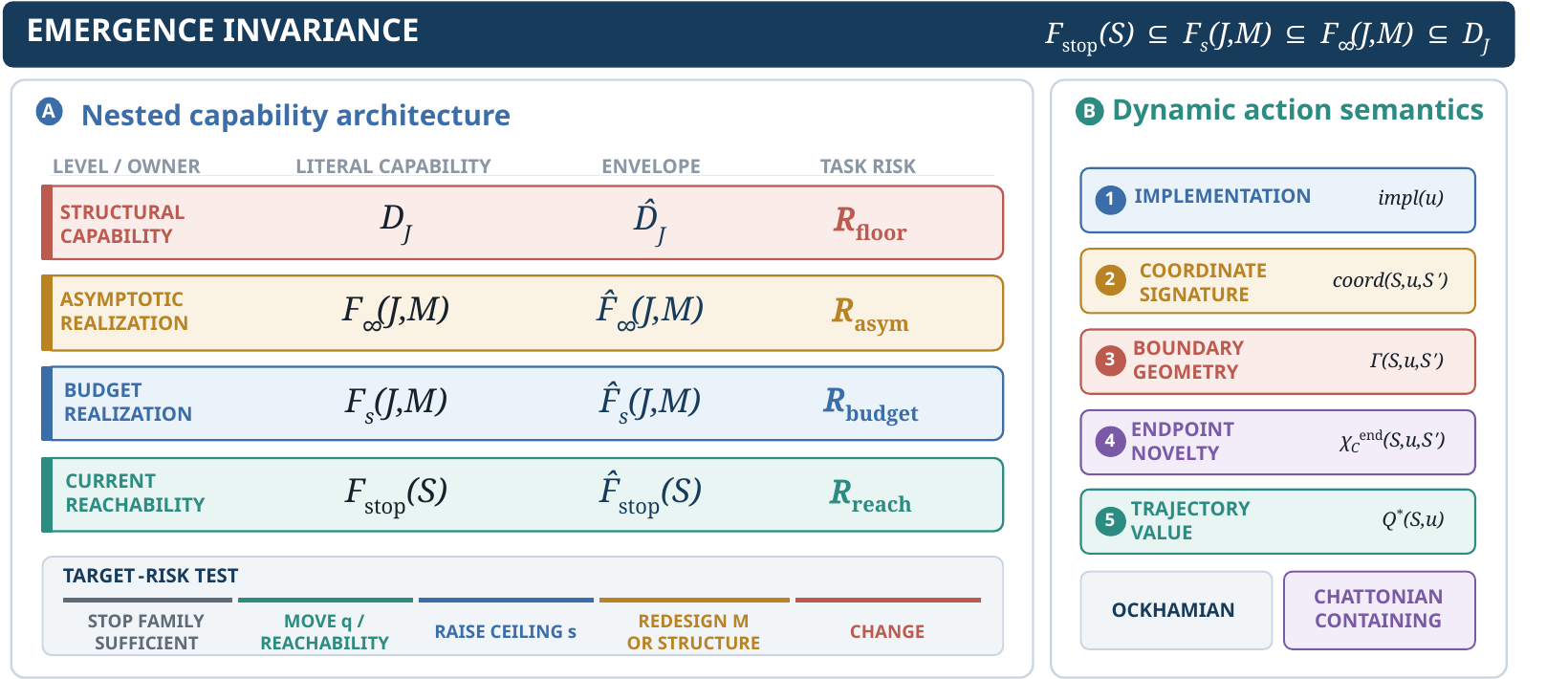}
\caption{\textbf{Full capability geometry.} The focused main text uses the structural and finite-budget levels. The broader theory resolves inherited capability into structural, asymptotic, budget, and current-reachability envelopes, and then attaches endpoint transition semantics and recursive control.}
\label{fig:appendix-full-geometry}
\end{figure}

The remaining appendix develops the theory's application space. Static sections derive information-comparison, executable-support, deficiency, and resource-transformation results. Dynamic sections treat changing feasible actions, contraction, hysteresis, recursive meta-interface collisions, and diagnostic value. Further sections project the same retention principle onto candidate-level overthinking, constrained safety, and multi-agent protocols, and compare language-mediated recursion with broader human substrate control. The empirical appendices give the complete API reconstruction, ceiling sweeps, endpoint crosses, diagnostic probes, retained-candidate checks, and routing curves.

\section{Additional formal details}

\subsection{Measurable dynamic setup}

Let $(\Sspace,\Sigma_{\Sspace})$ be a measurable state space and $(\overline{\U},\Sigma_{\U})$ a common measurable meta-action space. The feasible-action correspondence $S\mapsto\U(S)$ has a measurable graph and nonempty values, including stop. Each state contains an accessible-information variable $\iota_S$ with $\mathcal G_S=\sigma(\iota_S)$ and carries a probability kernel $S\mapsto P_S$ on $\mathcal X\times\mathcal Y$ that is a version of $\mathcal L((X,Y)\mid\mathcal G_S)$. Analyst-only latent variables can enter the transition kernel but do not condition terminal risk until an action reveals them to the deployed system. A transition that forgets evidence forms the next kernel from the retained accessible information; any evidential summary still usable by the terminal decision remains part of $\iota_S$. Separately, let $\mu:\Sspace\to\mathcal Z_\mu$ be a measurable meta-interface. An admissible deployed controller is a stochastic kernel $\pi$ from $\mathcal Z_\mu$ to $\overline{\U}$ satisfying $\pi(\U(S)\mid\mu(S))=1$; hence its state dependence factors through $\mu(S)$.

Let $S\mapsto\F_{\rm stop}(S)$ be a nonempty state-dependent decision-kernel correspondence contained in $\F_{s(S)}(\mathcal J(S),M(S))$, with the measurability conditions required for the risk infimum and measurable terminal selection. Let $k:\Sspace\times\overline{\U}\to[0,\infty]$, $c:\Sspace\times\overline{\U}\to[0,\infty)$, and $L_{\rm term}:\Sspace\to[0,\infty]$ be measurable, and let the controlled state transition be a measurable stochastic kernel. For feasible non-stop actions assume $c(S,u)\le b$, and assume stopping is feasible with finite terminal loss. Assume the state is Markov sufficient for future controlled evolution, policy classes are closed under one-step concatenation, and the dynamic programming principle applies. Assume $S\mapsto\R_{\rm floor}(S)$, $\R_{\rm asym}(S)$, $\R_{\rm budget}(S)$, $\R_{\rm reach}(S)$ and the gaps $C_{\rm reach}(S)$, $C_{\rm fin}(S)$, $C_{\rm asym}(S)$ are measurable and integrable whenever their conditional expectations are used. Whenever endpoint structural novelty is used, assume that
\begin{equation}
\mathcal N_C
:=
\left\{
(S,u,S'):
\mathcal D_{\mathcal J(S')}\not\subseteq\mathcal D_{\mathcal J(S)}
\right\}
\label{eq:noveltymeasurable}
\end{equation}
is measurable in the realized transition variables. Then $\chi_C^{\rm end}$ is measurable and
$ p_C(S,u)=\E[\chi_C^{\rm end}(S,u,S')\mid S,u] $ is a measurable state--action function. Whenever Bellman infima or measurable selections are taken over $\U_O(S)$, $\U_{C+}(S)$, $\U_C^{\rm pure}(S)$, or $\U_M(S)$, assume that the corresponding state-relative feasible-family correspondence has a measurable graph and nonempty values on the states where that envelope or selection is invoked. Whenever the operational loss is used, the policy-induced terminal kernel selection $S\mapsto d_S^\pi$ is measurable, $d_S^\pi\in\F_{\rm stop}(S)$, and the selection gap $G_{\rm sel}^{\pi}(S)$ is measurable and integrable under $P_S$.

\subsection{Budget-forced finite stopping}

\begin{theorem}[Uniform stopping under finite budget]
Suppose $0\le b_0<\infty$, budget is updated by $b_{t+1}=b_t-c(S_t,u_t)$ and is never replenished, every feasible non-stop action consumes at least $\kappa>0$, and no non-stop action is feasible when $b_t<\kappa$. Then every policy satisfies
\begin{equation}
\tau\le\left\lfloor\frac{b_0}{\kappa}\right\rfloor
\end{equation}
almost surely.
\end{theorem}
\begin{proof}
Let $m=\lfloor b_0/\kappa\rfloor$. If the process has not stopped before time $m$, the first $m$ actions are non-stop and consume at least $m\kappa$. Hence $b_m\le b_0-m\kappa<\kappa$. Only stop is then feasible, so $\tau=m$.
\end{proof}

\subsection{Complete dynamic policy-class nesting}

\begin{proposition}[Policy-class monotonicity]
Fix the same initial state, transition kernel, costs, terminal loss, and discount factor. If complete dynamic policy classes satisfy $\mathfrak P_t\subseteq\mathfrak P_{t+1}$, then
\begin{equation}
\inf_{\pi\in\mathfrak P_{t+1}}J^\pi(S_0)\le
\inf_{\pi\in\mathfrak P_t}J^\pi(S_0).
\end{equation}
\end{proposition}
\begin{proof}
The same cost functional is minimized over a superset.
\end{proof}

Interface refinement yields policy-class nesting when the new controller can ignore new information, reproduce old computation, preserve old executable semantics, and select the old action sequence. A non-comparable replacement has no such monotonicity guarantee unless it is embedded in a common envelope that can reproduce the old policy class.

\printProofs[experiments]

\section*{Module V. Related work and theoretical positioning}
\addcontentsline{toc}{section}{Module V. Related work and theoretical positioning}
This final module places the completed theory and evidence beside the closest mathematical and LLM-system lines, then applies the capability coordinates to established results.

\printProofs[originality]
\printProofs[interpretation]
\end{textAtEnd}

\printProofs[finalappendix]


\clearpage
\begin{thebibliography}{99}
\small
\bibitem{ockham} Spade et al. \emph{William of Ockham}.

\bibitem{chatton} Keele et al. \emph{Walter Chatton}.

\bibitem{fedorenko2024} Fedorenko et al. \emph{Language is Primarily a Tool for Communication Rather Than Thought}.

\bibitem{mahowald2024} Mahowald et al. \emph{Dissociating Language and Thought in Large Language Models}.

\bibitem{spelke2007} Spelke et al. \emph{Core Knowledge}.

\bibitem{cisek2007} Cisek. \emph{Cortical Mechanisms of Action Selection: The Affordance Competition Hypothesis}.

\bibitem{kirsh1994} Kirsh et al. \emph{On Distinguishing Epistemic from Pragmatic Action}.

\bibitem{shenhav2013} Shenhav et al. \emph{The Expected Value of Control: An Integrative Theory of Anterior Cingulate Cortex Function}.

\bibitem{lieder2017} Lieder et al. \emph{Strategy Selection as Rational Metareasoning}.

\bibitem{fleming2012} Fleming et al. \emph{The Neural Basis of Metacognitive Ability}.

\bibitem{maniscalco2012} Maniscalco et al. \emph{A Signal Detection Theoretic Approach for Estimating Metacognitive Sensitivity from Confidence Ratings}.

\bibitem{vaswani2017} Vaswani et al. \emph{Attention Is All You Need}.

\bibitem{yun2020} Yun et al. \emph{Are Transformers Universal Approximators of Sequence-to-Sequence Functions?}.

\bibitem{strobl2024} Strobl et al. \emph{What Formal Languages Can Transformers Express? A Survey}.

\bibitem{wei2022cot} Wei et al. \emph{Chain-of-Thought Prompting Elicits Reasoning in Large Language Models}.

\bibitem{xie2022} Xie et al. \emph{An Explanation of In-Context Learning as Implicit Bayesian Inference}.

\bibitem{lewis2020} Lewis et al. \emph{Retrieval-Augmented Generation for Knowledge-Intensive NLP Tasks}.

\bibitem{yao2023} Yao et al. \emph{ReAct: Synergizing Reasoning and Acting in Language Models}.

\bibitem{yao2023tot} Yao et al. \emph{Tree of Thoughts: Deliberate Problem Solving with Large Language Models}.

\bibitem{schick2023} Schick et al. \emph{Toolformer: Language Models Can Teach Themselves to Use Tools}.

\bibitem{jiang2023} Jiang et al. \emph{LLMLingua: Compressing Prompts for Accelerated Inference of Large Language Models}.

\bibitem{zhang2024ssa} Zhang et al. \emph{Selective Attention: Enhancing Transformer through Principled Context Control}.

\bibitem{bansal2026} Bansal et al. \emph{Let's (not) just put things in Context: Test-time Training for Long-context LLMs}.

\bibitem{hoffmann2022} Hoffmann et al. \emph{An Empirical Analysis of Compute-Optimal Large Language Model Training}.

\bibitem{he2016} He et al. \emph{Deep Residual Learning for Image Recognition}.

\bibitem{graves2014} Graves et al. \emph{Neural Turing Machines}.

\bibitem{graves2016} Graves. \emph{Adaptive Computation Time for Recurrent Neural Networks}.

\bibitem{snell2025} Snell et al. \emph{Scaling LLM Test-Time Compute Optimally Can be More Effective than Scaling Parameters for Reasoning}.

\bibitem{flare2023} Jiang et al. \emph{Active Retrieval Augmented Generation}.

\bibitem{selfrag2024} Asai et al. \emph{Self-RAG: Learning to Retrieve, Generate, and Critique through Self-Reflection}.

\bibitem{adaptiverag2024} Jeong et al. \emph{Adaptive-RAG: Learning to Adapt Retrieval-Augmented Large Language Models through Question Complexity}.

\bibitem{desabbata2024} De Sabbata et al. \emph{Rational Metareasoning for Large Language Models}.

\bibitem{routellm2025} Ong et al. \emph{RouteLLM: Learning to Route LLMs from Preference Data}.

\bibitem{bestroute2025} Ding et al. \emph{BEST-Route: Adaptive LLM Routing with Test-Time Optimal Compute}.

\bibitem{tecton2025} Alazraki et al. \emph{Meta-Reasoning Improves Tool Use in Large Language Models}.

\bibitem{meco2025} Li et al. \emph{Adaptive Tool Use in Large Language Models with Meta-Cognition Trigger}.

\bibitem{adaptthink2025} Zhang et al. \emph{AdaptThink: Reasoning Models Can Learn When to Think}.

\bibitem{refrain2026} Sun et al. \emph{Stop When Enough: Adaptive Early-Stopping for Chain-of-Thought Reasoning}.

\bibitem{autosearch2026} Sun et al. \emph{AutoSearch: Adaptive Search Depth for Efficient Agentic RAG via Reinforcement Learning}.

\bibitem{wei2022} Wei et al. \emph{Emergent Abilities of Large Language Models}.

\bibitem{schaeffer2023} Schaeffer et al. \emph{Are Emergent Abilities of Large Language Models a Mirage?}.

\bibitem{hu2024} Hu et al. \emph{Predicting Emergent Abilities with Infinite Resolution Evaluation}.

\bibitem{du2024} Du et al. \emph{Understanding Emergent Abilities of Language Models from the Loss Perspective}.

\bibitem{merrill2024} Merrill et al. \emph{The Expressive Power of Transformers with Chain of Thought}.

\bibitem{mielke2022} Mielke et al. \emph{Reducing Conversational Agents' Overconfidence Through Linguistic Calibration}.

\bibitem{farquhar2024} Farquhar et al. \emph{Detecting Hallucinations in Large Language Models Using Semantic Entropy}.

\bibitem{chrisman1992} Chrisman. \emph{Reinforcement Learning with Perceptual Aliasing: The Perceptual Distinctions Approach}.

\bibitem{kaelbling1998} Kaelbling et al. \emph{Planning and Acting in Partially Observable Stochastic Domains}.

\bibitem{blackwell1951} Blackwell. \emph{Comparison of Experiments}.

\bibitem{blackwell1953} Blackwell. \emph{Equivalent Comparisons of Experiments}.

\bibitem{bender2020} Bender et al. \emph{Climbing towards NLU: On Meaning, Form, and Understanding in the Age of Data}.

\bibitem{lecam1964} Le Cam. \emph{Sufficiency and Approximate Sufficiency}.

\bibitem{chen2025overthink} Chen et al. \emph{Do NOT Think That Much for $2+3=?$ On the Overthinking of Long Reasoning Models}.

\bibitem{fan2025} Fan et al. \emph{Missing Premise Exacerbates Overthinking: Are Reasoning Models Losing Critical Thinking Skill?}.

\bibitem{zhang2026overthink} Zhang et al. \emph{Do LLMs Really Need 10+ Thoughts for ``Find the Time 1000 Days Later''? Towards Structural Understanding of LLM Overthinking}.

\bibitem{su2021} Su et al. \emph{RoFormer: Enhanced Transformer with Rotary Position Embedding}.

\bibitem{mckenzie2023inverse} McKenzie et al. \emph{Inverse Scaling: When Bigger Isn't Better}.

\bibitem{wei2022ushape} Wei et al. \emph{Inverse Scaling Can Become U-Shaped}.

\bibitem{huang2023selfcorrect} Huang et al. \emph{Large Language Models Cannot Self-Correct Reasoning Yet}.

\bibitem{tyen2024} Tyen et al. \emph{LLMs Cannot Find Reasoning Errors, but Can Correct Them Given the Error Location}.

\bibitem{zhang2025dark} Zhang et al. \emph{Understanding the Dark Side of LLMs' Intrinsic Self-Correction}.

\bibitem{scholak2021picard} Scholak et al. \emph{PICARD: Parsing Incrementally for Constrained Auto-Regressive Decoding from Language Models}.

\bibitem{geng2023grammar} Geng et al. \emph{Grammar-Constrained Decoding for Structured NLP Tasks without Finetuning}.

\bibitem{li2024agents} Li et al. \emph{More Agents Is All You Need}.

\bibitem{kaesberg2025} Kaesberg et al. \emph{Voting or Consensus? Decision-Making in Multi-Agent Debate}.

\bibitem{cui2026} Cui et al. \emph{Free-MAD: Consensus-Free Multi-Agent Debate}.

\bibitem{traub1988} Traub et al. \emph{Information-Based Complexity}.

\bibitem{hinrichs2023} Hinrichs et al. \emph{Random Sections of $\ell_p$-Ellipsoids, Optimal Recovery and Gelfand Numbers of Diagonal Operators}.

\bibitem{adcock2021} Adcock et al. \emph{Uniform Recovery in Infinite-Dimensional Compressed Sensing and Applications to Structured Binary Sampling}.

\bibitem{gardner2006} Gardner et al. \emph{Convergence of Algorithms for Reconstructing Convex Bodies and Directional Measures}.

\bibitem{guntuboyina2012} Guntuboyina. \emph{Optimal Rates of Convergence for Convex Set Estimation from Support Functions}.

\bibitem{hofstatter2025} Hofst\"atter et al. \emph{The Elicitation Game: Evaluating Capability Elicitation Techniques}.

\bibitem{brown2025task} Brown et al. \emph{Adaptively Profiling Models with Task Elicitation}.

\bibitem{deepseekapi2026} DeepSeek-AI. Thinking Mode; DeepSeek-V4-Flash-0731 Reasoning Effort. \emph{Official API documentation and model repository}, accessed August 19, 2026. \url{https://api-docs.deepseek.com/guides/thinking_mode/}; \url{https://huggingface.co/deepseek-ai/DeepSeek-V4-Flash-0731/blob/main/encoding/README.md}.

\end{thebibliography}
\end{document}